\documentclass[twoside,11pt]{article}

\usepackage{color}
\usepackage{jmlr2e}
\usepackage{graphics}
\usepackage{epsfig}
\usepackage[ruled]{algorithm2e}
\usepackage{booktabs}
\usepackage{url}
\usepackage{psfrag}
\usepackage{relsize}
\usepackage{amsmath,amsfonts,mathrsfs,mathtools,amssymb}
\usepackage[colorlinks,linkcolor=blue,anchorcolor=green,citecolor=green]{hyperref}
\usepackage{tikz,pgfplots}
\usepackage{tkz-tab}
\usepackage[format=hang,justification=justified,singlelinecheck=false,font=footnotesize]{caption}
\usepackage[font=footnotesize]{subcaption}

\usepackage{enumitem} 
\usepackage{enumerate}

\usepackage[T1]{fontenc}
\usepackage[utf8]{inputenc}
\usepackage[font=small,labelfont=bf,tableposition=top]{caption}
\usepackage{booktabs}
\usepackage{threeparttable}

\usepackage{multirow}

\newcommand{\eins}{\boldsymbol{1}}

\DeclareSymbolFont{wideparensymbol}{OMX}{yhex}{m}{n}
\DeclareMathAccent{\wideparen}{\mathord}{wideparensymbol}{"F3} 

\newcommand{\argmin}{\operatornamewithlimits{arg \, min}}

\jmlrheading{}{}{~}{~}{~}{2018 Hanyuan Hang, Yingyi Chen and Johan A.K. Suykens}

\usepackage{tikz,pgfplots}
  \pgfplotsset{compat=newest}
  \pgfplotsset{plot coordinates/math parser=false,trim axis left}
     \newlength\figureheight
     \newlength\figurewidth

\begin{document}

\title{Two-stage Best-scored Random Forest for \\ Large-scale Regression}

\author{\name Hanyuan Hang \email hans2017@ruc.edu.cn \\
	\name Yingyi Chen \email chenyingyi@ruc.edu.cn \\
	\addr Institute of Statistics and Big Data \\  
	Renmin University of China \\
	100872 Beijing, China \\   
	      \\     
    \name Johan A.K. Suykens \email johan.suykens@esat.kuleuven.be\\
    \addr Department of Electrical Engineering, ESAT-STADIUS, KU Leuven\\
    Kasteelpark Arenberg 10, Leuven, B-3001, Belgium}


\maketitle



\begin{abstract}We propose a novel method designed for large-scale regression problems, namely the two-stage best-scored random forest (TBRF). \textit{Best-scored} means to select one regression tree with the best empirical performance out of a certain number of purely random regression tree candidates, and \textit{two-stage} means to divide the original random tree splitting procedure into two: In stage one, the feature space is partitioned into non-overlapping cells; in stage two, child trees grow separately on these cells. The strengths of this algorithm can be summarized as follows: First of all, the pure randomness in TBRF leads to the almost optimal learning rates, and also makes ensemble learning possible, which resolves the boundary discontinuities long plaguing the existing algorithms. Secondly, the two-stage procedure paves the way for parallel computing, leading to computational efficiency. Last but not least, TBRF can serve as an inclusive framework where different mainstream regression strategies such as linear predictor and least squares support vector machines (LS-SVMs) can also be incorporated as value assignment approaches on leaves of the child trees, depending on the characteristics of the underlying data sets. Numerical assessments on comparisons with other state-of-the-art methods on several large-scale real data sets validate the promising prediction accuracy and high computational efficiency of our algorithm.
\end{abstract}


\section{Introduction} \label{sec::Introduction}
The ever-increasing scale of modern scientific and technological data sets raises urgent requirements for learning algorithms that not only maintain desirable prediction accuracy but also have high computational efficiency \citep{wen2018thundersvm, guo2018efficient, thomann2017spatial, hsieh2014divide}. However, a major challenge is that the data analysis and learning algorithms suitable for modest-sized data sets often encounter difficulties or are even infeasible to tackle large-volume data sets, which leads to the current popular research direction named large-scale regression \citep{collobert2001svmtorch, raskutti2016statistical}. In the literature, efforts have been made to conquer the large-scale regression problems and each method has its own merits in its own regimes. Typically, the mainstream solutions come in two flavors, which are \emph{horizontal methods} and \emph{vertical methods}. The essence of the \emph{horizontal methods}, also called distributed learning, is to partition the data set into data subsets, store them on multiple machines and allow these machines to train in parallel with each machine processing its local data to give local predictors. Then the local predictors are synthesized to give a final predictor \citep{zhang2013divide, zhang15a, lin17a, guo17a}. Nevertheless, horizontal methods have to face their own problems. 
Specifically, what is originally needed is a global predictor defined on the whole feature space which should be trained based on all the training data via the chosen regression algorithm. 
However, the local predictors also defined on the whole feature space are actually trained based only on the information provided by the data subsets. In this manner, chances are high that each local predictor may be very different from the desired global predictor, let alone the synthesized final predictor.

The other category of methods to resolve large-scale regression problems is the \emph{vertical methods}. Its main idea is to first partition the whole feature space (i.e.~the input domain) into multiple non-overlapping cells where different partition methods \citep{suykens2002least, EsSuMo06a, bennett98a, wu99a, chang10a} can be employed. Then, for each of the resulting cells, a predictor is training based on samples falling into that cell via regression strategies such as Gaussian process regression \citep{park11a, park16a, park2018patchwork}, support vector machines \citep{meister16a,thomann2017spatial}, etc.
However, the long-standing boundary discontinuities have always been a headache for vertical methods for degrading the regression accuracy, and literature has committed to settling this problem. 
For example, \cite{park11a} first applies Gaussian process regression to each cell of the decomposed domain with equal boundary constraints merely at a finite number of locations.
After finding that this method cannot essentially solve the boundary discontinuities, they propose a solution specially for this issue in \cite{park16a} which constraints the predictions of local regressions to share a common boundary. 
To further mitigate the boundary discontinuities, recently, \cite{park2018patchwork} proposes Patch Kriging (PK) which improves previous work with the help of adding additional pseudo-observations to the boundaries.
However, boundaries where two adjacent Gaussian processes are joined up are artificially chosen, which may have a great impact on the final predictor. 
Moreover, their approach is fundamentally different from the original Gaussian process which is a global method with respect to algorithm structure.
Additionally, their method may not be that appropriate for parallel computing.
Another vertical method called the Voronoi partition support vector machine (VP-SVM) \citep{meister16a} is available for parallel computing, while boundary discontinuities are not demonstrably solved. Besides, their method also no longer shares the same spirit as the original global algorithm LS-SVMs \citep{suykens2002least}.
To the best of our knowledge, up till now, there is no such algorithm that not only overcomes the boundary discontinuities problems that long plague the vertical methods, but also takes full advantage of the huge parallel computing resources brought by the big data era to obtain results both efficient and effective.

Aiming at solving these tough problems, in this paper, we propose a novel vertical algorithm named the \emph{two-stage best-scored random forest}, which is an exact fit for solving large-scale regression problems. To be specific, in stage one, the feature space is partitioned following an adaptive random splitting criterion into a number of cells, which paves the way for parallel computing. In stage two, splits are continuously conducted on each cell separately following a purely random splitting criterion. Due to the inherent randomness of this splitting criterion, for each cell, we are able to establish different regression trees under different partitions, and then pick up the one with the best empirical performance to be the child best-scored random tree of that cell. Accordingly, we name this selection strategy the ``best-scored'' method. Subsequently, the concatenation of child best-scored random trees from all cells forms a parent best-scored random tree. By following the above construction procedure repeatedly, we are able to establish a number of different parent best-scored random trees whose ensemble is just the two-stage best-scored random forest. The prominent strengths of our algorithm over other vertical methods can be demonstrated from the following perspectives:

\textit{(i)} In most of the existing vertical methods, the feature space is usually artificially partitioned into different non-overlapping cells and the original algorithm is then applied to each of these regions, respectively. In the original algorithm, the prediction of any point in the feature space is influenced by the information of all the sample points, whereas in the corresponding vertical methods, the prediction of any point may be only affected by the information of sample points in its belonging cell. This usually leads to an essential change of the algorithm structure and accordingly,
the global smoothness of the original method is jeopardized and only
the smoothness within each cell
can now be guaranteed, often resulting in the boundary discontinuity problem. In contrast, this is never a problem for our two-stage best-scored random forest (TBRF) method, since random forest (RF) is intrinsically an ensemble method bringing its asymptotic smoothness. As for our two-stage random forest method, we only divide the whole original splitting process of one tree into two stages for the sake of parallelism. This does not change the nature of TBRF as an RF method.

\textit{(ii)} Owing to the two-stage structure of our proposed algorithm and the architecture of the random forests, the TBRF achieves satisfying performance in terms of computational efficiency and prediction accuracy, which have always had great significance in the big data era. Specifically, the computational efficiency is twofold. First of all, the algorithm can be significantly sped up by leveraging parallel computing in both stages. Considering that parent trees in the forest require different adaptive random partitions of the feature space which are conducted in stage one, we can assign each adaptive partition to a different core for acceleration. This is a direct advantage of parallelism brought by the ensemble learning resided in the random forest. Moreover, the establishment of child best-scored random trees whose total number is the total amount of cells in all parent trees, can be also assigned to different cores, so that the computational burden can be decentralized. For another, the adaptive random partition in stage one is completely data-driven and this splitting mechanism makes the number of samples falling into each cell more evenly distributed. Therefore, it increases the number of effective splits, and further reduces the training time for parallel computing. When it comes to the prediction accuracy, we manage to incorporate some existing mainstream regression algorithms as value assignment methods into our random forest architecture. In addition to only assigning a constant to each terminal node of the trees, we employ a few alternatives, such as fitting linear regression functions for low dimensional data, and utilizing a Gaussian kernel for high dimensional data, due to their different performances when encountering different dimensional data. Numerical experiments further demonstrate the effectiveness in choosing appropriate assignment strategies for different data. Moreover, the asymptotic smoothness brought by the ensemble learning and the property of having many tunable hyperparameters further contribute to the improvement of accuracy.

\textit{(iii)} The satisfactory performance of the two-stage best-scored random forest is supported by compact theoretical analysis under the framework of regularized empirical risk minimization. To be specific, by decomposing the error term into data-free and data-dependent error terms which are dealt with by techniques from approximation theory and empirical process theory, respectively, we establish the almost optimal learning rates for both parent best-scored random trees and their ensemble forest under certain mild assumptions on the smoothness of the target functions.

The paper is organized as follows: Section \ref{sec::MainAlgorithm} is dedicated to the explanation on the algorithm architecture. We present the main results and statements on the almost optimal learning rates in Section \ref{sec::MainResults} with the corresponding error analysis lucidly demonstrated in Section \ref{sec::ErrorAnalysis}. Architecture analysis and empirical assessments of comparisons between different vertical methods based on real data sets are provided in Sections \ref{sec::ArchitectureAnalysis} and \ref{sec::NumericalEvaluation}.
For the sake of clarity, all the proofs of Section \ref{sec::MainResults} and Section \ref{sec::ErrorAnalysis} are presented in Section \ref{sec::Proof}. Finally, we conclude this paper in Section \ref{sec::Conclusion}.

\section{Establishment of the Main Algorithm} \label{sec::MainAlgorithm}
In this section, we propose a new random forest method for regression which gathers the advantages of vertical methods and ensemble learning. A lucid illustration requires to break down the algorithm into four steps. 
First, we adopt an adaptive random partition method to split the feature space into several cells in stage one.
Second, by building the best-scored random tree for regression on each cell in stage two and gathering them together, we are able to obtain a parent random tree.  
Third, due to the intrinsic randomness of the partition method, we are able to establish a certain number of parent random trees under different partitions of the feature space.
Last but not least, by combining these parent random trees to form an ensemble,
we obtain the \emph{Two-stage Best-scored Random Forest}.

\subsection{Notations}
The goal in a supervised learning problem is to predict the value of an unobserved output variable $Y$ after observing the value of an input variable $X$. 
To be exact, we need to derive a predictor $f$ 
which maps the observed input value of $X$ 
to a prediction $f(X)$ of the unobserved output value of $Y$. 
The choice of predictor should be based on the training data $D := ((X_1, Y_1), \ldots, (X_n, Y_n))$ of i.i.d~observations, which are with the same distribution as the generic pair $(X, Y)$, drawn from 
an unknown probability measure $\mathrm{P}$ on $\mathcal{X} \times \mathcal{Y}$. 
We assume that $\mathcal{X} \subset \mathbb{R}^d$ is non-empty, 
$\mathcal{Y} := [-M, M]$ for some $M > 0$ and $\mathrm{P}_X$ is the marginal distribution of $X$.

According to the learning target, it is legitimate to consider the least squares loss
$L = L_{LS} : \mathcal{X} \times \mathcal{Y} \to [0, \infty)$ defined by $L (Y, f(X)) := (Y-f(X))^2$. Then, for a measurable decision function $f : \mathcal{X} \to \mathcal{Y}$, the risk is defined by
\begin{align*}
\mathcal{R}_{L, \mathrm{P}}(f) 
:= \int_{\mathcal{X} \times \mathcal{Y}} L(Y, f(X))\ d\mathrm{P}(X, Y), 
\end{align*}
and the empirical risk is defined by
\begin{align*}
\mathcal{R}_{L, \mathrm{D}}(f) 
:= \frac{1}{n} \sum_{i=1}^n L(Y_i, f(X_i)),
\end{align*}
where $\mathrm{D}:= \frac{1}{n} \sum_{i=1}^n \delta_{(X_i, Y_i)}$ is the empirical measure associated to data and $\delta_{(X_i, Y_i)}$ is the Dirac measure at $(X_i, Y_i)$. The Bayes risk which is the minimal risk with respect to $\mathrm{P}$ and $L$ can be given by
\begin{align*}
\mathcal{R}^*_{L, \mathrm{P}} 
:= \inf \{ \mathcal{R}_{L, \mathrm{P}} (f) \mid f : \mathcal{X} \to \mathcal{Y} \ \text{measuarable}\}.
\end{align*}
In addition, a measurable function $f^*_{L, \mathrm{P}} : \mathcal{X} \to \mathcal{Y}$ with $\mathcal{R}_{L, \mathrm{P}} (f^*_{L, \mathrm{P}}) = \mathcal{R}^*_{L, \mathrm{P}}$  is called a Bayes decision function. By minimizing the risk, the Bayes decision function is
\begin{align*}
f^*_{L, \mathrm{P}} = \mathbb{E}_\mathrm{P} (Y \arrowvert X)
\end{align*}
which is a $\mathrm{P}_X$-almost surely $[-M, M]$-valued function.

In order to achieve our two-stage random forest for regression, we first consider the development of the parent random forest under one specific feature space partition. 
Therefore,
we assume that $(V_j)_{ j=1, \ldots, m}$ is a partition of $\mathcal{X}$ such that none of its cells is empty, which is $V_j \neq \emptyset$ for every $j \in \{ 1, \ldots, m\}$. 
To present our approach 
in a clear and rigorous mathematical expression, there is a need for us to introduce some more definitions and notations. First of all, the index set is defined as
\begin{align*}
I_j 
:=\big\{ i \in \{1, \ldots, n\} : x_i \in V_j  \big\},
\quad \quad j = 1, \ldots, m,
\end{align*}
which indicates the samples of $D$ contained in $V_j$ and also the corresponding data set
\begin{align*}
D_j 
:= \big\{ (x_i, y_i) \in D : i \in I_j \big\}, 
\quad \quad j = 1, \ldots, m.
\end{align*}
Additionally, for every $j \in \{ 1, \ldots, m \}$,
the loss $L_j : \mathcal{X} \times \mathcal{Y} \to [0, \infty )$ on the corresponding cell $V_j$ is defined by
\begin{align*}
L_j( Y, f(X) ) := \boldsymbol{1}_{V_j} (X) L( Y, f(X) ) 
\end{align*}
where $L( Y, f(X) )$ is the least squares loss for our regression problem.

\subsection{Best-scored Random Trees } \label{LPRF}
One crucial step of the two-stage best-scored random forest algorithm is building parent best-scored random trees under certain partitions of the feature space. 
Therefore, we first focus on the development of one parent tree which is the summation of child trees.
An appropriate splitting approach of the feature space is inseparable for the tree establishment.
Therefore, we introduce a random partition method in our case.

\subsubsection{Purely Random Partition} \label{sec::RandomPartition}
Purely random forest put forward by \cite{breiman00a} is an algorithm parallel to forests based on well-known splitting criteria such as information gain \citep{quinlan1986induction}, information gain ratio \citep{quinlan1993c4.5} and Gini index \citep{breiman1984classification}. Since it is widely acknowledged that forest established by the latter three criteria are not universal consistent, while consistency can be obtained by the first one, we base our forest on this purely random splitting criterion.

A clear illustration of the splitting mechanism at the $i$-th step of one possible random tree construction requires a random vector $Q_i : = (L_i, R_i, S_i)$. The first term $L_i$ in the triplet 
denotes the leaf to be split at the $i$-th step chosen from all the leaves presented in the $(i-1)$-th step. The second term in the triplet $R_i \in \{ 1, \ldots, d\}$ represents the dimension chosen to be split from for the $L_i$ leaf. Moreover, $\{R_i, i \in \mathbb{N}_+\}$ are i.i.d.~multinomial random variables with all dimension having equal probability to be split from.
The third term $S_i$ stands for the ratio of the length in the $R_i$-th dimension of the newly generated leaf after the $i$-th split to the length in the $R_i$-th dimension of leaf $L_i$, which is a proportional factor. In this manner, the length in the $R_i$-th dimension of the newly generated leaf can be calculated by multiplying the length in the $R_i$-th dimension of leaf $L_i$ and the proportional factor $S_i$. We mention here that $\{S_i, i \in \mathbb{N}_+\}$ are independent and identically distributed from $\mathcal{U}(0, 1)$. 

To provide more insight into the above mathematical formulation of the splitting process of the purely random tree, we take the tree construction on $A = [0, 1]^d$ as a simple example, which is the same for construction on $V_j$. One specific construction procedure is shown in Figure \ref{fig:ap} where we take $d=2$. First of all, we pick up a dimension out of $d$ candidates randomly, and then split uniformly at random from that dimension. The resulting split being a $(d-1)$-dimensional hyperplane parallel to the axis partitions $A$ into two leaves, say $A_{1, 1}$ and $A_{1, 2}$. Next, a leaf is chosen uniformly at random, e.g.~$A_{1, 1}$, and we go on picking the dimension and the cut-point uniformly at random to implement the second split, which leads to a partition of $A$: $A_{2, 1}, A_{2,2}, A_{1,2}$. When conducting the third split, we still randomly select one leaf presented in the last step, e.g.~$A_{2, 2}$, and the third split is once again conducted on it as before. The resulting partition of $A$ then becomes $A_{2,1}, A_{3,1}, A_{3,2}, A_{1,2}$. This above recursive process will not stop until the number of splits $p$ reaches out satisfaction. Further scrutiny will find that the splitting procedure leads to a partition variable, namely $Z:=(Q_1, \ldots, Q_p,\ldots)$ which takes value in space $\mathcal{Z}$. From now on, $\mathrm{P}_Z$ stands for the probability measure of $Z$.

\begin{figure*}[htbp]
	\begin{minipage}[t]{0.8\textwidth}  
		\centering  
		\includegraphics[width=\textwidth]{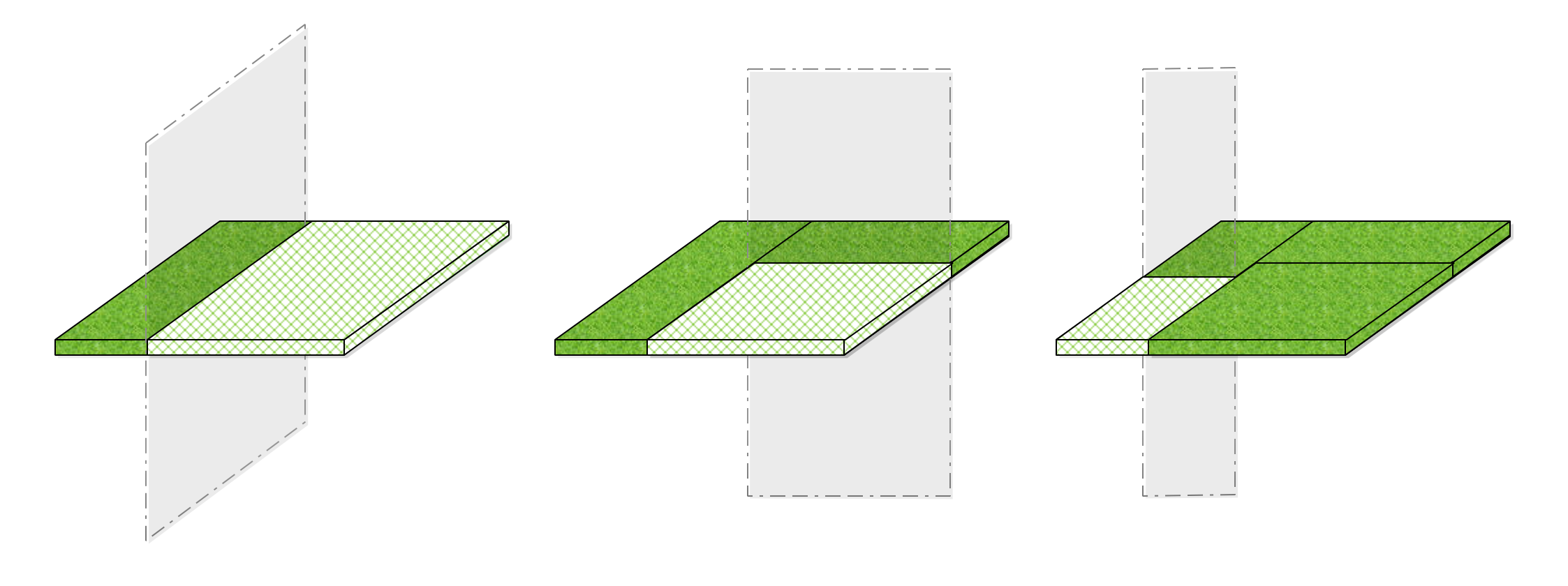}  
	\end{minipage}  
	\centering  
	\caption{Possible construction procedures of $3$-split axis-parallel purely random partitions in a $2$-dimensional space. }
	\label{fig:ap}
\end{figure*}

It is legitimate to assume that any specific partition variable $Z \in \mathcal{Z}$ can be recognized as a latent splitting criterion. To be specific, if we consider a $p$-split procedure carried out by following $Z$, then the collection of the resulting non-overlapping leaves can be defined by $\mathcal{A}_{(Q_1 \ldots Q_p)}$, and further abbreviated as $\mathcal{A}_{Z, p}$. Now, if we focus on the partition on certain cell $V_j$, for example, then we have $\mathcal{A}_{Z, 0}:=V_j$. Moreover, for any point $X \in V_j$, it is bound to fall into certain cell which can then be denoted by $A_{Z, p}(X)$.

Here, we introduce a map $h_{Z, p} : V_j \to \mathcal{Y}$ defined by
\begin{align} \label{TreeDecisionRule}
h_{Z, p} (X) : =
\frac{\sum_{i=1}^n Y_i \eins_{\{X_i \in A_{Z, p} (X) \}}}{\sum_{i=1}^n \eins_{\{X_i \in A_{Z, p} (X) \}}} \eins_{E_{Z, p}(X)},
\end{align}
where the event set $E_{Z, p}(X)$ is defined by
\begin{align*}
E_{Z, p}(X) = \bigg\{ \sum_{i=1}^n \eins_{\{X_i \in A_{Z, p}(X)\}} \neq 0 \bigg\}.
\end{align*}
Formula \eqref{TreeDecisionRule} is called the random tree decision rule for regression on $V_j$.

\subsubsection{Child Best-scored Random Tree} \label{sec::LBRT}
In this subsection, we consider the establishing procedure of a child best-scored random tree defined on the feature space $\mathcal{X}$. Specifically, the child random tree is originally developed on $V_j$ and then extended to $\mathcal{X}$. Concerning with the fact that the performance of the tree obtained by conducting random partition once may not be that desirable, we improve this by choosing one tree with the best performance out of $k_j$ candidates on $V_j$.
The tree picked out is then called the child best-scored random tree.
Therefore, when analyzing the behaviors of $k_j$ trees on $V_j$, we suppose that splitting procedures they follow can be represented by the independent and identically distributed random variables $\{ Z_{j 1}, \ldots, Z_{j k_j}\}$ drawn from $\mathcal{Z}$, respectively. 

For clearer illustration of theoretical analysis, we first give the definitions of some function sets.
We assume that $\mathcal{T}_j$ is a function set containing all the possible partitions of a random tree over $V_j$, which is defined as follows:
\begin{align} \label{NaiveFunctionSet}
\mathcal{T}_j 
:= \biggl\{ \sum_{i=0}^p c_i \boldsymbol{1}_{A_{i}} : p \in \mathbb{N}, c_i \in [-M, M], \bigcup_{i=0}^p A_{i} = V_j,  A_{s} \cap A_{s'} = \emptyset, s \neq s' \biggr\}.
\end{align}
Here, we choose $p \in \mathbb{N}$ as the number of splits, the resulting leaves presented as $A_{0}, A_{1}, \ldots,$ $A_{p}$ actually form a $p$-split partition of $V_j$. 
It is important to notify that $c_i$ is the value of leaf $A_{i}$.
Without loss of generality, in this paper, we only consider cells with the shape of $A_{i} = \bigtimes_{\ell=1}^d [a_{i\ell}, b_{i\ell}]$.
Moreover, 
for $s \in \{1, \ldots, k_j\}$, we derive the function set induced by the splitting policy $Z_{j s}$ as
\begin{align} \label{CandidateSpace}
\mathcal{T}_{Z_{j s}, p} 
:= \bigg\{\sum_{i=0}^{p} c_{i} \boldsymbol{1}_{A_{i}} : c_{i} \in [-M, M], A_{i} \in \mathcal{A}_{Z_{j s}, p} \bigg\},
\end{align}
where $\mathcal{A}_{Z_{j s}, p}$ represents the resulting $p$-split partition of $V_j$ by following the splitting policy $Z_{j s}$. Note that $\mathcal{T}_{Z_{j s}, p} $ is a subset of $\mathcal{T}_j$.

However, we should notice that every function $h \in \mathcal{T}_{j}$ is only defined on $V_j$ while a random tree function from $\mathcal{X}$ to $\mathcal{Y}$ is finally needed. To this end, for every $h \in \mathcal{T}_{j}$, we define the zero-extension $g : \mathcal{X} \to \mathcal{Y}$ by
\begin{align} \label{ZeroExtension}
g(X) 
:= \begin{cases}
h(X), & X \in V_j,
\\
0, & X\notin V_j,
\end{cases}
\end{align}
which should be equipped with the same number of splits $p$ as the decision tree $h$. 
Then, the function set only defined on $V_j$ can also be extended to $\mathcal{X}$, that is 
\begin{align} \label{ExtendedFunctionSet}
\hat{\mathcal{T}}_{j}
:= \{ g: h \in \mathcal{T}_{j}\}.
\end{align}
Moreover, the extension of function set $\mathcal{T}_{Z_{j s}, p}$ can also be obtained with the same manner, which is
\begin{align} \label{ExtendedCandidatesSet}
\hat{\mathcal{T}}_{Z_{j s}, p} 
:= \{ g: h \in \mathcal{T}_{Z_{j s}, p} \}.
\end{align}
Furthermore, we denote $\hat{\mathcal{T}}_{Z_{j s}} := \cup_{p \in \mathbb{N}} \hat{\mathcal{T}}_{Z_{j s}, p}$.

In order to find an appropriate random tree decision rule under policy $Z_{j s}$ denoted as $g_{L_j, \mathrm{D}, Z_{js}}$, we are supposed to conduct an optimization problem. To this end, we conduct our analysis under the framework of regularized empirical risk minimization.
To begin with, regularized empirical risk minimization is a learning method providing us with a better preparation for more involved analysis of our specific random forest.
Let $L : \mathcal{X} \times \mathcal{Y} \to [0, \infty)$ be a loss and $\mathcal{T} \subset \mathcal{L}_0 ( \mathcal{X} )$ be a non-empty set, where $\mathcal{L}_0$ is the set of measurable functions on $\mathcal{X}$ and $\Omega : \mathcal{T} \to [0, \infty)$ be a function. The learning method whose decision function $f_D$ satisfying
\begin{align*}
\mathcal{R}_{L, \mathrm{D}} (f_D) + \Omega (f_D) = \inf_{f \in \mathcal {T}} \mathcal{R}_{L, \mathrm{D}} (f) + \Omega (f)
\end{align*}
for all $n \ge 1$ and $D \in (\mathcal{X} \times \mathcal{Y})^n$ is named regularized empirical risk minimization.

In this paper, we propose that the number of splits $p$ is what we should penalize on. By penalizing on $p$, we are able to give some constraints on the complexity of the function set so that the set will have a finite VC dimension \citep{vapnik71a}, 
and therefore make the algorithm PAC learnable \citep{valiant84a}.
Besides, it can also refrain the learning results from overfitting.
With data set $D$, the above regularized empirical risk minimization problem with respect to each function set $\hat{\mathcal{T}}_{Z_{j s}}$ turns into
\begin{align} \label{RERM1}
\min_{p \in \mathbb{N}} \min_{g \in \hat{\mathcal{T}}_{Z_{j s}, p}} \ \lambda_j p^2 + \mathcal{R}_{L_j, \mathrm{D}} (g),
\quad
s = 1, \ldots, k_j.
\end{align}
It is well worth mentioning that since the exponent of $p$ will not have influence on the performance of the selection procedure, we penalize on $p^2$ to obtain better convergence properties.

Observation finds that the regularized empirical risk minimization under any policy can be bounded simply by considering the case where no split is applied to $V_j$. Consequently, we present the optimization problem as follows:
\begin{align*}
\min_{p \in \mathbb{N}} \min_{g \in \hat{\mathcal{T}}_{Z_{j s}, p}} \ \lambda_j p^2 + \mathcal{R}_{L_j, \mathrm{D}} (g)
\le \mathcal{R}_{L_j, \mathrm{D}} (0) \le M^2,
\end{align*}
where $\mathcal{R}_{L_j, \mathrm{D}} (0)$ stands for the empirical risk for taking $g(x) = 0$ for all $x \in \mathcal{X}$ with $p = 0$. Therefore, from the above inequality, we obtain that the number of splits $p$ is upper bounded by $M \lambda_j^{-1/2}$. Accordingly, the capacity of the underlying function set can be largely reduced, and here and subsequently, the function sets will all be added an extra condition where $p \leq M \lambda_j^{-1/2}$.

To establish the random tree decision rule for regression on $\mathcal{X}$, we zero-extend  \eqref{TreeDecisionRule} to the whole feature space.
It can be apparently observed that our random tree decision rule on $\mathcal{X}$ induced by $V_j$ is the solution to the  optimization problem \eqref{RERM1} and it can be further denoted by
\begin{align} \label{RERM2}
(g_{Z_{j s}}, \ p_{Z_{j s}})
= \argmin_{p \in \mathbb{N}} \ \argmin_{g \in \hat{\mathcal{T}}_{Z_{j s}, p}} \ \lambda_j p^2 + \mathcal{R}_{L_j, \mathrm{D}} (g),
\quad 
s = 1, \ldots, k_j,
\end{align}
where $p_{Z_{j s}}$ is the number of splits of the decision function $g_{Z_{j s}}$.
Its population version is presented by
\begin{align*} 
(g^*_{Z_{j s}},\ p^*_{Z_{j s}}) = \argmin_{p \in \mathbb{N}} \  \argmin_{g \in \hat{\mathcal{T}}_{Z_{j s}, p}} \ \lambda_j p^2 + \mathcal{R}_{L_j, \mathrm{P}} (g), 
\quad 
s = 1, \ldots, k_j.
\end{align*}
It is necessary to note that our primary idea is to conduct the regularized empirical risk minimization problem using $\mathcal{T}_{Z_{j s}}$ and $D_j$, which is
\begin{align} \label{RERM3}
(h_{L, \mathrm{D}_j, Z_{j s}}, p_{L, \mathrm{D}_j, Z_{j s}}) := \argmin_{p \in \mathbb{N}}\  \argmin_{h \in \mathcal{T}_{Z_{j s}, p}} \ \tilde{\lambda}_j p^2 + \mathcal{R}_{L, \mathrm{D}_j} (h),
\quad 
s = 1, \ldots, k_j. 
\end{align}
It can be observed that when we take $\tilde{\lambda}_j := n\lambda_j/\lvert I_j \rvert $, the solution of the optimization problem \eqref{RERM3} coincides with \eqref{RERM2} on $V_j$.
Since the following analysis will be carry out on $\mathcal{X}$, we can directly optimize \eqref{RERM2}.
Furthermore, it is easy to verify that if a Bayes decision function $f_{L, \mathrm{P}}^*$ w.r.t.~$L$ and $\mathrm{P}$ exists, it additionally is a Bayes decision function w.r.t.~$L_j$ and $\mathrm{P}$.

Now, we focus on establishing
the best-scored random tree on $\mathcal{X}$ induced by $V_j$, also called the child best-scored random tree,
which is chosen from $k_j$ candidates. The main principle is to retain only the tree yielding the minimal regularized empirical risk, which is
\begin{align}  \label{LocalBestScoredTree}
(g_{Z_{j}}, \ p_{Z_{j}})
= \argmin_{s=1, \ldots, k_j} \ \lambda_j p^2_{Z_{j s}} + \mathcal{R}_{L_j, \mathrm{D}} (g_{Z_{j s }}),
\end{align}
where $p_{Z_j}$ is the number of splits of $g_{Z_j}$ and $Z_j = \{ Z_{j1}, \ldots, Z_{j k_j} \}$. 
Apparently, $g_{Z_j}$ is the regularized empirical risk minimizer with respect to the random function set
\begin{align} \label{SpaceK}
\hat{\mathcal{T}}_{Z_j} : = \bigcup_{s=1}^{k_j}\hat{\mathcal{T}}_{Z_{j s}}.
\end{align}
Put another way, $g_{Z_j}$ is the solution to the regularized empirical risk minimization problem
\begin{align*}
\min_{g \in \hat{\mathcal{T}}_{Z_j}} \ \lambda_j p^2(g) + \mathcal{R}_{L_j, \mathrm{D}} (g)
: = \min_{s=1, \ldots, k_j} \ \min_{p \in \mathbb{N}} \ \min_{g \in \hat{\mathcal{T}}_{Z_{j s}, p}} \ \lambda_j p^2 + \mathcal{R}_{L_j, \mathrm{D}}(g).
\end{align*}
Similar as it is, we denote by $g^*_{Z_j}$ the solution of the population version of regularized minimization problem in the set $\hat{\mathcal{T}}_{Z_j}$
\begin{align} 
( g^*_{Z_j}, \, p^*_{Z_j}) 
& = \argmin_{g \in \hat{\mathcal{T}}_{Z_j}} \ \lambda p^2(g) + \mathcal{R}_{L_j, \mathrm{P}} (g) 
\nonumber\\
& = \argmin_{s = 1, \ldots, k_j}\ \lambda (p^*_{Z_{j s}})^2 + \mathcal{R}_{L_j, \mathrm{P}}(g^*_{Z_{j s}}).
\label{BestScoreMinimizerPopulation}
\end{align}
We mention here that $p^*_{Z_j}$ is the corresponding number of splits of 
$g^*_{Z_j}$.

\subsubsection{Parent Best-scored Random Tree} \label{sec::WeightedGRT}
In this subsection, we first build the parent random tree by adding all the child ones.
After that, in order to show that our parent random tree is indeed a solution of an usual random tree algorithm on the feature space, we need to consider the indicator function sets defined on $\mathcal{X}$ of a child random tree and direct sums of the indicator function sets of several trees.

First of all, adding all child best-scored random trees generated by \eqref{LocalBestScoredTree} together leads to the parent best-scored random tree, which is defined by
\begin{align} \label{GlobalDecisionTree}
g_{Z} (X)
:= \sum_{j=1}^m g_{Z_{j }} (X),
\end{align}
where $Z :=\{ Z_{1}, \ldots, Z_{m} \}$ denotes the splitting criteria on $\{V_j\}_{j=1}^m$.

Recall that we have mention the process of extending
the indicator function set of a tree on $V \subsetneq \mathcal{X}$ to an indicator function set on $\mathcal{X}$ in \eqref{ZeroExtension} and \eqref{ExtendedFunctionSet}, we now give a formal description of that in
the following proposition.

\begin{proposition}\label{SpaceExtension}
	Let $V \subset \mathcal{X}$ and $\mathcal{T}_V$ be an indicator function space of the form \eqref{NaiveFunctionSet} on $V$.	Denote by $g$ 
	the zero-extension of $h \in \mathcal{T}_V$ to $\mathcal{X}$ defined by
	\begin{align*}
	g(X) 
	:= \begin{cases}
	h(X), & X \in V,
	\\
	0, & X\in \mathcal{X} \backslash V.
	\end{cases}
	\end{align*}
	Then, the set $\hat{\mathcal{T}}_V := \{ g : h \in \mathcal{T}_V \}$ is still an indicator function set on $\mathcal{X}$. We define that the number of splits of the decision tree on $\hat{\mathcal{T}}_V$ is the same as the number of splits on $\mathcal{T}_V$, which is
	\begin{align} \label{LocalSplitNumber}
	p_{\hat{\mathcal{T}}_V} := p_{\mathcal{T}_V}.
	\end{align}
\end{proposition}

Based on this proposition, we are now able to construct an indicator function set by a direct sum of indicator function sets $\hat{\mathcal{T}}_V$ and $\hat{\mathcal{T}}_W$ 
with $V, W \subset \mathcal{X}$ and $V \cap W = \emptyset$. 

\begin{proposition} \label{DirectSumSpace}
	For $V, W \subset \mathcal{X}$ such that $V \cap W =\emptyset$ and 
	$V \cup W \subset \mathcal{X}$, let $\mathcal{T}_V$ and $\mathcal{T}_W$ be indicator function sets of the form \eqref{NaiveFunctionSet} on $V$ and $W$, respectively. Furthermore, let $\hat{\mathcal{T}}_V$ and $\hat{\mathcal{T}}_W$ be the indicator function sets of all functions of $\mathcal{T}_V$ and $\mathcal{T}_W$ extended to $\mathcal{X}$ in the sense of Proposition \ref{SpaceExtension}. Let $p_{\hat{\mathcal{T}}_V}$ and 
	$p_{\hat{\mathcal{T}}_W}$ given by \eqref{LocalSplitNumber} be the associated the number of splits. Then $\hat{\mathcal{T}}_V \cap \hat{\mathcal{T}}_W = \{0\}$ and hence
	the direct sum 
	\begin{align*}
	\mathcal{T} := \hat{\mathcal{T}}_V \oplus \hat{\mathcal{T}}_W
	\end{align*}
	exists. The direct sum $\mathcal{T}$ is also an indicator function set of random trees. 
	For $\lambda_V, \lambda_W > 0$ and $g \in \mathcal{T}$, 
	let $g_V \in \hat{\mathcal{T}}_V$ and $g_W \in \hat{\mathcal{T}}_W$ be the unique function that $g = g_V + g_W$.
	Then, we define the number of splits on the direct sum space by
	\begin{align*}
	p_{\mathcal{T}} 
	:= (\lambda_{V} p_{\hat{\mathcal{T}}_V}^2 + \lambda_{W} p_{\hat{\mathcal{T}}_W}^2)^{1/2}.
	\end{align*}
\end{proposition}

To relate Proposition \ref{SpaceExtension} and Proposition \ref{DirectSumSpace} with 
\eqref{GlobalDecisionTree}, there is a need to introduce more notations. For pairwise disjoint 
$V_1, \ldots, V_m \subset \mathcal{X}$ with $\cup_{j=1}^m V_j = \mathcal{X}$, let $\hat{\mathcal{T}}_{Z_{j}}$ be the best-scored function space \eqref{SpaceK} induced by $V_j$ for every $j \in \{1, \ldots, m\}$ based on  Proposition \ref{SpaceExtension}. A joined indicator function space of $\hat{\mathcal{T}}_{Z_{1}}, \ldots, \hat{\mathcal{T}}_{Z_{m}}$ can be therefore designed analogously to Proposition \ref{DirectSumSpace}. Specifically, for an arbitrary index set $J \subset \{1, \ldots, m\}$ and a vector 
$\boldsymbol{\lambda} = (\lambda_j)_{j \in J} \in (0, \infty)^{\lvert J \rvert}$,
the direct sum
\begin{align*}
\mathcal{T}_{Z_{J }} := \bigoplus_{j \in J} \hat{\mathcal{T}}_{Z_{j}}
=\Bigg\{ g=\sum_{j \in J} g_j : g_j \in \hat{\mathcal{T}}_{Z_{j }} \ \text{for all} \  j\in J  \Bigg\},
\end{align*}
where $Z_{J } := \{Z_{j }: j \in J\}$,
is still an indicator function space of random tree with squared number of splits
\begin{align} \label{GlobalDepth}
p_{\mathcal{T}_{Z_{J }}}^2 = \sum_{j \in J} \lambda_j p_{\hat{\mathcal{T}}_{Z_{j }}}^2 .
\end{align}
If $J = \{1, \ldots, m\}$, we simply write $\mathcal{T}_{Z} := \mathcal{T}_{Z_{J }}$.
To notify, $\mathcal{T}_{Z}$ contains inter alia $g_{Z}$ given by \eqref{GlobalDecisionTree}.

Here, we briefly investigate the regularized empirical risk of 
$g_{Z} 
= \sum_{j=1}^m  g_{Z_{j}} $.
For arbitrary $g \in \mathcal{T}_{Z}$, we have
\begin{align}
\mathcal{R}_{L, \mathrm{D}}(g_{Z}) + p^2(g_{Z} )
& = \sum_{j=1}^m \mathcal{R}_{L_j, \mathrm{D}} (g_{Z} ) 
+ p^2(g_{Z})
= \sum_{j=1}^m \big( \mathcal{R}_{L_j, \mathrm{D}} (g_{Z_j} ) 
+  \lambda_j p^2_{Z_j}  \big)
\nonumber\\
& \leq \sum_{j=1}^m \big( \mathcal{R}_{L_j, \mathrm{D}} (\eins_{V_j} g ) 
+ \lambda_j p^2({\eins_{V_j} g})  \big) 
= \sum_{j=1}^m  \mathcal{R}_{L_j, \mathrm{D}} ( g ) 
+ p^2( g)  
\nonumber\\
& = \mathcal{R}_{L, \mathrm{D}}(g) + p^2(g).
\label{UsualDecisionTree}
\end{align}
The first equality is derived by
$\mathcal{R}_{L, \mathrm{D}} (g) = \sum_{j=1}^m \mathcal{R}_{L_j, \mathrm{D}} (g)$ \citep{meister16a}.
The second equality is established because the risk of $g_{Z}$ on $V_j$ equals that of $g_{Z_j}$.
The inequality is a direct result of \eqref{LocalBestScoredTree}, where the number of splits 
$p(g)$ for arbitrary $g \in \mathcal{T}_{Z}$
according to Proposition \ref{DirectSumSpace} is defined by 
$p^2(g) := \sum_{j=1}^m \lambda_j p^2(\eins_{V_j} g)$,
and 
$p(\eins_{V_j} g)$ is the corresponding number of splits of $g$ on $V_j$. 
The last two equalities hold the same ways as the first two ones.

Judging from \eqref{UsualDecisionTree}, $g_{Z}$ is the random tree function with respect to $\mathcal{T}_{Z}$ and $L$, as well as the 
regularized parameter $\lambda= 1$. In other words, the latter best-scored random tree derived from
$\argmin_{g \in \mathcal{T}_{Z}}\mathcal{R}_{L, \mathrm{D}}(g) + p^2(g)$
equals our parent best-scored random tree \eqref{GlobalDecisionTree}.

For the sake of clarity, we summarize some assumptions for the joined best-scored function sets as follows:

\begin{assumption}[Joined best-scored decision tree spaces] \label{JoinedSpace}
	For pairwise disjoint subsets $V_1, \ldots, V_m$ of $\mathcal{X}$, let $\hat{\mathcal{T}}_{Z_{j }}$ be the best-scored random tree function sets induced by $V_j$. Consequently, for $\boldsymbol{\lambda} := (\lambda_1, \ldots, \lambda_m) \in (0, \infty)^m$, we define the joined best-scored function space $\mathcal{T}_{Z} := \bigoplus_{j=1}^m \hat{\mathcal{T}}_{ Z_{j}}$ and equip it with the number of splits \eqref{GlobalDepth}.
\end{assumption}

\subsection{Two-stage Best-scored Random Forest} \label{sec::TBRF}
Having developed the parent random tree under one specific partition of the feature space, it is legitimate to ponder whether we can devise an ensemble of trees by injecting randomness into  the feature partition in stage one.
To fulfill this idea, we propose a data splitting approach named as the adaptive random partition and establish the \emph{Two-stage Best-scored Random Forest} by ensemble learning.

\subsubsection{Adaptive Random Partition of the Feature Space} \label{sec::AdaptivePartition}

To describe the above two-stage random forest algorithm, $(V_j)_{j=1, \ldots, m}$ only has to be some partition of $\mathcal{X}$. Nevertheless, concerning with the theoretical investigations that will be conducted on the learning rates of our new algorithm, there is a need for us to further specify the partition. 
For this purpose, 
we denote 
a series of balls $B_1, \ldots, B_m$ 
with radius $r_j > 0, j = 1, \ldots, m$ and mutually distinct centers 
$z_1, \ldots, z_m \in \mathcal{X} $
by
\begin{align*}
B_j 
:= B_{r_j}(z_j) 
:= \{x \in \mathbb{R}^d : \Vert x - z_j \Vert_2 \leq r_j\}, 
\quad \quad j \in\{1, \dots, m\}.
\end{align*}
where $\Vert \cdot \Vert_2$ is the Euclidean norm in $\mathbb{R}^d$.
Furthermore, we can choose $r_1, \ldots, r_m$ and $z_1, \ldots, z_m$ such that $\mathcal{X} \subset \bigcup_{j=1}^m B_j$.

Considering how large the sample size will be and how the sample density may vary in the feature space $\mathcal{X}$, we propose an adaptive random partition approach.
This method serves as a preprocessing of partitioning the feature space into cells containing fewer data which facilitates the following regression works on cells. Moreover, owing to the randomness resided in the partition, it paves the way for ensemble.
A considerable advantage of this proposal over the purely random partition is that it efficiently takes the sample information into consideration. To be precise, since the construction of the purely random partition is independent of the whole data set, it may possibly suffer from the dilemma where there is over-splitting on sample-sparse area and under-splitting on sample-dense area. However, the adaptive random partition is much wiser for it utilizes sample information in a relatively easy way and still fulfills the objective of dividing the space into small cells. The specific partition procedure is similar to the proposed process in Section \ref{sec::RandomPartition} with difference in how to choose the to-be-split cell.

In the purely random partition, $L_i$ in the random vector $Q_i := (L_i, R_i, S_i)$ denotes the randomly chosen cell to be split at the $i$-th step of tree construction. Here, we propose that when choosing a to-be-split cell, we first randomly select $t$ sample points from the training data set who are then labeled by the cells they belong to. Later, we choose the cell that is the majority vote of the $t$ sample labels to be $L_i$. This idea follows the fact that
when randomly picking sample points from the whole training data set, cells with more samples will be more likely to be selected while cells with fewer samples are less possible to be chosen.
In this manner, we may obtain feature space partitions where the sample sizes of resulting cells are more evenly distributed.

\subsubsection{Ensemble Forest} \label{sec::EnsembleForest}
We now construct the two-stage best-scored random forest basing on the average results of $T$ parent best-scored random trees.
Due to the intrinsic randomness resided in the partition method, we are able to construct several different parent best-scored random trees under different partitions of the feature space.
To be specific, each of these trees is generated according to the procedure in \eqref{GlobalDecisionTree} under different input partition $V^t := \{V_j^t\}_{j=1}^m$, $t=1,\ldots,T$. 
To clarify, the splitting criterion for each of the tree in the forest is denoted by $Z_t = \{Z_{1 t}, \ldots, Z_{m t}\}$, $t =1, \ldots, T$, where $Z_{j t}$ is already the splitting criterion corresponding to the child best-scored random tree for the $t$-th tree on its $V_j^t$.
Moreover, we denote the parent best-scored trees in the forest as $g_{Z_{t}},\ 1 \leq t \leq T$. 
As usual, we perform average to obtain the two-stage best-scored random forest decision rule
\begin{align} \label{TBRF}
f_{Z} (X)
:= \frac{1}{T}\sum_{t = 1}^T g_{Z_t} (X),
\end{align}
where $Z= \{ Z_1, \ldots, Z_T \}$ denotes the collection of all splitting criteria of trees in the forest. Finally, we establish our large-scale regression predictor, the two-stage best-scored random forest $f_{Z}$.

\section{Main Results and Statements}\label{sec::MainResults}
In this section, we present main results on the oracle inequalities and learning rates for the random trees and forests. 

\subsection{Fundamental Assumption}
In this paper, we are interested in the ground-truth functions that satisfy the following restrictions on their smoothness:

\begin{assumption} \label{Smoothness}
	The Bayes decision function $f_{L, \mathrm{P}}^* : \mathcal{X} \to \mathcal{Y}$ is $\alpha$-H\"{o}lder continuous with respect to $L_1$-norm $\|x\|_1 : = \sum_{i=1}^d |x_i|$. That is, there exists a constant $c_\alpha > 0$ such that
	\begin{align*}
	| f_{L, \mathrm{P}}^*(x) - f_{L, \mathrm{P}}^*(z)|
	\leq c_\alpha \|x - z\|_1^\alpha,
	\quad
	\forall x, z \in \mathcal{X}.
	\end{align*}
\end{assumption}

\subsection{Oracle Inequality for Parent Best-scored Random Trees}
We now establish an oracle inequality for parent best-scored random trees based on the least squares loss and best-scored function space. 

\begin{theorem} \label{OracleVPtree}
	Let $\mathcal{Y} : = [-M, M]$ for $M > 0$, $L : \mathcal{X} \times \mathcal{Y} \to [0, \infty)$ be the least squares loss, $\mathrm{P}_{X \times Y} : = \mathrm{P}$ be the probability measure on $\mathcal{X} \times \mathcal{Y}$ and $\mathrm{P}_Z$ be the probability measure induced by the splitting criterion $Z$. Then for all $\tau > 0$, $\lambda : = (\lambda_1, \ldots, \lambda_m) > 0$ and $\delta \in (0, 1)$, the parent best-scored random tree \eqref{GlobalDecisionTree} satisfies
	\begin{align*}
	p^2 (g_{Z}) + \mathcal{R}_{L, \mathrm{P}} (g_{Z}) - \mathcal{R}_{L, \mathrm{P}}^* 
	& \leq 9 ( p^2 \big(g^*_{Z}) + \mathcal{R}_{L, \mathrm{P}} (g^*_{Z}) - \mathcal{R}_{L, \mathrm{P}}^* \big) \\
	& \phantom{=}
	+ 3 c_{d \delta M} \bigg(\frac{1}{n}\sum_{j=1}^m \lambda_j^{-1/2}\bigg)^{\frac{2}{1+2\delta}}
	+ 3456 M^2 \tau / n
	\end{align*}
	with probability $\mathrm{P}_{(X \times Y)| Z}$ at least $1 - 3 e^{-\tau}$, where $c_{d \delta M}$ is a constant depending on $d$, $\delta$ and $M$. The result holds for all parent best-scored random tree criterion $Z$.
\end{theorem}

\subsection{Learning Rates for Parent Best-scored Random Trees}
We now state our main result on the learning rates for parent best-scored random trees based on the established oracle inequality.

\begin{theorem} \label{RateVPTree}
	Let $L : \mathcal{X} \times \mathcal{Y} \to [0, \infty)$ be the least squares loss,
	$\mathrm{P}_{X \times Y} : = \mathrm{P}$ be the probability measure on $\mathcal{X} \times \mathcal{Y}$ and $\mathrm{P}_Z$ be the probability measure induced by the splitting criterion $Z$. Let $\{ V_j \in B_j \}_{j=1}^m$ be a partition of $\mathcal{X}$ and $k_j$ be the number of candidate trees on $V_j$. Suppose that the Bayes decision function 
	$f_{L, \mathrm{P}}^* : \mathcal{X} \to \mathcal{Y}$
	satisfies Assumption \ref{Smoothness} with exponent $\alpha$. Then for all $\tau > 0$ and $\delta \in (0,1)$, with probability $\mathrm{P}_{(X \times Y) \otimes Z}$ at least $1 - 4 e^{- \tau}$, there holds for the parent best-scored random tree \eqref{GlobalDecisionTree} that
	\begin{align*} 
	p^2 (g_{Z}) 
	+ \mathcal{R}_{L, \mathrm{P}} (g_{Z}) 
	- \mathcal{R}_{L, \mathrm{P}}^* 
	\leq
	C n^{-\frac{c_T \alpha}{c_T \alpha (1 + \delta) + 2d}},
	\end{align*}
	where $c_T = 0.22$ and $C$ depending on $\alpha, \tau, \delta, d, m, M$ and $\{r_j, k_j, \mathrm{P}_X (V_j) \}_{j=1}^m$. 
\end{theorem}

\subsection{Learning Rates for Two-stage Best-scored Random Forest}
We now present the main result on the learning rates for two-stage best-scored random forest in \eqref{TBRF}. This diverse and also accurate ensemble forest is based on the collection of parent best-scored random trees generated by different feature space partition.

\begin{theorem} \label{RateVPForest}
	Let $L : \mathcal{X} \times \mathcal{Y} \to [0, \infty)$ be the least squares loss, $\mathrm{P}_{X \times Y} : = \mathrm{P}$ be the probability measure on $\mathcal{X} \times \mathcal{Y}$ and $\mathrm{P}_Z$ be the probability measure induced by the splitting criterion $Z$. Let the collection of $T$ different partitions that generate the ensemble be $V : = \{V^t\}_{t=1}^T : = \{\{V_j^t\}_{j=1}^m\}_{t=1}^T$ and $k_j^t$ be the number of candidate trees on $V_j^t$. Suppose that the Bayes decision function $f_{L, \mathrm{P}}^* : \mathcal{X} \to \mathcal{Y}$ satisfies Assumption \ref{Smoothness} with exponent $\alpha$. Then, for all $\tau > 0$ and $\delta \in (0,1)$, with probability $\mathrm{P}_{(X \times Y) \otimes Z}$ at least $1 - 4 e^{- \tau}$, there holds
	\begin{align*}
	\mathcal{R}_{L, \mathrm{P}} (f_{Z}) - \mathcal{R}_{L, \mathrm{P}}^*
	\leq C n^{-\frac{c_T \alpha}{c_T \alpha (1 + \delta) + 2d}},
	\end{align*}
	where $C$ depending on $\alpha, \tau, \delta, d, m, M, T, \{\{r_j^t, k_j^t, \mathrm{P}_X (V_j^t) \}_{j=1}^m\}_{t=1}^T$ and $c_T = 0.22$.
\end{theorem}

According to the proof related to Theorem \ref{RateVPForest}, we find that the coefficient $C$ may decrease with the number of trees in the forest $T$ increasing. In other words, in theory, more trees may lead to smoother forest predictor and therefore, better learning rates. Moreover, this phenomenon is also supported by the experimental results shown later in Figure \ref{fig:syn_TBRF} where the predictor becomes smoother and has a better fit when $T$ increases.

\subsection{Comments and Discussions}
In this subsection, we present some comments and discussions on the obtained theoretical results on the oracle inequality, learning rates for the parent random trees and then for the two-stage best-scored random forest.

We highlight that our two-stage best-scored random forest algorithm aims at dealing with regression problems with enormous amount of data. To begin with, in the literature, vertical methods to deal with large-scale regression problem have gained its popularity owing to its capability of parallel computing. In this paper, we adopt a decision-tree like feature space splitting criterion named the adaptive random partition which is defined as the \emph{partition in stage one}. 
Moreover, the following partitions for conducting random trees on the resulting cells from stage one is called the \emph{partitions in stage two}, and  they follow a purely random splitting criterion.
In the literature, classical splitting criteria
such as information gain, information gain ratio and Gini index
have been scrutinized mostly from the perspective of experimental performance, while there are only a few of them concerning with theoretical learning rates, such as \cite{biau12a} and \cite{scornet15a}.
However, the conditions under which their learning rates are derived are too strong to testified in practical.
Compared to these classical splitting criteria, our purely random splitting criterion achieves satisfying learning rates only with some descriptions of the smoothness of the Bayes decision functions.

Second, we propose a novel idea in our model selection process, which is denominated as the \emph{best-scored} method. To clarify, choosing one random tree with the best regression performance out of several candidates helps to improve the accuracy of the base predictors. For a certain order of number of splits $p$, when the number of candidates $k$ is large enough, the function space generated by those trees will also be large enough to cover sufficient possible partition results. Consequently, the probability is high for us to choose the random tree with the best performance, which will lead to a remarkably small approximation error.

Third, the learning rate of one parent best-scored random tree is $O(n^{-c_T\alpha/(c_T\alpha(1+\delta)+2d)})$ 
and the learning rate of the two-stage best-scored random forest is with the same order. Here, we should notice that due to the intrinsic randomness of our splitting criterion, for a $p$-split random tree, the effective number of splits for each dimension is approximately $c_T\log p$ rather than $\log p$, where we take $c_T=0.22$. Moreover, since $\delta$ is concerned with the capacity of the partition function space and our function space is not that large, we can take $\delta$ as small as possible, even close to $0$.

In the machine learning literature, all kinds of vertical or horizontal regression methods have been studied extensively and understood. 
For example,
a vertical-like method mixing $k$-NN and SVM for regression is theoretically scrutinized by \cite{hable13a}. In his paper, 
for every testing point, the global SVM is applied to the $k$ nearest neighbors instead of to the whole training data. Moreover, a universal risk-consistency is provided.
In \cite{meister16a}, they derive the learning rate $O(n^{-2\alpha/(2\alpha+d)+\xi})$ of the localized SVM when the Bayes decision function is in a Besov-like space with $\alpha$-degrees of smoothness. 
As for large-scale regression problem with horizontal methods, \cite{zhang15a} proposes a divide and conquer kernel ridge regression and provide learning rates with respect to different kernels. With the Bayes decision function in the corresponding reproducing kernel Hilbert space (RKHS), they obtain a learning rate of $O(r/n)$ for kernels with finite rank $r$ and a learning rate of $O(n^{-2\nu/(2\nu+1)})$ for kernels in a Sobolev space with $\nu$-degrees of smoothness. Both of these prediction rates are minimax-optimal learning rates. 
\cite{lin17a} conducts a distributed learning with the least squares regularization scheme in a RKHS and obtains the almost optimal learning rates in expectation which are $O(n^{-2\alpha r/(4\alpha r + 1)})$. 
The learning rates are established 
under the smoothness assumption with respect to the $r$-th power of the integral operator $L_k$ and an $\alpha$-related capacity assumption.
\cite{guo17a} focuses on the distributed regression with bias corrected regularization kernel network and also obtains a learning rates of order $O(n^{-2r/(2r+\beta)})$ where $\beta$ is the capacity related parameter.
According to the above analysis, it can be seen that the work presented in our study has not only innovations but also complete theoretical supports.

\section{Error Analysis}\label{sec::ErrorAnalysis}
In this section, we give error analysis by bounding the approximation error term and the sample error term, respectively. 

\subsection{Bounding the Approximation Error Term}

Denote the population version of the parent best-scored random tree as
$$
g^*_{Z} : = \sum_{j=1}^m g^*_{Z_j}
$$
with $g^*_{Z_j}$ as in \eqref{BestScoreMinimizerPopulation}. 
The following theoretical result on bounding the approximation error term shows that, 
under smoothness assumptions for the Bayes decision function, 
the regularized approximation error
possesses a polynomial decay with respect to each
regularization parameter $\lambda_j$.

\begin{proposition} \label{ApproximationError}
	Let $L : \mathcal{X} \times \mathcal{Y} \to [0, \infty)$ be the least squares loss, 
	$\mathrm{P}_{X \times Y} : = \mathrm{P}$ be the probability measure on $\mathcal{X} \times \mathcal{Y}$ with marginal distribution $\mathrm{P}_X$,
	$\mathrm{P}_Z$ be the probability measure induced by the splitting criterion $Z$.
	Assume that $\{ V_j \in B_j\}_{j=1}^m$ is a partition of $\mathcal{X}$ and
	$k_j$ is the number of candidate trees on each $V_j$. 
	Suppose that the Bayes decision function 
	$f_{L, \mathrm{P}}^* : \mathcal{X} \to \mathcal{Y}$
	satisfies Assumption \ref{Smoothness} with exponent $\alpha$. Then, for any fixed $\tau > 0$ and $\lambda : = (\lambda_1, \ldots, \lambda_m) > 0$, with probability $\mathrm{P}_Z$ at least $1 - e^{- \tau}$, there holds that
	\begin{align*}
	p^2 (g^*_{Z}) + \mathcal{R}_{L, \mathrm{P}}(g^*_{Z}) - \mathcal{R}_{L, \mathrm{P}}^* 
	\leq
	c_{\alpha d}
	\sum_{j=1}^m
	\Bigl( \mathrm{P}_X(V_j) r_j^{2 \alpha} m^{2 \alpha / k_j} e^{2 \alpha \tau / k_j} \Bigr)^{\frac{4 d}{c_T \alpha + 4 d}}
	\lambda_j^{\frac{c_T \alpha}{c_T \alpha + 4 d}},
	\end{align*}
	where $c_{\alpha d}$ is a constant depending on $\alpha$ and $d$, $c_T = 0.22$ and $K$ is a universal constant.
\end{proposition}

\subsection{Bounding the Sample Error Term}

To establish the bounds on the sample error, we give four descriptions of the capacity of the function set in Definition \ref{VCdimension}, Definition \ref{covering numbers}, Definition \ref{entropy numbers} and Definition \ref{RademacherDefinition}. Then, we should analyze on the complexity of the regression function set so as to derive the sample error bounds. More specifically, the complexity of the random forest function set comes from two aspects which are one induced by the feature space partition and the other induced by value assignment. 

Firstly, we consider the complexity induced by partition. 
In that case, we might scrutinize the situation where there is a binary value assignment, i.e.~$\{-1, 1\}$. 
Particularly, we need to focus on its VC dimension (Lemma \ref{VCIndex}), covering numbers (Lemma \ref{BpTpCoveringNumbers}) and entropy numbers (Lemma \ref{HrEntropyNumber}). 
Secondly, there exists a relationship in terms of empirical Rademacher average between the binary value assignment induced complexity and the continuous value assignment induced complexity. Therefore, we are able to derive the empirical Rademacher average for regression in Lemma \ref{Rademacher}.

\begin{definition}[VC dimension] \label{VCdimension}
Let $\mathcal{B}$ be a class of subsets of $\mathcal{X}$ and $A \subset \mathcal{X}$ be a finite set.
The trace of $\mathcal{B}$ on $A$ is defined by $\{ B \cap A : B \in  \mathcal{B} \}$. 
Its cardinality is denoted by $\Delta^{\mathcal{B}}(A)$. 
We say that $\mathcal{B}$ shatters $A$ if $\Delta^{\mathcal{B}}(A) = 2^{\#(A)}$,
that is, if for every $\tilde{A} \subset A$, there exists a $B \subset \mathcal{B}$ such that
$\tilde{A} = B \cap A$. For $k \in \mathbb{N}$, let 
$$
m^{\mathcal{B}}(k)  := \sup_{A \subset \mathcal{X}, \, \#(A) = k}  \Delta^{\mathcal{B}}(A).
$$
Then, the set $\mathcal{B}$ is a Vapnik-Chervonenkis class if there exists $k < \infty$ such that $m^{\mathcal{B}}(k) < 2^k$ and the minimal of such $k$ is called the \emph{VC dimension} of $\mathcal{B}$, 
and abbreviated as $\mathrm{VC}(\mathcal{B})$.
\end{definition}

\begin{definition}[Covering Numbers] \label{covering numbers}
	Let $(X, d)$ be a metric space, $A \subset X$ and $\varepsilon > 0$.
	We call $\tilde{A} \subset A$ an $\varepsilon$-net of $A$ 
	if for all $x \in A$ there exists an $\tilde{x} \in \tilde{A}$ such that $d(x, \tilde{x}) \leq \varepsilon$.
	Moreover.
	the $\varepsilon$-covering number of $A$ is defined as
	\begin{align*}
	\mathcal{N} (A, d, \varepsilon) 
	= \inf \biggl\{ n \geq 1 : \exists x_1, \ldots, x_n \in X \text{ such that } A \subset \bigcup_{i=1}^n B_d (x_i, \varepsilon) \biggr\},
	\end{align*}
	where $B_d(x, \varepsilon)$ denotes the closed ball in $X$ centered at $x$ with radius $\varepsilon$.
\end{definition}

\begin{definition}[Entropy Numbers] \label{entropy numbers}
	Let $(X, d)$ be a metric space, $A \subset X$ and $n \geq 1$ be an integer.
	The $n$-th entropy number of $(A, d)$ is defined as
	\begin{align*}
	e_n(A, d) = \inf \biggl\{ \varepsilon > 0 : \exists x_1, \ldots, x_{2^{n-1}} \in X \text{ such that } A \subset \bigcup_{i=1}^{2^{n-1}} B_d(x_i, \varepsilon) \biggr\}.
	\end{align*}
\end{definition}

\begin{definition}[Empirical Rademacher Average] \label{RademacherDefinition}
Let $\{\varepsilon_i\}_{i=1}^m$ be a Rademacher sequence with respect to some distribution $\nu$,
that is, a sequence of i.i.d.~random variables such that 
$\nu(\varepsilon_i = 1) = \nu(\varepsilon_i = -1) = 1/2$.
The $n$-th empirical Rademacher average of $\mathcal{F}$ is
defined as
\begin{align*}
\mathrm{Rad}_\mathrm{D} (\mathcal{F}, n)
:= \mathbb{E}_{\nu} \sup_{h \in \mathcal{F}} 
\biggl| \frac{1}{n} \sum_{i=1}^n \varepsilon_i h(x_i) \biggr|.
\end{align*}
\end{definition}

Here, we first analyze the complexity of the function set of binary value assignment case. For fixed $p \in \mathbb{N}$,
we denote the collection of trees with number of splits $p$ as
\begin{align} \label{TpSpace}
\tilde{\mathcal{T}}_p : = \bigg\{\sum_{j=0}^p c_j \eins_{A_j} : c_j \in \{-1, 1\}, \bigcup_{j=0}^p A_j = \mathcal{X},
A_s \cap A_t = \emptyset, s \neq t \bigg\},
\end{align}
where we should emphasize that all trees in \eqref{TpSpace} must follow our specific construction procedure described in Section \ref{sec::RandomPartition} and \ref{sec::AdaptivePartition}.
It can be verified that the nested relation $\tilde{\mathcal{T}}_{q} \subset \tilde{\mathcal{T}}_p$ for $q \leq p$ holds. In the following analysis, for convenience sake,
we need to reformulate the definition of $\tilde{\mathcal{T}}_p$. 
Let $p \in \mathbb{N}$ be fixed.
Let $\mathcal{\pi}$ be a partition of $\mathcal{X}$ with number of splits $p$ 
and $\mathcal{\pi}_p$ denote the family of all partitions $\mathcal{\pi}$.
Moreover, we define
\begin{align} \label{Bp}
\mathcal{B}_p
:= \biggl\{ B = \bigcup_{j \in J} A_j : J \subset \{0, 1, \ldots, p\}, A_j \in \mathcal{\pi}
\in  \mathcal{\pi}_p \biggr\}.
\end{align}
Then, for all $g \in \tilde{\mathcal{T}}_p$, 
there exists some $B \in \mathcal{B}_p$ such that
$g$ can be written as $g = \boldsymbol{1}_B - \eins_{B^c}$.
Therefore, $\tilde{\mathcal{T}}_p$ can be equivalently defined as
\begin{align} \label{Tp2}
\tilde{\mathcal{T}}_p := \bigl\{ \eins_B - \eins_{B^c} : B \in \mathcal{B}_p \bigr\}.
\end{align}
Now, we are able to give the VC dimension of $\mathcal{B}_p$ as follows:
\begin{lemma} \label{VCIndex}
	The VC dimension of $\mathcal{B}_p$ can be upper bounded by $d p + 2$. 
\end{lemma}

After establishing the bound of the VC dimension, we can then give the covering numbers of \eqref{Bp} and \eqref{Tp2}.
Let $\mathcal{B}$ be a class of subsets of $\mathcal{X}$, 
denote $\eins_{\mathcal{B}}$ as the collection of the indicator functions of all  $B \in \mathcal{B}$, that is, $\eins_{\mathcal{B}} := \{ \eins_B : B \in \mathcal{B} \}$.
Moreover, as usual, for any probability measure $Q$,
$L_2(Q)$ is denoted as the $L_2$ space with respect to $Q$ 
equipped with the norm $\|\cdot\|_{L_2(Q)}$.

\begin{lemma} \label{BpTpCoveringNumbers}
	Let $\mathcal{B}_p$ and $\tilde{\mathcal{T}}_p$ be defined as in \eqref{Bp} and \eqref{Tp2} respectively. Then, for all
	$0 < \varepsilon < 1$, there exists a universal constant $K$
	such that
	\begin{align} \label{BpCoveringNumber}
	\mathcal{N} \big( \eins_{\mathcal{B}_p}, \|\cdot\|_{L_2(Q)}, \varepsilon \big) 
	\le  K (d p + 2) (4 e)^{d p + 2} (1 / \varepsilon)^{2 (d p + 1)}
	\end{align}
	and
	\begin{align} \label{TpCoveringNumber}
	\mathcal{N} \bigl( \tilde{\mathcal{T}}_p, \|\cdot\|_{L_2(Q)}, \varepsilon \bigr) 
	\leq K (d p + 2) (4 e)^{d p + 2}(2 / \varepsilon)^{2 (d p + 1)}
	\end{align}
	hold for any probability measure $Q$.
\end{lemma}

For the parent \emph{best-scored} random tree function set $\tilde{\mathcal{T}}_Z$ where the partition follows the same splitting criterion as in Assumption \ref{JoinedSpace} but in the binary value case, we denote
\begin{align} \label{RStar}
\tilde{r}^* : = \inf_{g \in \tilde{\mathcal{T}}_Z} \ p^2 (g) + \mathcal{R}_{L, \mathrm{P}}(g) - \mathcal{R}_{L, \mathrm{P}}^*.
\end{align}
We emphasize at this point that the functions in $\tilde{\mathcal{T}}_Z$ take values in $\{-1, 1\}$. For $\tilde{r} > r^*$,
\begin{align}  \label{Gr}
& \tilde{\mathcal{T}}_r : = \big\{g \in \tilde{\mathcal{T}}_Z : p^2 (g) + \mathcal{R}_{L, \mathrm{P}}(g) - \mathcal{R}_{L, \mathrm{P}}^* \le r \big\},
\\
\label{Hr}
& \tilde{\mathcal{H}}_r = \{ L \circ g - L \circ f_{L, \mathrm{P}}^* : g \in \tilde{\mathcal{G}}_r \},
\end{align}
where $L \circ g$ denotes the least squares loss of $g$. 
Moreover, we define $r^*$, $\mathcal{T}_r$ and $\mathcal{H}_r$ similarly for the function set $\mathcal{T}_Z$, respectively.

\begin{lemma} \label{HrEntropyNumber}
	Let $\tilde{\mathcal{H}}_r$ be defined as in \eqref{Hr}.
	Then, for all $\delta\in(0,1)$,
	the $i$-th entropy number of $\tilde{\mathcal{H}}_r$ satisfies
	\begin{align*}
	\mathbb{E}_{\mathrm{D} \sim \mathrm{P}} \,
	e_i (\tilde{\mathcal{H}}_r, \|\cdot\|_{L_2(\mathrm{D})})
	\leq \biggl( \frac{75 d}{2 e \delta} \sum_{j=1}^m (r / \lambda_j)^{1/2} \biggr)^{1/(2 \delta)}
	i^{- 1/(2 \delta)}.
	\end{align*}
\end{lemma}

According to the above Lemma \ref{HrEntropyNumber}, we are able to deduce the Rademacher average of the binary value case and then the Rademacher average of our regression case.

\begin{lemma} \label{Rademacher}
	Let $\tilde{\mathcal{H}}_r$ be the function set defined as in \eqref{Hr}. For any $\delta \in (0,1)$, there holds that
	\begin{align*}
	\mathbb{E}_{\mathrm{D} \sim \mathrm{P}} \mathrm{Rad}_{\mathrm{D}} & (\tilde{\mathcal{H}}_r, n) 
	\\
	& \leq
	\mathrm{max} \bigg\{
	10 c_{\delta,1} 
	\bigg(\frac{6d}{e\delta n} \sum_{j=1}^m \lambda_j^{-1/2}\bigg)^{\frac{1}{2}} r^{\frac{3-2\delta}{4}},
	5 c_{\delta,2}  \bigg(\frac{6d}{e\delta n}\sum_{j=1}^m \lambda_j^{-1/2}\bigg)^{\frac{1}{1+\delta}} r^{\frac{1}{2+2\delta}} 
	\bigg\},
	\end{align*}
	where $c_{\delta,1}$ and $c_{\delta,2}$ are constants depending only on $\delta$ that will be stated in the proof. And
	\begin{align*}
	\mathbb{E}_{\mathrm{D} \sim \mathrm{P}} \mathrm{Rad}_{\mathrm{D}} & (\mathcal{H}_r, n) 
	\\
	& \leq
	\mathrm{max} \bigg\{
	10 M c_{\delta,1}
	\bigg(\frac{6d}{e\delta n} \sum_{j=1}^m \lambda_j^{-1/2}\bigg)^{\frac{1}{2}} r^{\frac{3-2\delta}{4}},
	5 M c_{\delta,2} \bigg(\frac{6d}{e\delta n}\sum_{j=1}^m \lambda_j^{-1/2}\bigg)^{\frac{1}{1+\delta}} r^{\frac{1}{2+2\delta}}
	\bigg\}.
	\end{align*}
\end{lemma}

\section{Architecture Analysis} \label{sec::ArchitectureAnalysis}
In this section, we first summarize the two-stage best-scored random forest algorithm in Section \ref{sec::AlgorithmConstruction}, then clear illustrations on the reasons why our strategy is a subtle match for the large-scale regression are presented in Section \ref{sec::ASubtleMatch}. In Section \ref{sec::InclusiveFramework}, we emphasize the fact that our two-stage forest algorithm is indeed an inclusive framework where different mainstream regression approaches can be incorporated as value assignment strategies for leaves of trees. To give further descriptions of our inclusive framework, analysis on certain real data sets are provided as illustrative examples in Section \ref{sec::IllustrativeExamples}.

\subsection{Algorithm Construction} \label{sec::AlgorithmConstruction}
The construction of the two-stage best-scored random forest (TBRF) predictors demonstrated in Section \ref{sec::MainAlgorithm} can be summarized by pseudocode in Algorithm \ref{alg:1}.

It is worth noting that when in experiments, the partition of the feature space at stage one and the continuous partitions on cells at stage two all follow the adaptive random partition described in Section \ref{sec::AdaptivePartition}. In this way, by increasing the effective number of splits, the empirical performances can be further improved. Moreover, for the convenience of computation, we let the number of candidates in each cell to be equal, that is $k^t_1 = k^t_2 = \ldots = k^t_m = k$, $t=1,\ldots, T$.

When it comes to the empirical performances in terms of prediction accuracy, an appropriate measurement is in demand.
In this paper, we adopt the ubiquitous Mean Squared Error (\textit{MSE}) through all experiments:
\begin{align*}
\textit{MSE}\,(\hat{f})=\frac{1}{n_{\text{test}}}\sum_{i=1}^{n_{\text{test}}} \big(Y_i - \hat{f}(X_i) \big)^2,
\end{align*}
where $\hat{f}$ represents the predictor and $\{ (X_i, Y_i) \}_{i=1}^{n_{\text{test}}}$ are the test samples.

\subsection{A Subtle Match for Large-scale Regression} \label{sec::ASubtleMatch}
This subsection sheds light on the truth that compared to some published vertical methods, the two-stage best-scored random forest is authentically a subtle match for the large-scale regression problems. The essence of our algorithm lies in the fact that by
separating the splitting process of the feature space into two stages, the TBRF significantly speeds up the calculation via parallel computing without changing the algorithm structure of the trees in the random forest method. Moreover, on account that the randomness resided in the splitting criterion makes the ensemble learning available, the boundary discontinuities that long plague other vertical strategies can then be naturally solved. We mention here that partitions in all stages of TBRF in experiments follow the adaptive random splitting criterion to increase the effective number of splits.

\begin{algorithm}
	
	\caption{Two-stage best-scored random forest Algorithm}
	
	\label{alg:1}
	
	\KwIn{
		training data $D:=\{(X_1, Y_1),\ldots,(X_n,Y_n)\}$;
		\\
		\quad\quad\quad\quad number of trees in the forest $T$;
		\\
		\quad\quad\quad\quad number of cells in the feature space partition $m$;
		\\
		\quad\quad\quad\quad number of candidates for each child tree on $\{V_j^t\}_{j=1}^m$ is $\{k_j^t\}_{j=1}^m$, $t=1,\ldots, T$.
	}   
	\For{$t =1 \to T$}{
		Conduct the adaptive random partition on $\mathcal{X}$ with $m-1$ splits to establish an $m$-cell feature space partition;
		\\
		\For{$j=1 \to m$}{
			Build $k_j^t$ candidate trees with purely random partition on cell $V_j^t$;
			\\
			Using cross-validation to choose one $g_{Z_{jt}}$ with the smallest validation error out of $k_j^t$ trees (Best-scored method).
		}
		A parent best-scored random tree is developed by
		\begin{align*}
		g_{Z_t} = \sum_{j=1}^m g_{Z_{jt}}
		\end{align*}
	}
	\KwOut{The two-stage best-scored random forest which is
		\begin{align*}
		f_{Z} = \frac{1}{T}\sum_{t=1}^T g_{Z_{t}}.
		\end{align*}}
\end{algorithm} 

\begin{itemize}[leftmargin=*]
	\item \textbf{PK}: Patchwork kriging (PK) proposed by \cite{park2018patchwork} is an approach for Gaussian process (GP) regression for large datasets with the well-known discontinuity problem mitigated via adding pseudo-observations to the boundaries. Previous Gaussian process vertical methods put forward in \cite{park11a} and \cite{park16a} have tried to join up the boundaries of the adjacent local GP models, and PK is the newest upgraded version of this kind. Though spatial tree is employed in PK to generate data partitioning of uniform sizes when data is unevenly distributed, the boundaries are still artificially selected just as the mesh partitioning proposed in previous papers. Moreover, although efforts have been paid to conquer the boundary discontinuities, this method in turn has to face other challenges. For example, when encountering data with high dimension and large volume, in order to achieve better prediction accuracy, more pseudo-observations need to be added to the boundaries, which leads to a significant growth in computational complexity. In addition, the discontinuities sometimes still exist, see Figure \ref{fig:syn_PK}. Last but not least, PK has not been proceeded by parallel computing and therefore, takes time to derive accurate results.

	
	\item \textbf{VP-SVM}: Support vector machines for regression being a global algorithm is impeded by super-linear computational requirements in terms of the number of training samples in large-scale applications. To address this, \cite{meister16a} employs a spatially oriented method to generate the chunks in feature space, and fit LS-SVMs for each local region using training data belonging to the region. This is called the Voronoi partition support vector machine (VP-SVM). However, the boundaries are artificially selected and the boundary discontinuities do exist.

	\item \textbf{TBRF}: When facing large-scale regression problems, our two-stage best-scored random forest (TBRF) subtly divides the partition process of each tree in random forest into two steps so that parallel computing can come in handy. As is mentioned in Section \ref{sec::MainAlgorithm}, in stage one, the feature space is partitioned into non-overlapping cells following an adaptive random splitting criterion which is completely data-driven, and then random partitions are continuously conducted on each cell separately in stage two. 
	Careful observation of the overall partition will find that the two-stage splitting result is equivalent to that of carrying out continuous uninterrupted splits on the feature space directly. Therefore, the computation speed can be significantly accelerated by assigning each cell to different cores without changing the algorithm structure of trees in the random forest. Recall that the boundary discontinuities inevitable for the existing vertical methods come from the fact that they all attempt to execute a globally smooth algorithm independently to cells of the artificial partition of the feature space. In this manner, compared to predictors obtained by applying the algorithm integrally to the input space, predictors of the existing vertical methods can only maintain smoothness within cells while the discontinuities appear on those artificially chosen boundaries. Different from those vertical methods with fixed partitions, the randomness resided in our partition paves the way for ensemble learning. 
	Even though each tree in the random forest is inherently a piecewise smooth algorithm, by taking full advantages of the ensemble learning, the integrating forest achieves an asymptotic smoothness, which naturally blurs or even almost eliminates the boundary discontinuities of each tree. 
	Therefore, there is no need to underscore the concept of boundaries anymore. 
	As a result, instead of adding continuity constraints to the boundaries, we settle the well-known discontinuities via ensemble learning within the random forest architecture.
	\end{itemize}

To provide a more intuitive explanation on how our forest model smoothes the boundaries, we present the following simulation works. We generate a dataset of $50,000$ noisy observations from $y_i = f(x_i) + \epsilon_i$ for $i=1,\ldots, 50000$, where $f(x) = \sin x$, and $x_i \sim \mathcal{U}(0, 10)$ and $\epsilon_i \sim \mathcal{N}(0, 0.2)$ are independently sampled. Comparisons are conducted between PK, VP-SVM and TBRF with results shown in Figures \ref{fig:syn_PK}, \ref{fig:syn_VP-SVM} and \ref{fig:syn_TBRF}, respectively.

\begin{figure*}[htbp]
	\begin{minipage}[t]{0.6\textwidth}  
		\centering  
		\includegraphics[width=\textwidth]{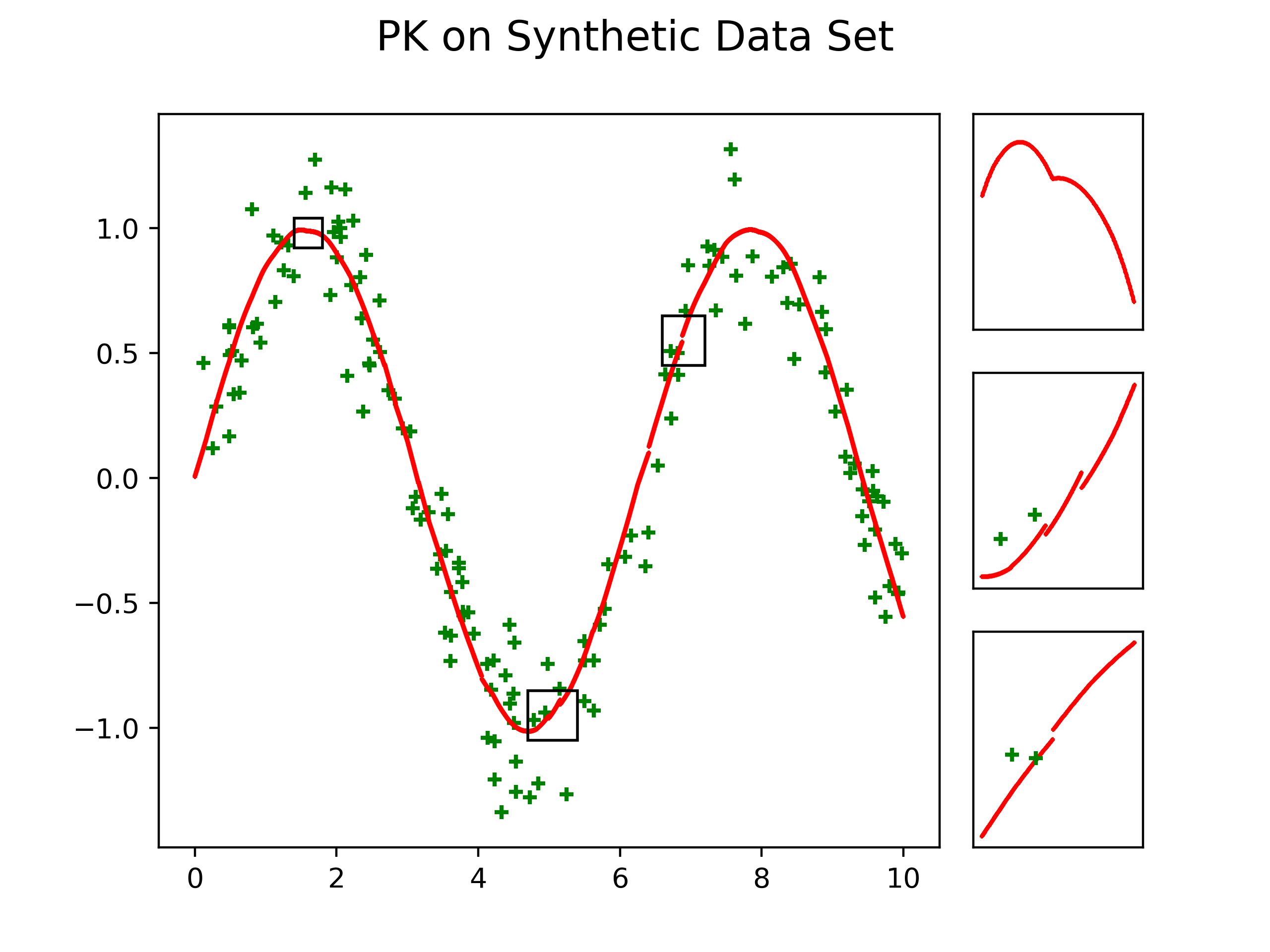}  
	\end{minipage}  
	\centering  
	\caption{Example illustrating the PK predictor on the synthetic data set. Only $150$ data out of the total $50,000$ are plotted in green for clarity. The red curve represents the predictor of the PK with three of the significant boundaries joining areas specifically framed and amplified on the right.}
	\label{fig:syn_PK}
\end{figure*}

\begin{figure*}[htbp]
	\begin{minipage}[t]{0.6\textwidth}  
		\centering  
		\includegraphics[width=\textwidth]{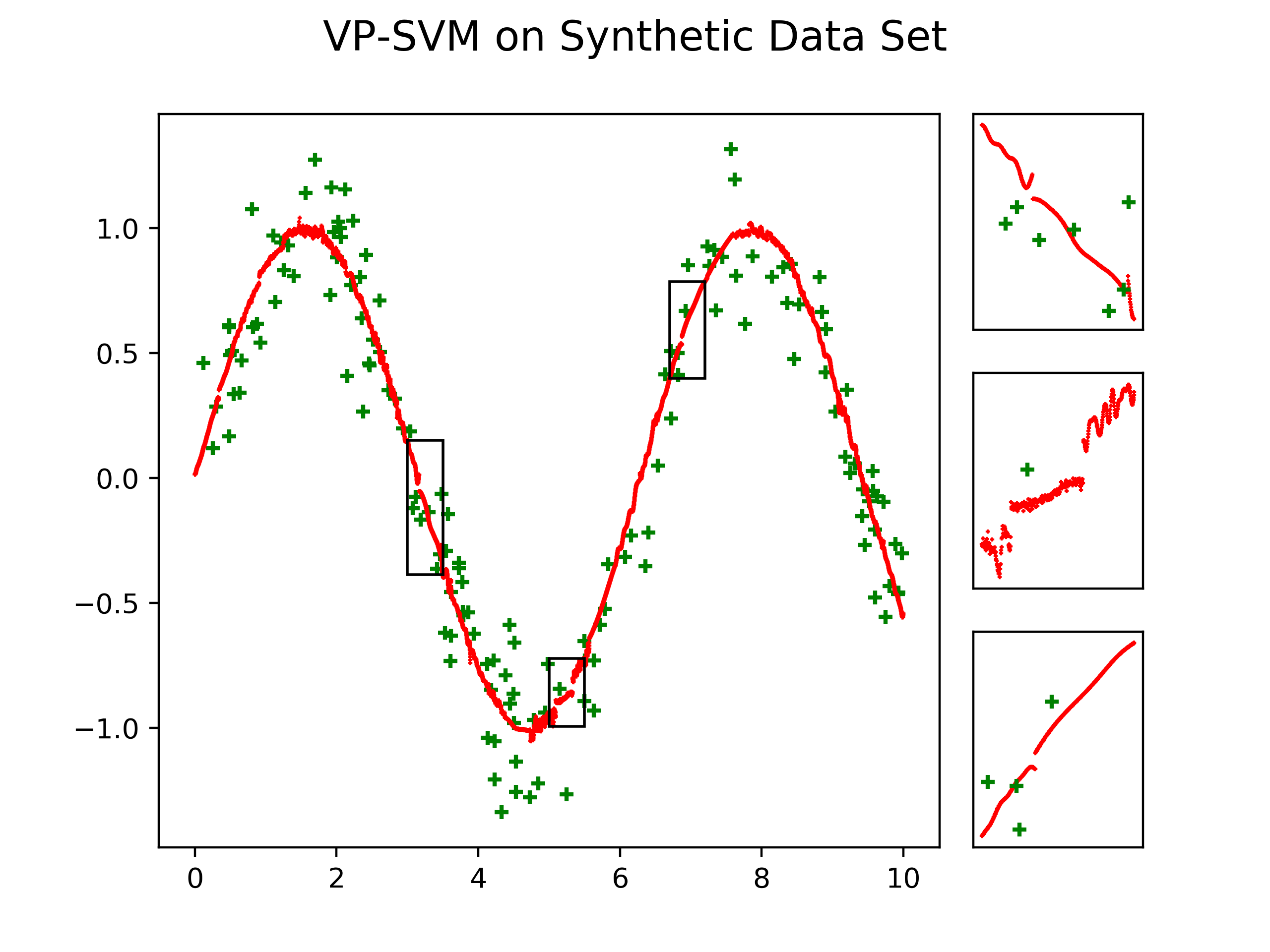}  
	\end{minipage}  
	\centering  
	\caption{Example illustrating the VP-SVM predictor on the synthetic data set. Only $150$ data out of the total $50,000$ are plotted in green for clarity sake. The red discontinuous curve represents the predictor of the VP-SVM with three of the significant discontinuous areas specially framed and amplified on the right.}
	\label{fig:syn_VP-SVM}
\end{figure*}

\begin{figure*}[htbp]
	\begin{minipage}[t]{1.0\textwidth}  
		\centering  
		\includegraphics[width=\textwidth]{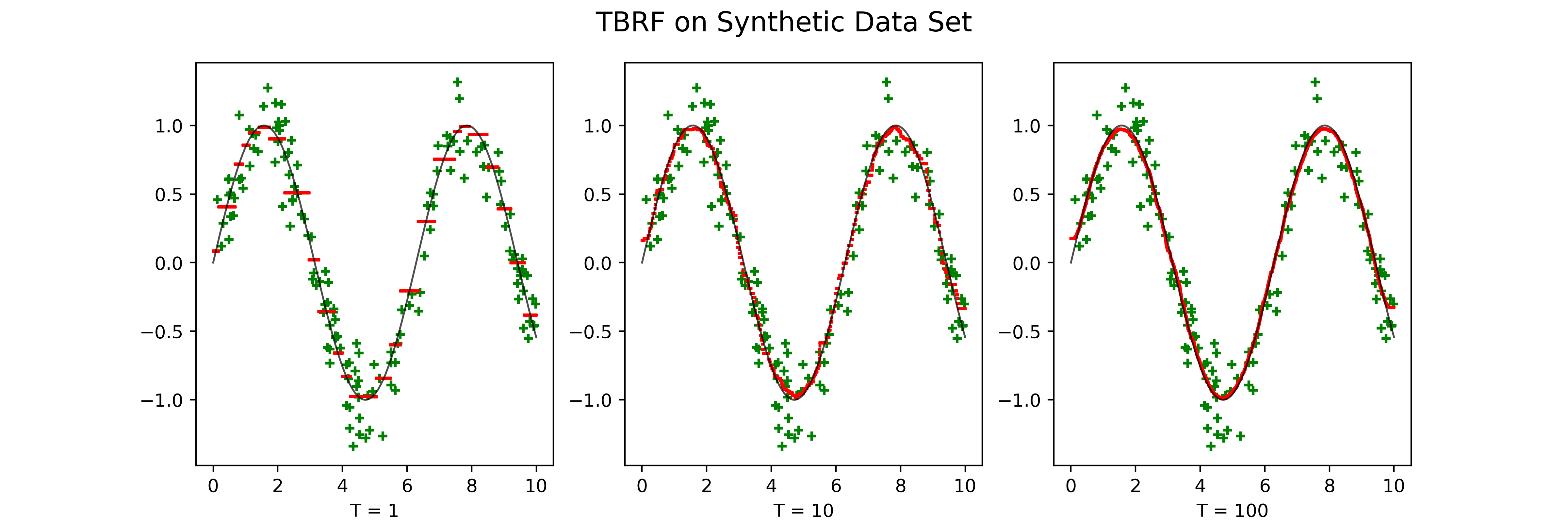}  
	\end{minipage}  
	\centering  
	\caption{Example illustrating the TBRF predictor on the synthetic data set. Only $150$ data out of the total $50,000$ are plotted in green for clarity sake. The TBRF predictors are plotted in red with the Bayes decision function $f(x) = \sin x$ in black. $T$ is the number of trees in the forest.}
	\label{fig:syn_TBRF}
\end{figure*}  

It can be seen from Figure \ref{fig:syn_PK} that even though PK tries to join up the boundaries, the resulting predictor at the connected point is kind of artificial (upper-right image), or is sometimes still discontinuous (middle-right and lower-right images).
The boundary discontinuities of VP-SVM can be easily observed from Figure \ref{fig:syn_VP-SVM}, while our TBRF achieves asymptotic smoothness with the number of trees in the forest $T$ increasing in Figure \ref{fig:syn_TBRF}. Note that the above satisfactory performances of TBRF also have strong theoretical supports: 

\textit{(i)} Although our algorithm is obtained by conducting local computations (e.g.~stage two), we establish the global convergence and particularly, the good global learning rates in Theorem \ref{RateVPForest}. One fact is that if we conduct purely random splits on the feature space recursively, the resulting tree will converge. Now, our algorithm only divides the original recursive splits into two stages and it is intrinsically the same with the original splitting procedures from the perspective of the results. There being no change in the random tree structure is the reason why TBRF maintains a global convergence even under local computations.

\textit{(ii)} The TBRF achieves asymptotic smoothness. Since we employ the purely random splitting criterion, for any point in the feature space, the probability of its being on the boundaries of one tree is zero in theory. Moreover, the probability remains zero of its being on the boundaries of several trees simultaneously. Therefore, for any point in the feature space, the estimation function is smooth with high probability at that point. Even if, unfortunately, this point appears on the boundaries of certain trees, it is with low probability for it to be also on the boundaries of others. Consequently, with the number of trees in the forest increasing, the discontinuities occur at this point will be asymptotically smoothed by other trees.

\subsection{An Inclusive Framework} \label{sec::InclusiveFramework}
The TBRF investigated so far hinges on random trees whose leaf value assignments are according to the random tree decision rule shown in \eqref{TreeDecisionRule}. More precisely, trees in the forest are piecewise constant predictors. However, experiments on high-dimensional data reveal that the piecewise constant random forest might not be adequate enough to provide accurate prediction. This imperfection actually serves as an opportunity for us to figure out how inclusive and versatile the framework of the TBRF can be. Specifically, by incorporating some standard regression algorithms as alternative value assignment methods into our two-stage random forest framework, the prediction accuracy is essentially improved. Accordingly, the TBRF model is then branched into three different modeling paths with respect to what kinds of leaf value assignment approaches are chosen:
\begin{itemize}[leftmargin=*]
	\item \textbf{TBRF-C}: If the value of each leaf node of the trees is assigned the average of the output values of samples falling into that leaf, then the trees in the forest are piecewise constant predictors. We name this branch of our algorithm the \textit{piecewise constant TBRF}, abbreviated as TBRF-C. This simple model captures the constant models and shows great efficiency when dealing with low-dimensional data.

	\item \textbf{TBRF-L}: If the value of each leaf node of the trees is calculated via the linear regression based on the output values of samples falling into that leaf, then the trees in the forest are piecewise linear predictors. This model is named the \textit{piecewise linear TBRF}. For the convenience of computation, we adopt the LS-SVMs for regression with linear kernel 
	$K(X, X') = X^T X'$
	where $X$ and $X'$ are vectors in the input space. The assumption of this underlying linear model may improve the prediction accuracy in the low-dimensional examples.
%
	
	\item \textbf{TBRF-G}: The prediction accuracy of the above two branches of the TBRF decrease when encountering high-dimensional data sets. To address this, we propose to use the LS-SVMs with Gaussian RBF kernel 
	$K(X, X') = \exp(-\gamma \| X - X' \|^2)$, $\gamma > 0$
	for regression as the value assignment strategy. This \textit{piecewise Gaussian TBRF} (TBRF-G) shows great superiority when facing data with high dimension.
	\end{itemize}
By incorporating different mainstream regression strategies into the framework of our two-stage best-scored random forest, we are not only able to accelerate the algorithm via parallel computing, but also obtain prediction with high accuracy. Consequently, our TBRF is an inclusive and versatile framework that is exactly befitting for large-scale regression problems.

\subsection{Illustrative Examples} \label{sec::IllustrativeExamples}
In this section, we design illustrative real data analysis to figure out what kind of assignment strategy should be combined into the two-stage best-scored random forest framework when handling different types of data sets. Specifically, the analysis is based on the following three real data sets. 

The first real data set is the \textbf{\tt{TCO}}, which contains data collected by NIMBUS-7/TOMS satellite to measure the total column of ozone over the globe on Oct 1, 1988. The data consist of $48,331$ two-dimensional samples representing the locations recorded by longitudes and latitudes. The observations are uniformly spread over the range of the longitude and latitude, and the learning goal is to predict the total column of ozone at any unobserved location. We randomly split each dataset into a training set containing $70\%$ of the total observations and a test set containing the remaining $30\%$ of the observations. 

The second real data set is the \emph{Physicochemical Properties of Protein Tertiary Structure Data Set} ($\tt{PTS}$) on UCI data sets, which contains $45,730$ samples of dimension $9$. Similar as other high dimensional data, the measurements are embedded on a low dimensional subspace of the entire domain. The ratio of number of samples in the training set and the testing set is $7 : 3$.

The third real data set is the {\tt SARCOS}, containing $44,484$ training and $4,449$ testing observations from a seven degrees-of-freedom {\tt SARCOS} anthropomorphic robot arm.
The $21$-dimensional input data represent attributes such as positions, moving velocities, accelerations of the joints of the robot arm, etc. As for outputs, we use the first one out of the seven response variables for numerical study. The main learning task is to predict one of the joint torques in a robot arm when the input observation is available.

We now conduct empirical comparisons on \textit{MSE} and training time among TBRF-C, TBRF-L and TBRF-G on these three data sets where $n$ and $d$ denote sample size and dimension, respectively. Since all three branches of the TBRF are based on Algorithm \ref{alg:1}, we now discuss how to conduct the leaf values assignment:

\begin{itemize}[leftmargin=*]
	\item \textbf{TBRF-C}: We take the average output of the samples falling into certain leaf as the corresponding leaf value.
	
	\item \textbf{TBRF-L}: We utilize LS-SVMs with linear kernel based on the samples falling into certain leaf to derive a regression function on that leaf. To pick up the appropriate hyperparameter $C$ which trades off the accuracy deduced by the training data and the simplicity of the decision function, we randomly divide the samples falling into certain leaf into two parts where 70\% of the samples are used for training and the rest for validation. By employing the appropriate hyperparameter, we derive the regression function on that leaf based on the overall samples.
	
	\item \textbf{TBRF-G}: We utilize LS-SVMs with Gaussian kernel based on the samples falling into certain leaf to derive a regression function on that leaf. In order to select the appropriate hyperparameters pair $(C, \gamma)$, we randomly divide the samples falling into the leaf into $7:3$ where grid search is conducted on $70\%$ of the samples and the rest $30\%$ are utilized to make the choice of the hyperparameters pair. After that, regression with chosen hyperparameters is conducted based on all samples falling in that leaf.
	
	
	\end{itemize}

From now on, to train models for sufficiently large data sets, we use a professional compute server equipped with eight Intel(R) Xeon(R) CPU E7-8860 v3 (2.20GHz) 16-core processors,  64 GB RAM.
For the sake of fairness, other methods for comparison in the following work are also trained on this server.

As is illustrated in the Algorithm \ref{alg:1}, after partitioning the feature space into $m$ non-overlapping cells, we need to keep splitting on each cell to build child tree. However, considering that the numbers of samples falling into different cells are different, it is improper to apply the same number of splits to all cells. To address this, we let the number of splits on the cell proportional to the number of samples falling into that cell, and the proportional coefficient $\tt{pro}$ is tunable. For all three data sets {\tt TCO}, {\tt PTS} and {\tt SARCOS}, we adopt the same parameters configuration for grid search:
\begin{itemize}[leftmargin=*]
	\item TBRF-C: $T \in \{20, 50\}$, $m \in \{20, 50, 200 \}$, $k \in \{10, 100\}$ and ${\tt pro} \in \{0.2, 0.5, 0.8 \}$; 
	
	\item TBRF-L: $T \in \{20, 50\}$, $m \in \{30, 50, 70\}$, $k \in \{1, 5, 20\}$ and ${\tt pro} \in \{0.2, 0.5, 0.7 \}$;
	
	\item TBRF-G: $T \in \{20, 50 \}$, $m \in \{ 5, 20, 40 \}$, $k \in \{1, 5\}$ and ${\tt pro} \in \{0.005, 0.05, 0.2 \}$. 
\end{itemize}

\begin{figure}[htbp]
	\begin{minipage}[t]{0.49\textwidth}  
		\centering  
		\includegraphics[width=\textwidth]{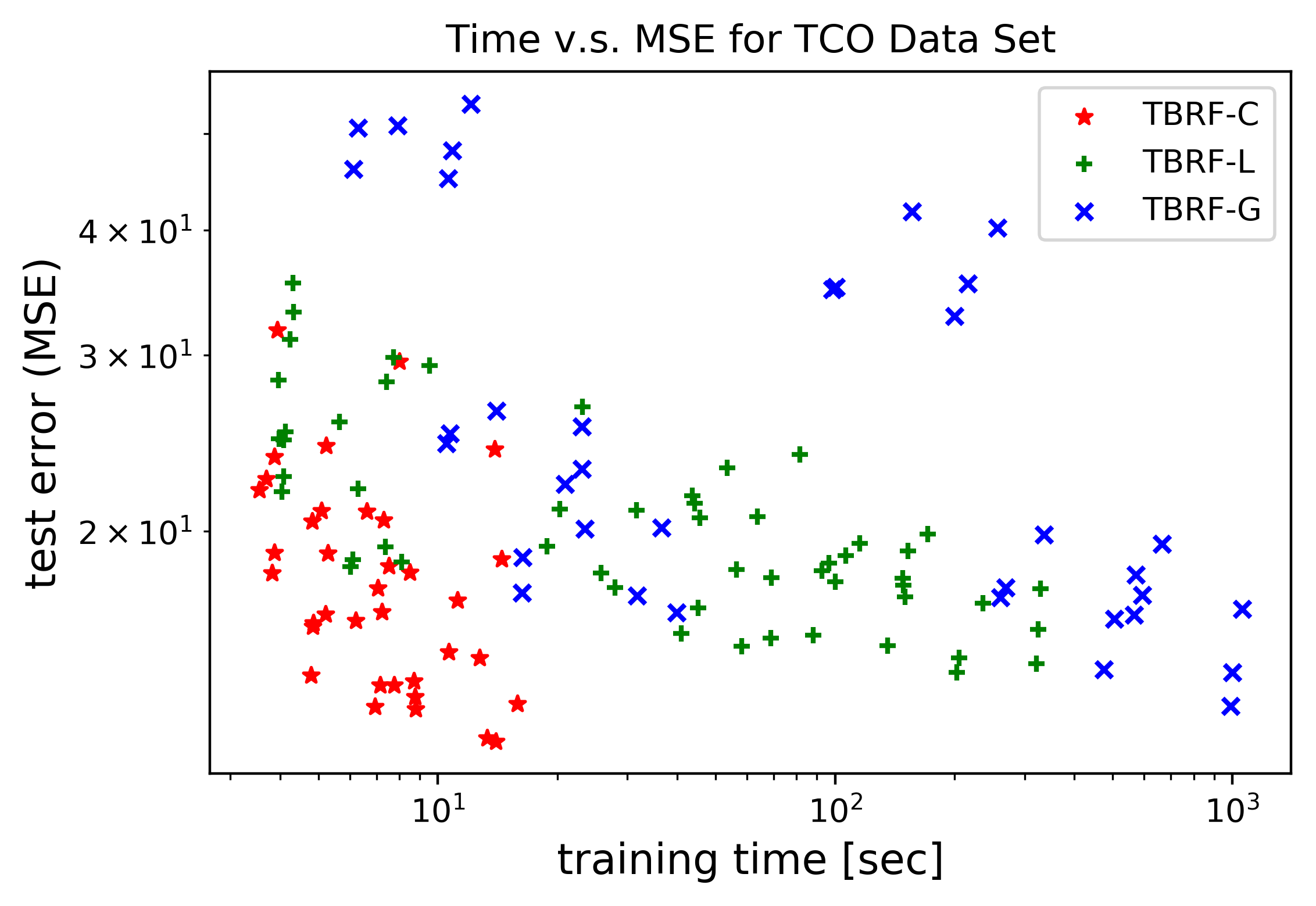}  
	\end{minipage}  
	\begin{minipage}[t]{0.47\textwidth}  
		\centering  
		\includegraphics[width=\textwidth]{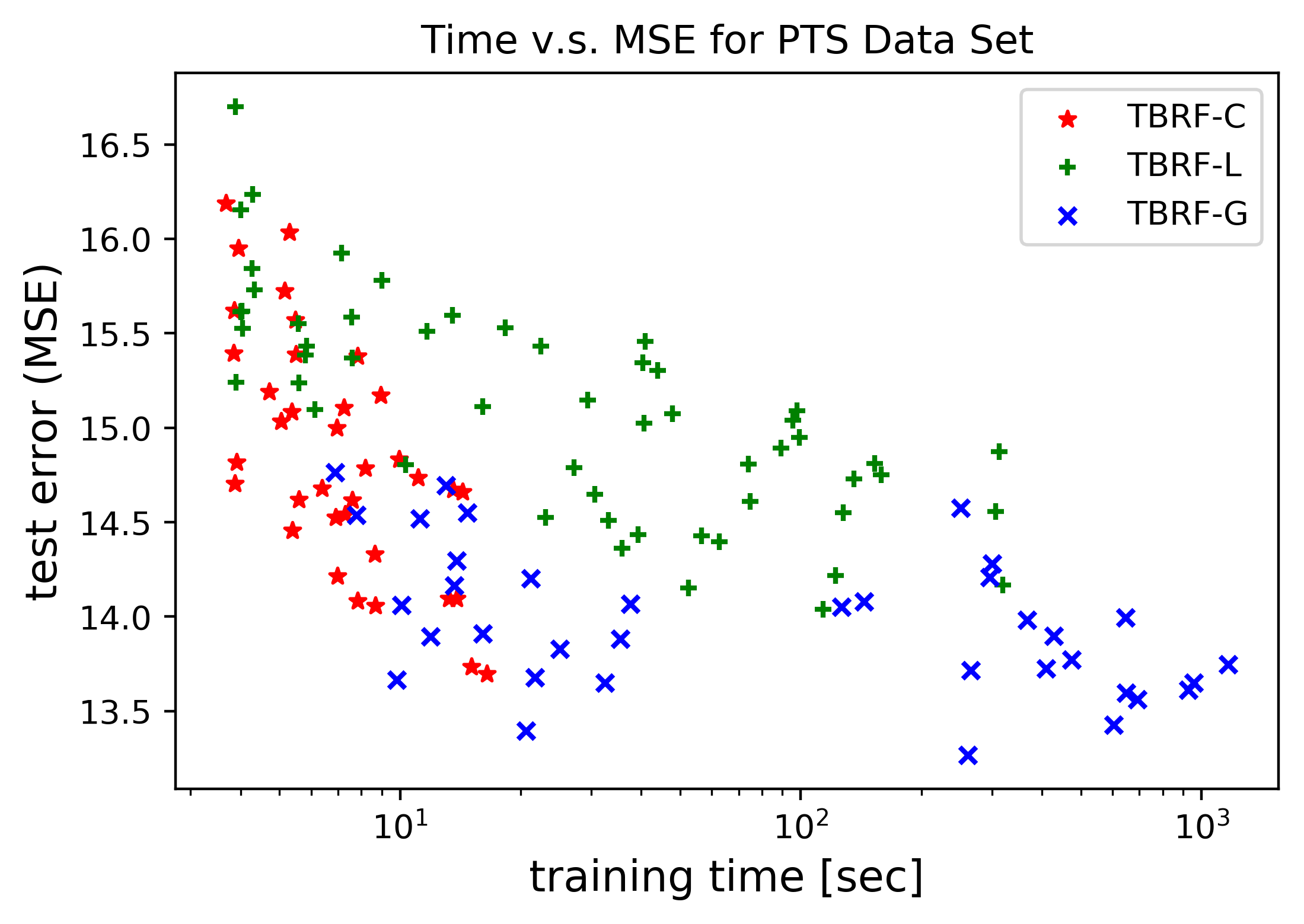}  
	\end{minipage}  
	\begin{minipage}[t]{0.47\textwidth}  
		\centering
		\includegraphics[width=\textwidth]{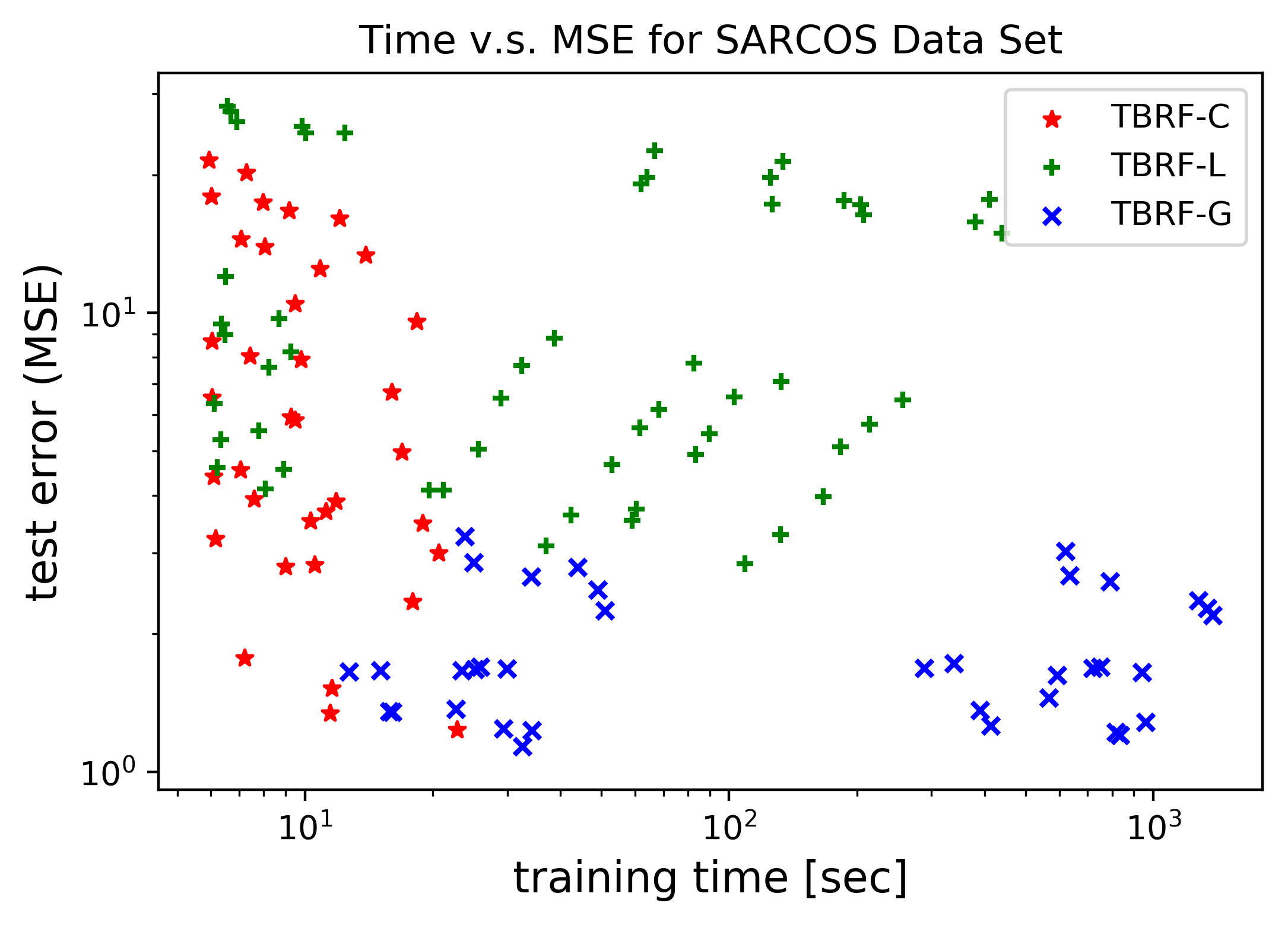}  
	\end{minipage}  
	\centering  
	\caption{Prediction accuracy versus total training time for the data sets {\tt TCO}, {\tt PTS} and {\tt SARCOS}. Each panel compares all the three variants TBRF-C, TBRF-L and TBRF-G. Different points in the panels refer to the experimental results under the corresponding model with different hyperparameter tetrads $(T, m, k, {\tt pro})$.}
	\label{fig::illustrative}
\end{figure} 

Figure \ref{fig::illustrative} shows the main results. It can be observed that for the low-dimensional data, TBRF-C shows both fast and accurate performance over TBRF-L and TBRF-G where TBRF-G takes longer time to derive satisfying results. For {\tt PTS} consisting of $9$-dimensional data, TBRF-G contends against TBRF-C.
For {\tt SARCOS}, TBRF-G gradually shows its advantages over the other two methods in terms of test error and training time on higher-dimensional data set. Through these experiments, we find out that TBRF-C is enough to handle low-dimensional data while TBRF-G is qualified to settle the high-dimensional ones. Therefore, we treat TBRF-C and TBRF-G as two main variants of the TBRF.

\section{Numerical Evaluation} \label{sec::NumericalEvaluation}
This section is concerned with the implementation issues and empirical assessments of the two-stage best-scored random forest. In order to further enhance the prediction accuracy of the forest predictor, we propose some improvement techniques in Section \ref{sec::ExperimentalImprovemets}. Parameter analysis for two main variants TBRF-C and TBRF-G are conducted separately in Section \ref{sec::ParameterAnalysis}. Last but not least, comparison experiments on more real-world data are designed to test the sharpness of our theoretical predictions.

\subsection{Experimental Improvements} \label{sec::ExperimentalImprovemets}
This subsection introduces two experimental improvements available for ameliorating the prediction accuracy of TBRF predictors. According to the type of data at hand, we can selectively use these methods for better experimental results.

\subsubsection{Adaptive Oblique Random Partition}
Till now, the partition processes considered have only performed in an axis-parallel manner. Nevertheless, there comes one caveat that if the underlying concept is defined by a polygonal space partitioning, then the axis-parallel partition may not be that accurate for it can only approximate the correct model with staircase-like structure. On the contrary, oblique partitioning \citep{breiman1984classification, utgo1991linear, brodley1992multivariate, murthy1994system} approaches proposed in the literature serves as a proper alternative.

Similar as the axis-parallel random splitting criterion demonstrated in Section \ref{sec::RandomPartition}, we can also formalize one possible building process following the oblique random splitting criterion. To be concrete, a random vector $O_i := (L_i, G_i, W_i)$ is introduced to reveal the splitting mechanism at the $i$-th step. Let $L_i$ denote the to-be-split cell at the $i$-th step that is chosen from all cells presented at the $(i-1)$-th step uniformly at random, $G_i \in \mathrm{R}^d$ represent the barycenter of cell $L_i$ and $W_i$  be the normal vector of the split hyperplane at that step. To notify, since all partitions conducted in the experiments follows the adaptive criterion, we now give the operation of the adaptive random oblique partition. First of all, $t$ samples are randomly selected from the training data set each of which is certain to fall into one of the cells formed in the $(i-1)$-th step. Hence, the cell where most of these samples fall in is assigned to be the to-be-split cell $L_i$. Then, since the coordinates of samples falling into cell $L_i$ among the $t$ samples are recorded, we can substitute the barycenter $G_i$ for their centroid $X_i^c$. Next, the actual split performed on $L_i$ is a part of the chosen hyperplane $W_i^TX + b_i = 0$, $X \in L_i$. To live out the random splitting rule, we let the normal vectors of the hyperplanes $\{W_i,\ i\in \mathbb{N}_+\}$ be i.i.d.~distributed from $\mathcal{U}[-1, 1]^d$ and $b_i:=-W_i^TX_i^c$. Now that we have finished the operation of the $i$-th step, the random tree with oblique partitions can be constructed by following this procedure recursively. An example of the construction process of the oblique partition is shown in Figure \ref{fig:ob}.
 
 \begin{figure*}[htbp]
 	\begin{minipage}[t]{0.8\textwidth}  
 		\centering  
 		\includegraphics[width=\textwidth]{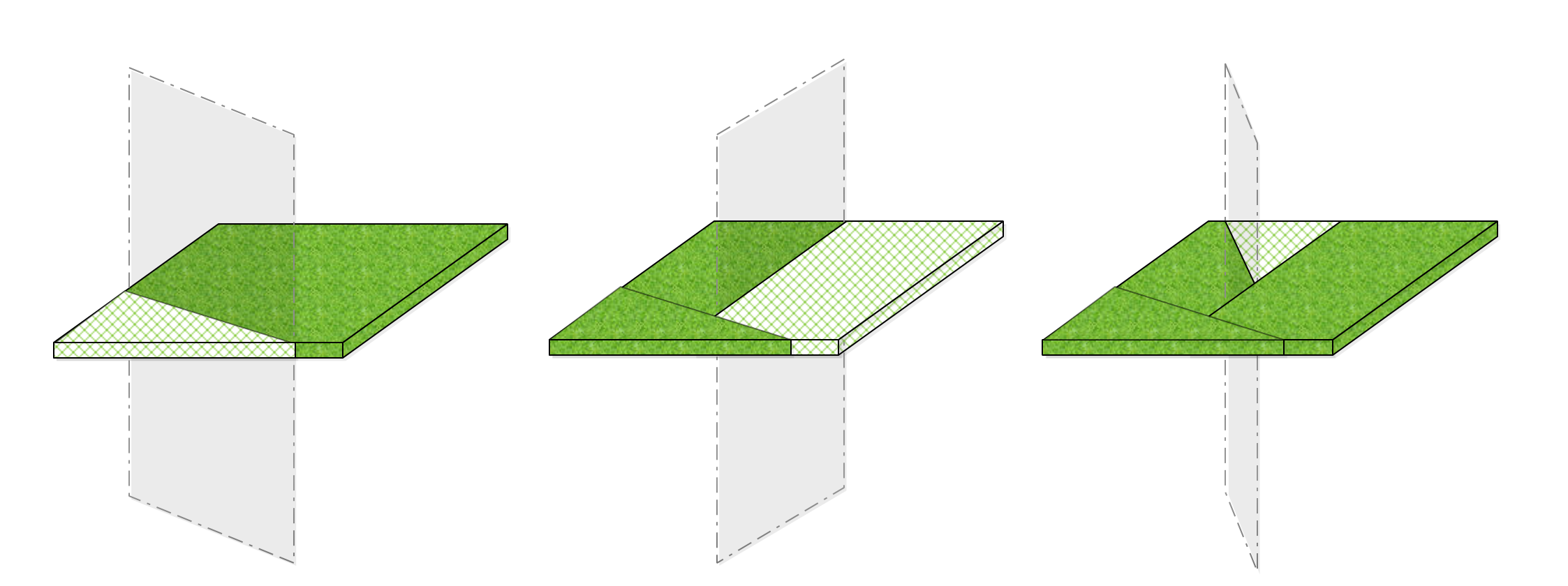}  
 	\end{minipage}  
 	\centering  
 	\caption{Possible construction procedures of $3$-split oblique random partitions in a $2$-dimensional space.}
 	\label{fig:ob}
 \end{figure*}

\subsubsection{Vacancy Filling} \label{sec::VacancyFilling}
Since we employ a random partition, after completing the two-stage partition, we need to deal with the situation where there are child cells\footnote{In the previous sections, we name the results $V_1, \ldots, V_m$ of the feature partition at stage one as cells. Since the splits will continue to be applied to partition these cells into even small cells at stage two, the resulting small cells are called the child cells here for clarity.} with no samples fallen in. Therefore, we come up with two different solutions to label the empty child cells, namely, to fill the vacancy: 
 \begin{itemize}[leftmargin=*]
 	\item \textbf{Mean-based solution}: For an empty child cell, the mean value of the samples contained in the cell where this empty child cell is located is assigned to it. This is a generally applicable vacancy filling methods, available for both axis-parallel and oblique partitions. Moreover, this method has already been used in the illustrative examples.
 	
 	\item \textbf{$1$-NN-based solution}: For an empty child cell, we assign the value of the closest non-empty child cell to it where we take the Euclidean distance between the geometric centers of child cells as child cell distance. Note that this vacancy filling can only be applied to partitions with regular shapes, e.g.~axis-parallel partition in our case. 
 \end{itemize}
Experiments in the following work illustrate that our vertical method combined with mean-based solution runs faster than that with $1$-NN-based solution, however, the $1$-NN-based one provides more accurate results.
Therefore, considering that there exists a trade-off between training time and accuracy, we use method with the appropriate solution for different data sets.
To be specific, if two solutions lead to almost the same accuracy, we choose the one with shorter training time; if the accuracy caused by the two methods is very different, we pick up the high-accuracy one.

\subsection{Parameter Analysis} \label{sec::ParameterAnalysis}
The satisfying results of the illustrative examples in Section \ref{sec::IllustrativeExamples} strike us that the multiple tunable hyperparameters in our model may be the reason why better prediction accuracy is available. 
Aiming at verifying this idea, we design the following experiments concentrating on the \textit{MSE} and training time performance of the two main variants of the TBRF, which are TBRF-C and TBRF-G, separately for different hyperparameters values. 
On account that the main focus is on the analysis of the two-stage forest structure of TBRF, we omit the hyperparameters in the value assignment strategies here. 
More specifically, the target hyperparameters include the number of trees in the forest $T$, the number of cells in feature space partition at the stage one $m$, the proportional coefficient related to the number of splits at the stage two $\tt{pro}$, and the number of candidates for each child tree $k$. Considering that TBRF-C has better performance when handling low-dimensional data, the analyses are carried out on the {\tt TCO} data set and we set the ratio of the training set to the testing set being $7 :3$. As for TBRF-G, since it shows its advantages in settling high-dimensional data, we introduce the fourth data set utilized to conduct the analysis.

\begin{figure}[htbp]
	\begin{minipage}[t]{0.43\textwidth}  
		\centering  
		\includegraphics[width=\textwidth]{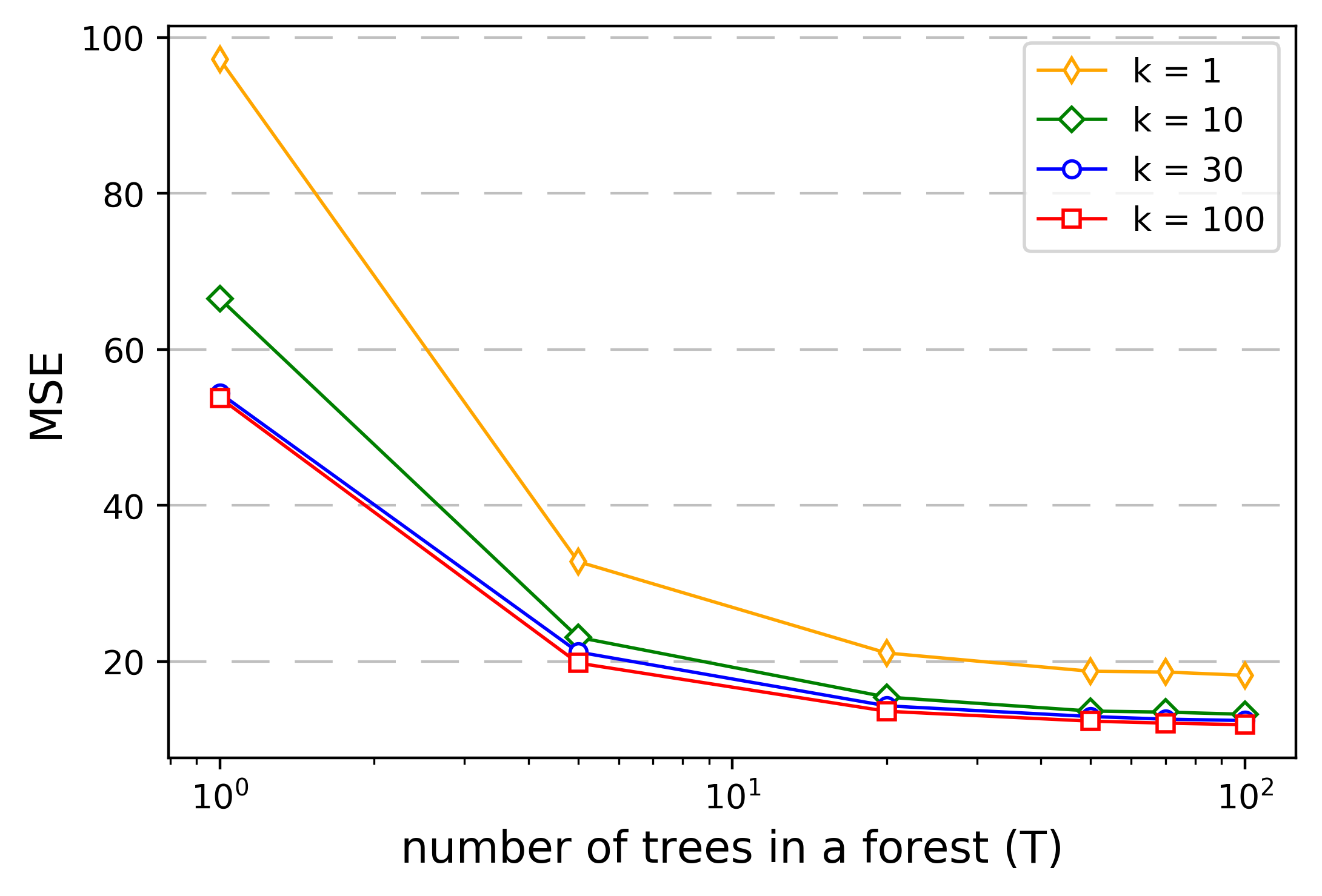}  
	\end{minipage}  
	\begin{minipage}[t]{0.43\textwidth}  
		\centering  
		\includegraphics[width=\textwidth]{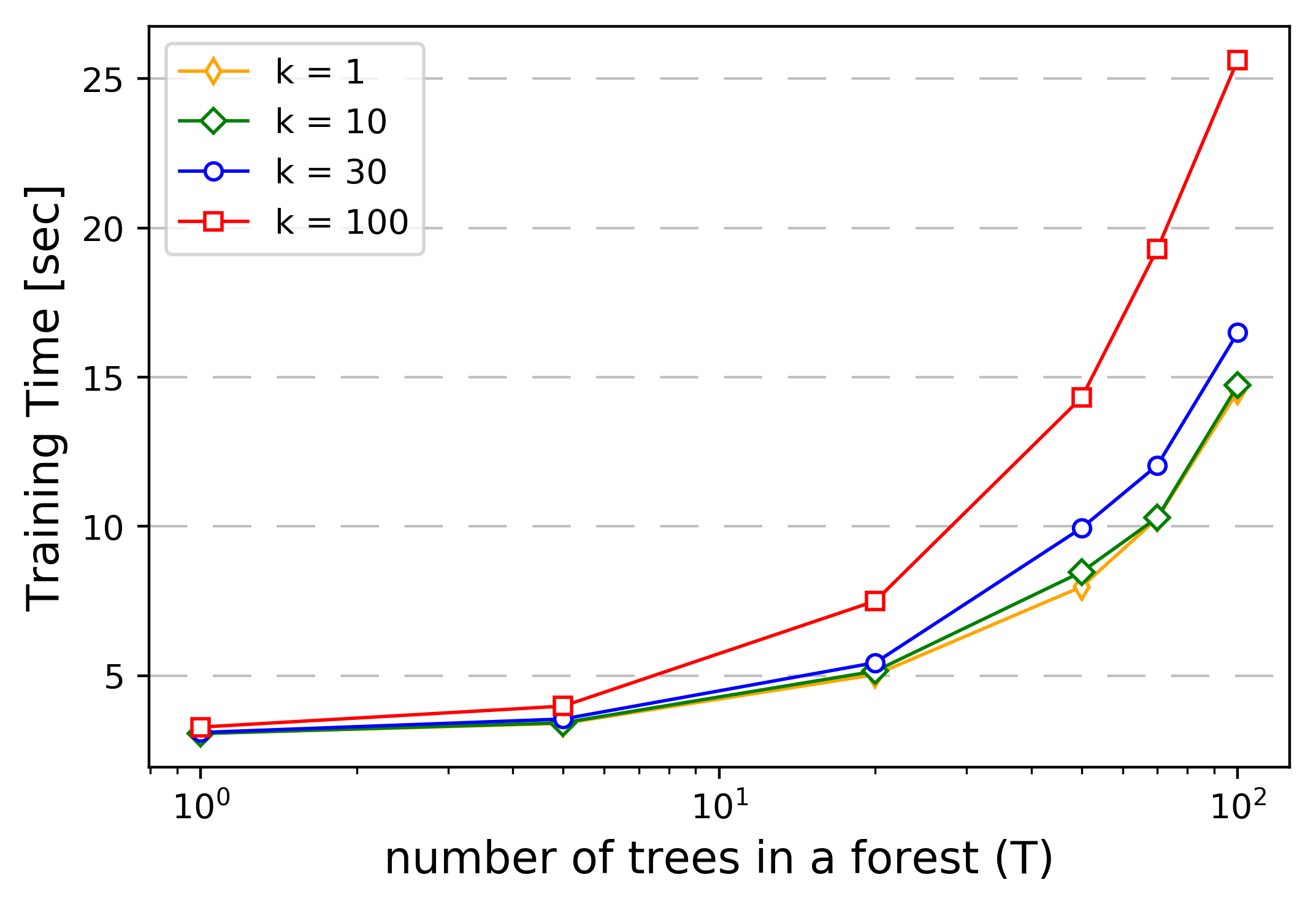}   
	\end{minipage} 
	\begin{minipage}[t]{0.43\textwidth}  
		\centering  
		\includegraphics[width=\textwidth]{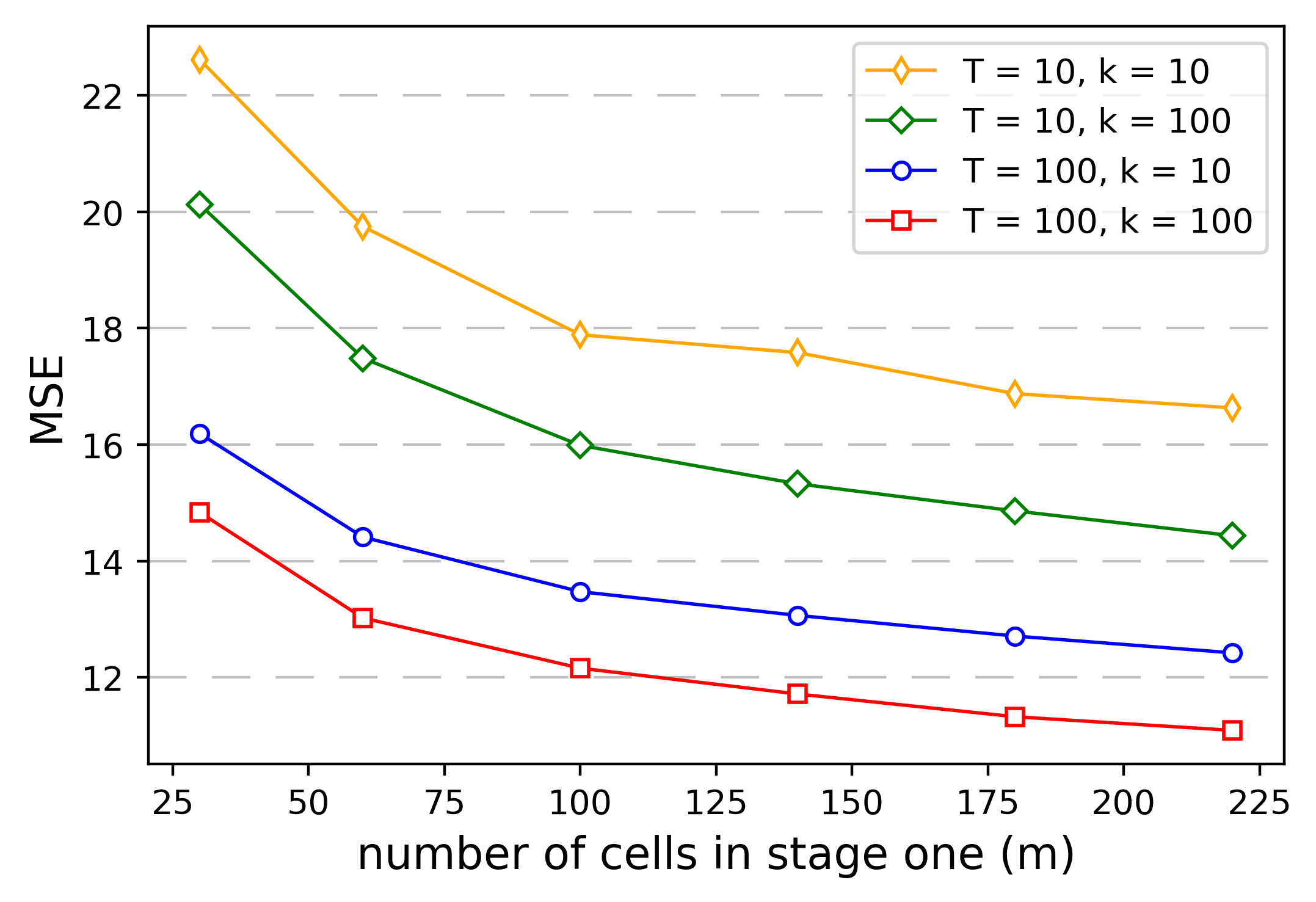}   
	\end{minipage} 
	\begin{minipage}[t]{0.43\textwidth}  
		\centering  
		\includegraphics[width=\textwidth]{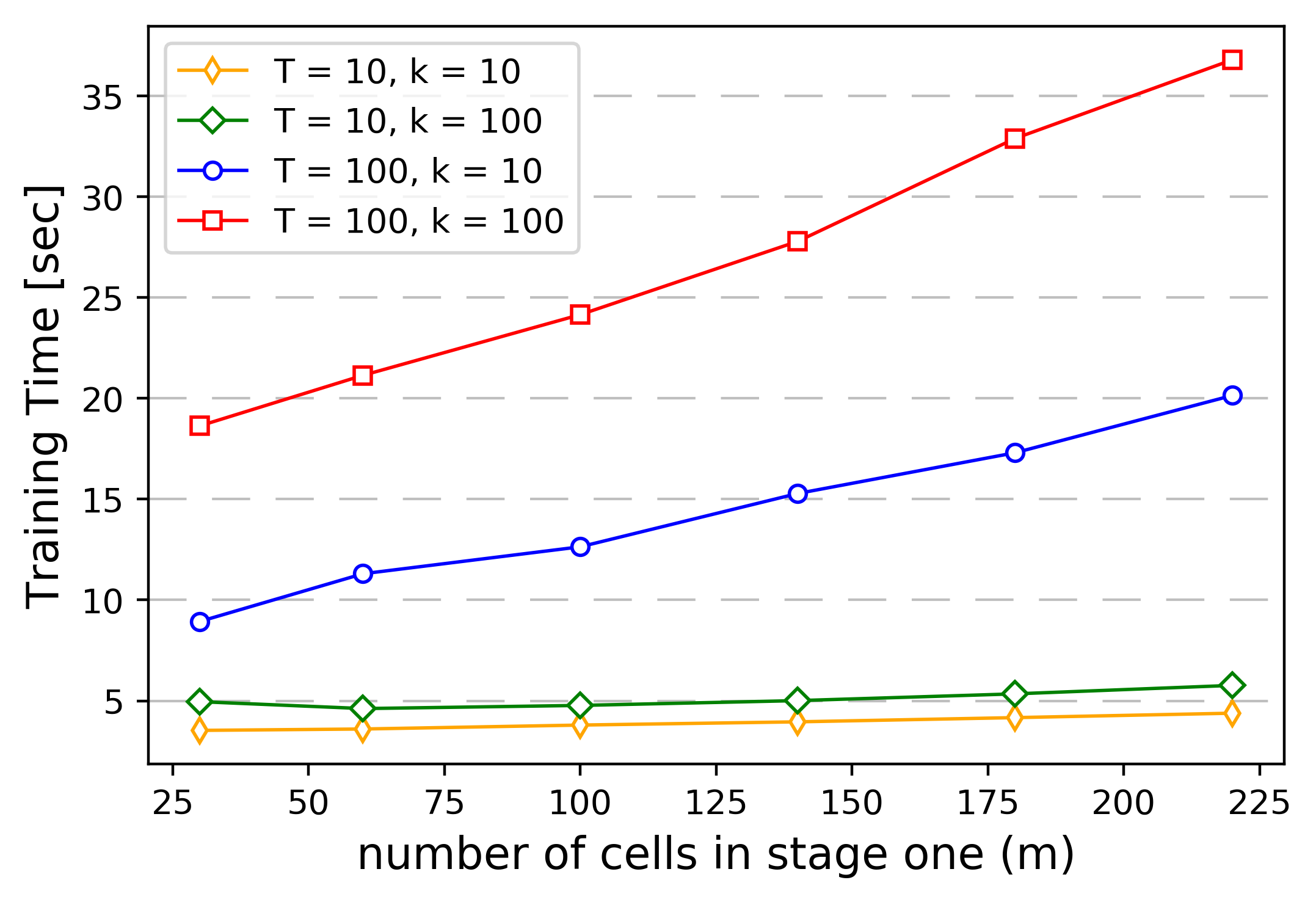}   
	\end{minipage} 
	\begin{minipage}[t]{0.43\textwidth}  
		\centering  
		\includegraphics[width=\textwidth]{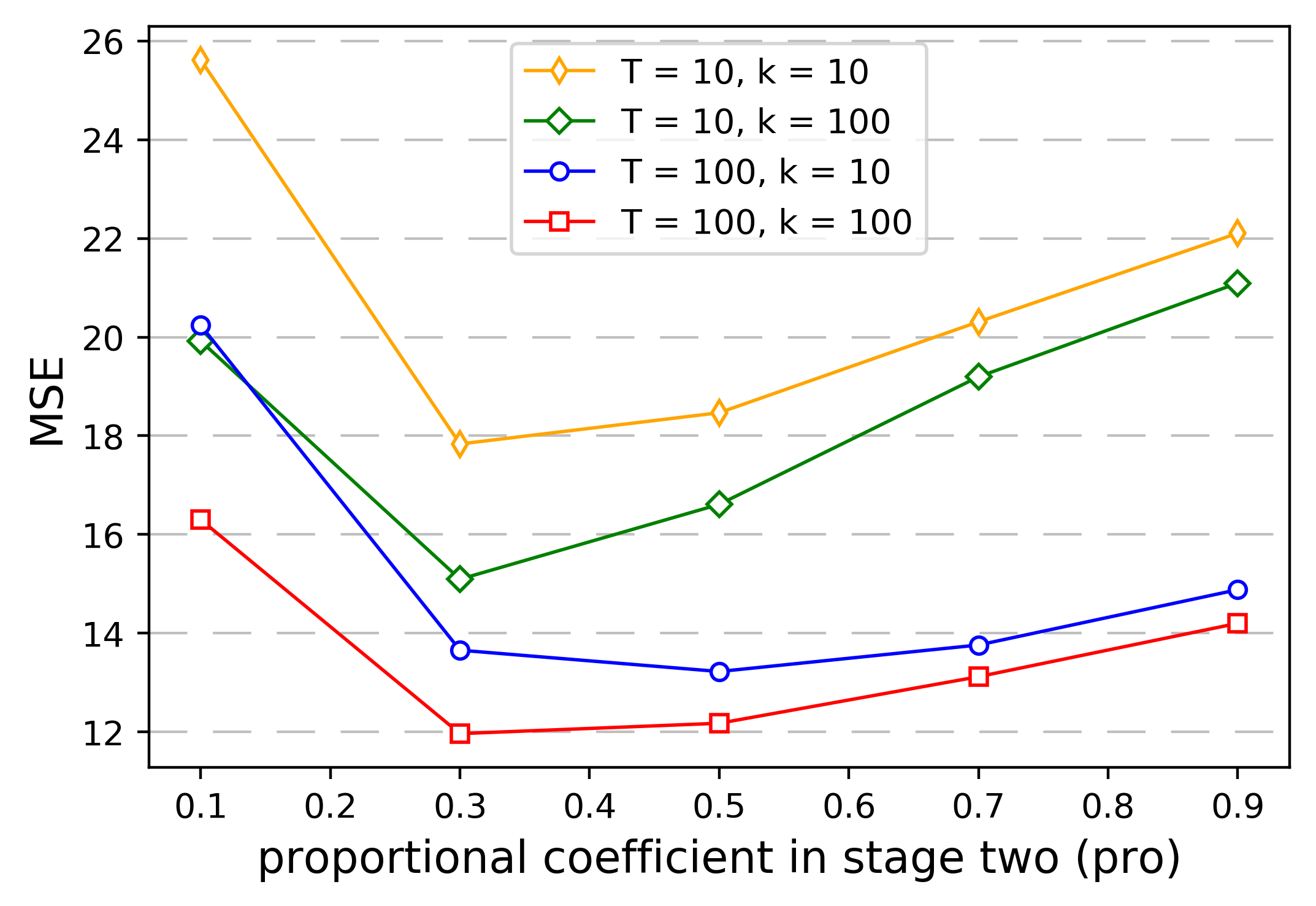}   
	\end{minipage} 
	\begin{minipage}[t]{0.43\textwidth}  
		\centering  
		\includegraphics[width=\textwidth]{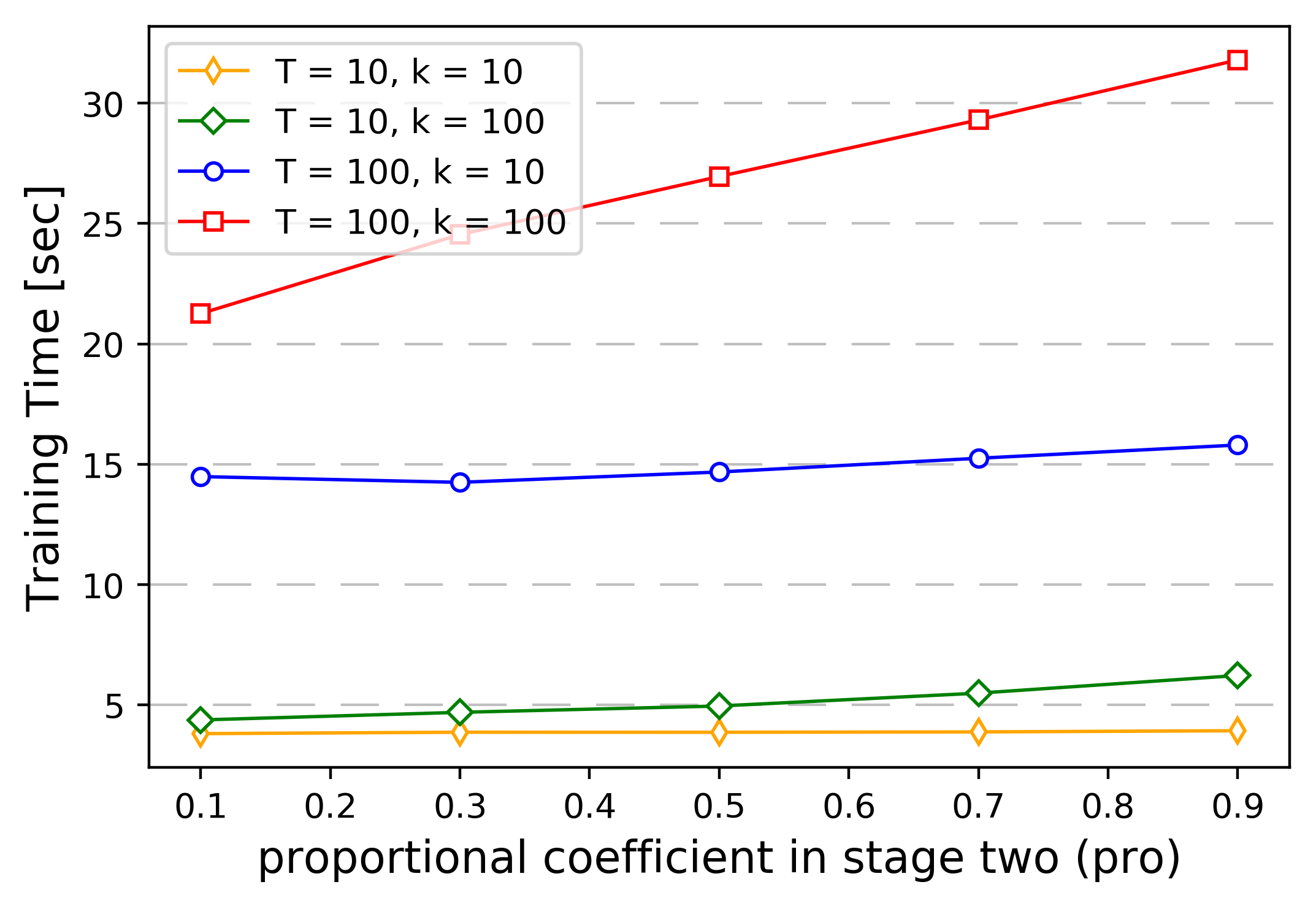}   
	\end{minipage} 
	\begin{minipage}[t]{0.43\textwidth}  
		\centering  
		\includegraphics[width=\textwidth]{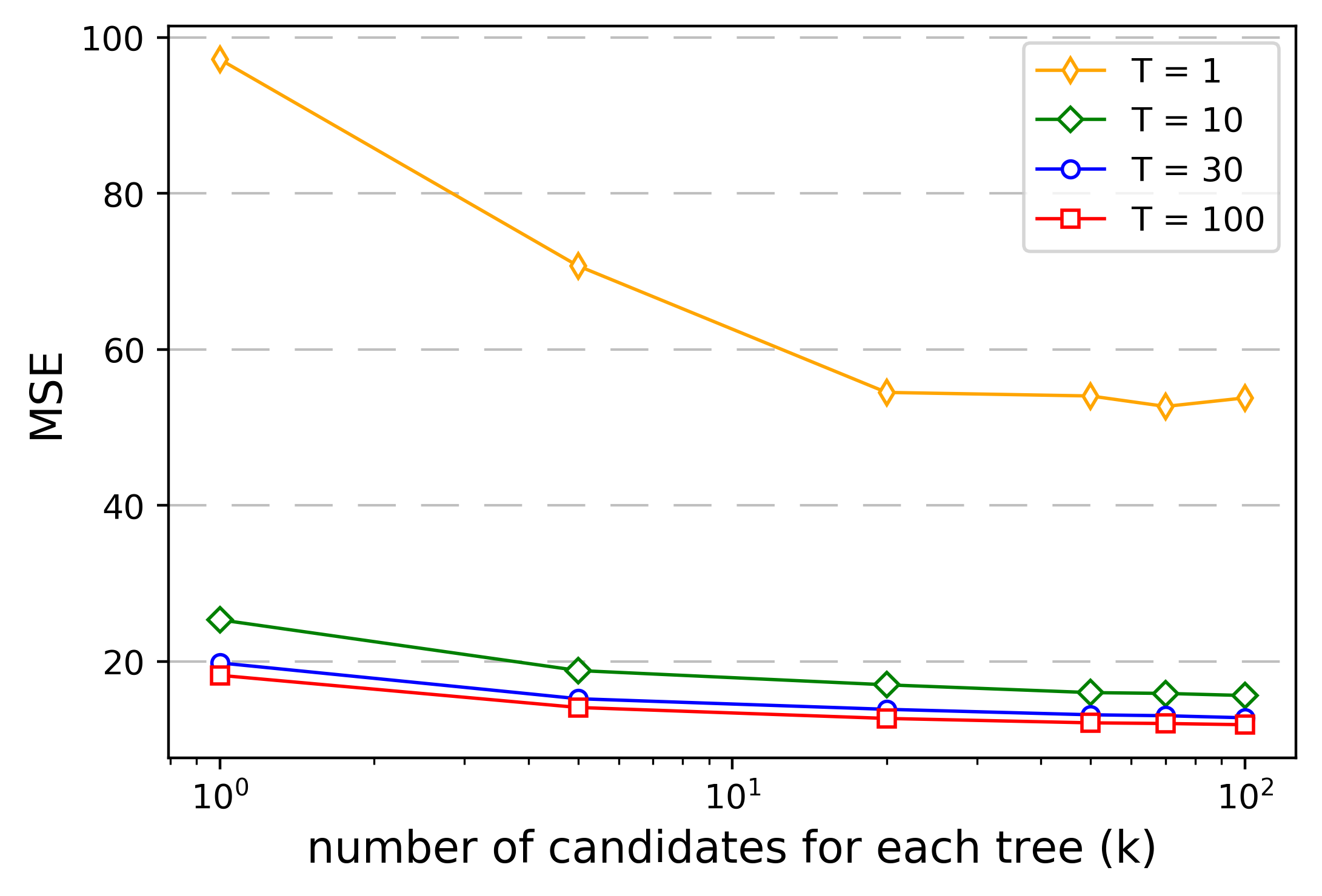}   
	\end{minipage} 
	\begin{minipage}[t]{0.43\textwidth}  
		\centering  
		\includegraphics[width=\textwidth]{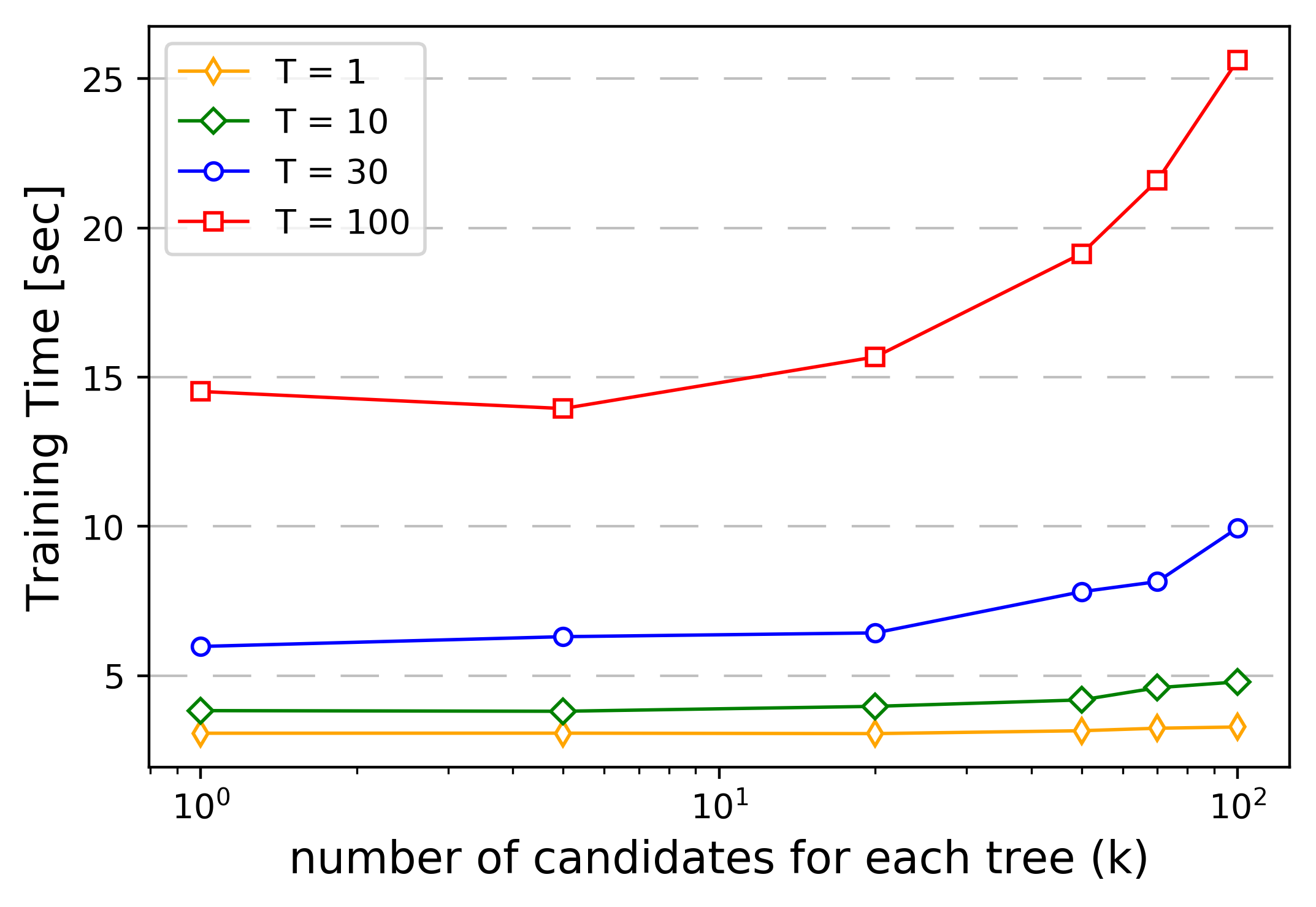}   
	\end{minipage} 
	\centering  
	\caption{\textit{MSE} and average training time of TBRF-C for the real-world data {\tt TCO}
		depending on different hyperparameters. The first row illustrates the results using different number of trees in a forest (T); figures in the second row contain change trend depending on the number of cells in the feature partition (m) in stage one; 
		figures in the third row describe the trend depending on different proportional coefficient ({\tt pro}) in stage two; 
		the last row shows the results using different number of candidates for each tree (k).}  
	\label{fig::DataAnalysisTBRF-C}
\end{figure}  

The fourth data set is the \emph{Year Prediction MSD Data Set} ({\tt MSD}) available on UCI. This data contains $463,715$ training samples and $51,630$ testing samples with $90$ attributes describing timbre average and timbre covariance. Each example is a song (track) released between 1922 and 2011. The main task is to predict the year in which a song was released based on audio features associated with the song.

\begin{figure}[htbp]
	\begin{minipage}[t]{0.43\textwidth}  
		\centering  
		\includegraphics[width=\textwidth]{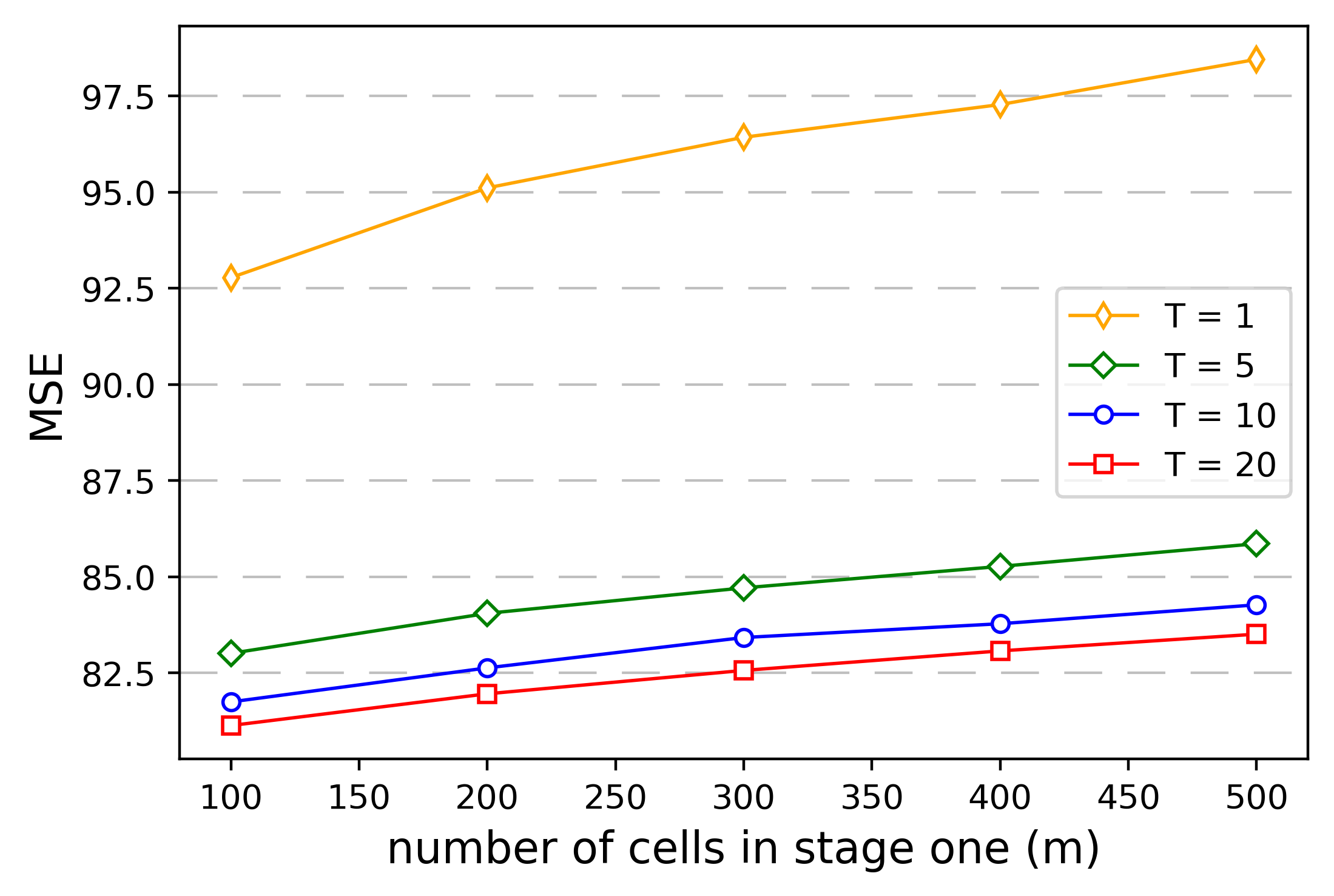}  
	\end{minipage}  
	\begin{minipage}[t]{0.43\textwidth}  
		\centering  
		\includegraphics[width=\textwidth]{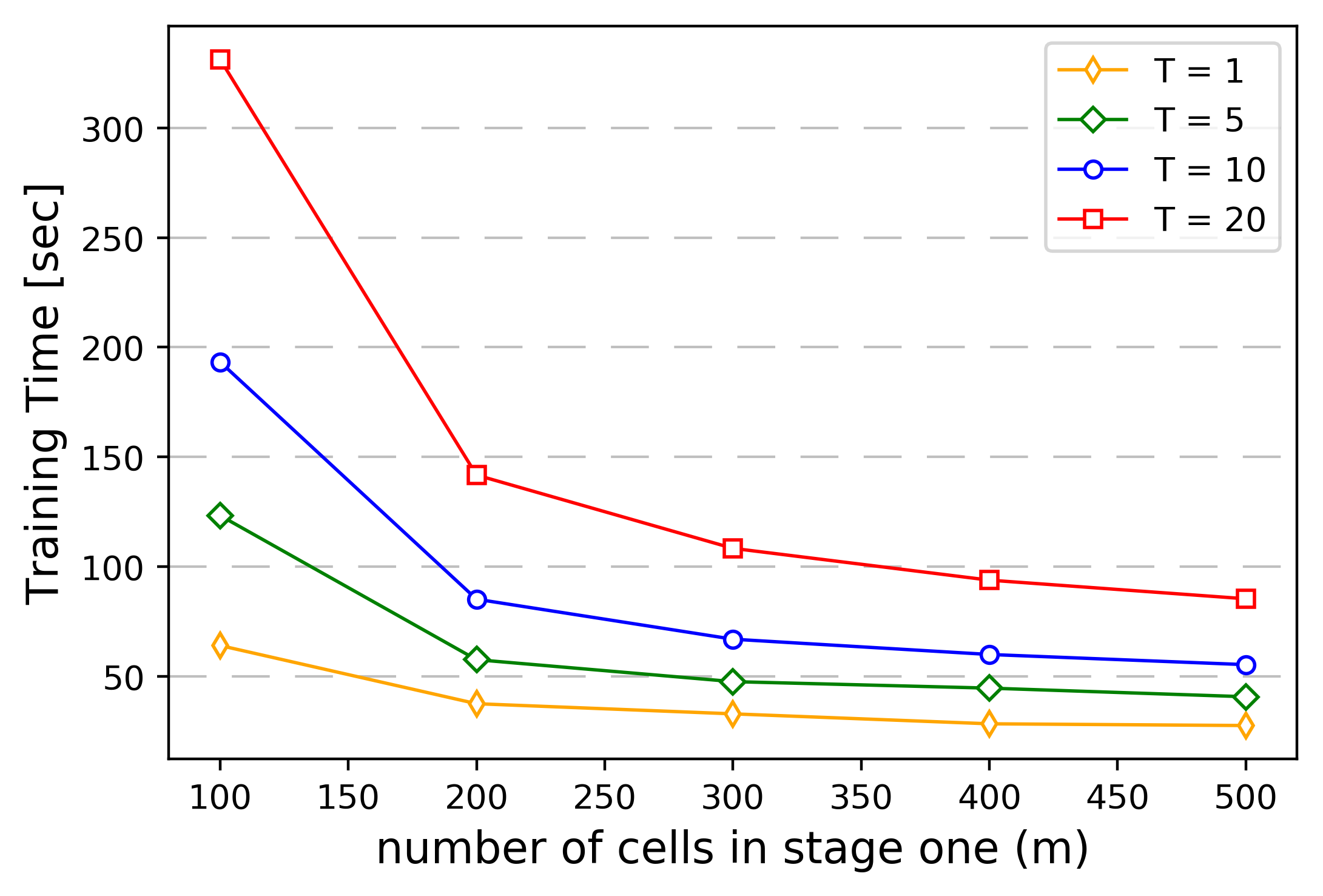}   
	\end{minipage} 
	\begin{minipage}[t]{0.43\textwidth}  
		\centering  
		\includegraphics[width=\textwidth]{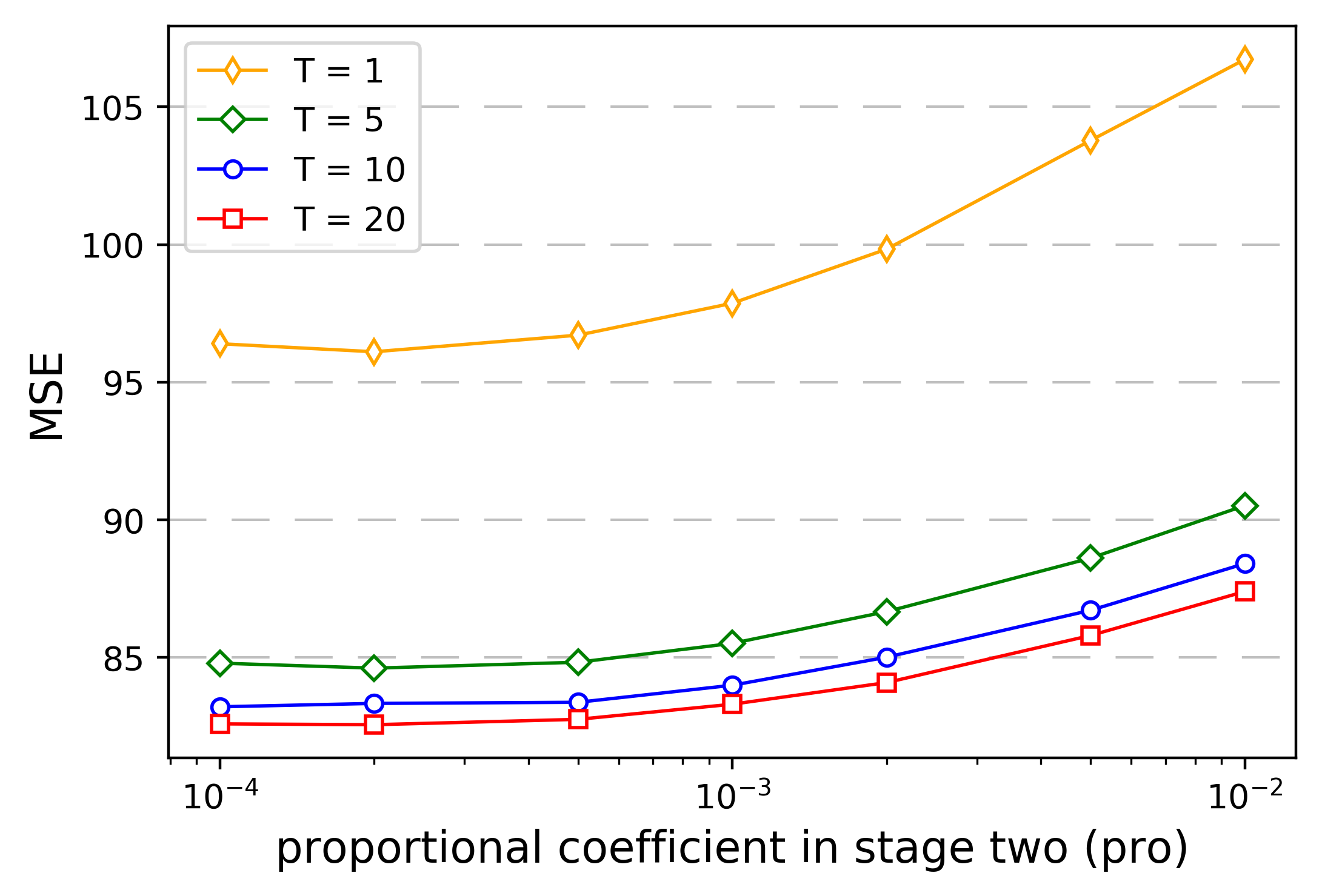}   
	\end{minipage} 
	\begin{minipage}[t]{0.43\textwidth}  
		\centering  
		\includegraphics[width=\textwidth]{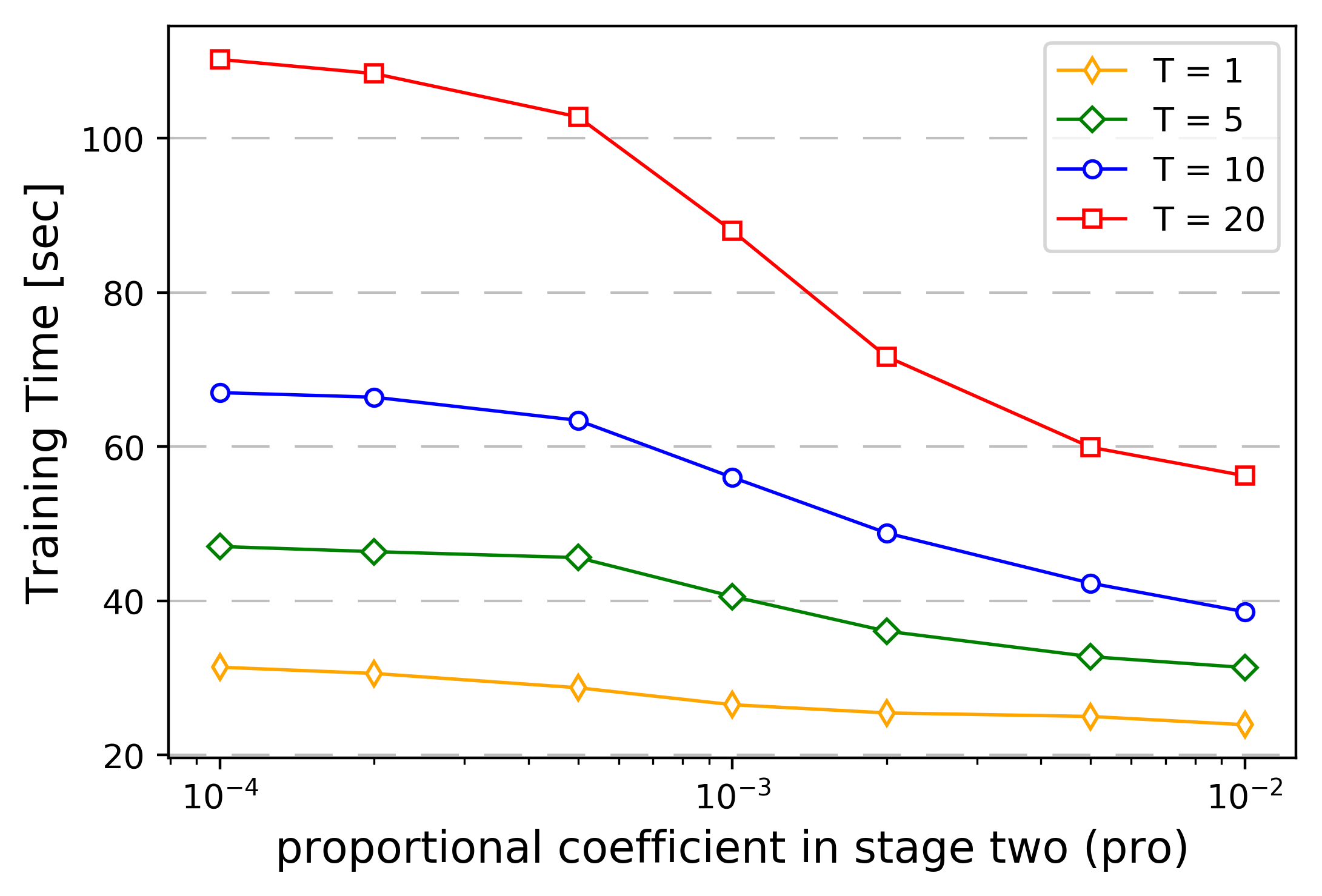}   
	\end{minipage} 
	\centering  
	\caption{\textit{MSE} and average training time of TBRF-G for the real-world data {\tt MSD}
		depending on different hyperparameters. The first row describes the change trend depending on the number of cells in the feature partition ($m$) in stage one; figures in the second row describe the trend depending on different proportional coefficient ({\tt pro}) in stage two. We mention that since the data size of {\tt MSD} is too large, we take the number of candidates for each tree $k = 1$ for simplicity in this case.}
	\label{fig::DataAnalysisTBRF-G}
\end{figure}

The parameter analysis of TBRF-C on {\tt TCO} is shown in Figure \ref{fig::DataAnalysisTBRF-C}. 
Experiments here are repeated for $50$ times. An observation on all subfigures will find that by increasing either $T$ or $k$, the accuracy of our forest algorithm will have a significant improvement with only some training time sacrificed. Moreover, for each of the different $(T, k)$ pairs under a fixed $m$, there exists a optimal {\tt pro} in stage two with regard to the test error.

The parameter analysis of TBRF-G on {\tt MSD} is shown in Figure \ref{fig::DataAnalysisTBRF-G}. Since the sample volume is relatively large, experiments here are repeated for $10$ times.
For experiments in the first row, we fixed the proportional coefficient {\tt pro} in stage two to $0.0001$. It can be observed that for fixed $T$, with $m$ or {\tt pro} increasing, which is approximately equivalent to partitioning the whole feature space into more child cells, the \textit{MSE} increases. It is because that more partitions will lead to more discontinuous boundaries, so that the ensemble predictors will also be less smooth than before. However, the increase in $m$ or {\tt pro} is beneficial for the training time. Specifically, when applying LS-SVMs with Gaussian kernel for value assignment in each child cell, the number of samples needed is reduced, which in turn speeds up the training procedure of LS-SVMs and then the whole algorithm. Moreover, with the number of trees in the forest $T$ increasing, the ensemble predictor will be smoother than ever so that lower \textit{MSE} can then be obtained. The analysis on experiments in the second row is nearly the same, only with fixed $m = 299$. To conclude, for TBRF-G, less splits on the feature space will bring less discontinuous boundaries so that the ensemble predictor will be smoother and lower \textit{MSE} can be achieved, even though the training time will be longer.


\subsection{Real Data Comparisons}
The previous sections have defined TBRF and achieved satisfying learning rates. Based on these, we wonder whether the theoretical advantages of our algorithm in terms of accuracy can still be preserved in practice compared to other state-of-the-art vertical strategies. Moreover, we are encouraged to explore from the experiments whether our vertical method can also save computational costs. Taking these into consideration, comparisons are made in between our method and the other vertical methods previously listed.

In addition to the four data sets {\tt TCO}, {\tt PTS}, {\tt SARCOS} and {\tt MSD} considered earlier in Sections \ref{sec::IllustrativeExamples} and \ref{sec::ParameterAnalysis}, other two data sets are introduced to give more comprehensive comparisons between approaches.


The fifth data set is the \emph{Appliances energy prediction Data Set} ({\tt AEP}) on UCI containing $19,735$ samples of dimension $27$ with attribute ``date'' removed from the original data set. The data is used to predict the appliances energy use in a low energy building.

The last real data set is the House-Price-8H prototask ({\tt HPP}) of the \emph{Census-house Data Set} available from the DELVE repository which contains $22,784$ samples of dimension $8$. To notify, for the sake of clarity, all house prices in the original data set has been modified to use one thousand dollars as the unit.

Note that except for {\tt MSD} and {\tt SARCOS} data sets, we divide each of the other four data sets randomly into training set and testing set with $70\%$ and $30\%$ of the total number of samples, respectively. Furthermore, all experiments conducted here are repeated for $10$ times. Now, we summarize the comparison results of TBRF, PK and VP-SVM in Table \ref{tab::1}.

\begin{table*}[h] 
	\setlength{\tabcolsep}{11pt}
	\centering
	\captionsetup{justification=centering}
	\caption{Mean Squared Error/ Training Time (seconds) on Test Data}
	\label{RMSETable}	
	\resizebox{\textwidth}{23mm}{
		\begin{tabular}{ccccccccc}
			\toprule
			\multirow{2}*{\text{Datasets}}
			& \multirow{2}*{$(n,d)$} 
			& \multicolumn{2}{c}{\text{TBRF}}
			& \multicolumn{2}{c}{\text{PK}} 
			& \multicolumn{2}{c}{\text{VP-SVM}}\\
			\cline{3-8}
			\multicolumn{2}{c}{}&\textit{MSE}&\text{Time} &\textit{MSE}&\text{Time} &\textit{MSE}&\text{Time} \\
			\hline 
			\hline
			{\tt TCO*} & $(48331, 2)$ & $\textbf{8.94}$ & $\textbf{40.99}$ & $10.77$ & $111.19$ & $10.85$ & $48.65$ \\
			{\tt PTS} & $(45730, 9)$ & $\textbf{13.43}$ & $\textbf{27.48}$ & $18.76$ & $255.31$ & $14.21$ & $60.06$\\
			{\tt SARCOS} & $(48933, 21)$ & $\textbf{1.20}$ & $32.26$ & $2.04$ & $158.69$ & 2.94 & $\textbf{27.94}$\\ 
			{\tt AEP} & $(19735, 27)$ & $\textbf{5261.9}$ & $\textbf{9.44}$ & $8194.3$ & $109.49$ & $7037.9$ & $11.79$\\ 
			{\tt HPP} & $(22784, 8)$ & $\textbf{1226.7}$ & $\textbf{14.41}$ & $1299.3$ & $195.62$ & $1273.8$ & $18.29$\\ 
			\hline
			\multirow{2}*{{\tt MSD*}}
			& \multirow{2}*{$(515345, 90)$} 
			& $\textbf{81.11}$ & $\textbf{326.85}$ & \multirow{2}*{$--$} & \multirow{2}*{$\geq 36$h}    & \multirow{2}*{$85.10$} & \multirow{2}*{$419.29$} \\ 
			& & $\textbf{80.33}$ & $3474.6$ &  &  & \\ 
			\bottomrule	
	\end{tabular}}
	\begin{tablenotes}
		\footnotesize
		\item{*} For data set {\tt TCO}, we adopt the $1$-NN-based TBRF since the \textit{MSE} of mean-based TBRF is not smaller than those of PK and VP-SVM. For other data sets, we employ mean-based TBRF which not only has high accuracy but also runs faster. Moreover, for {\tt MSD}, we employ the oblique random partitions to achieve better results. Partition methods on other data sets remain to be axis-parallel.
	\end{tablenotes}
	\label{tab::1}
\end{table*}

As is apparently observed from Table \ref{tab::1}, compared to other vertical methods, our TBRF has the lowest \textit{MSE} on all data sets, and by taking full advantages of the parallel computing, we also almost achieve the shortest training time. By contrast, PK is unable to be parallelized, leading to the longest training time.

It is well worth mentioning that aiming at significantly enhancing the prediction accuracy, appropriate variants of the TBRF are applied to the data sets proposed in Table \ref{tab::1}. For {\tt TCO}, considering that it is only a $2$-dimensional data set, we directly utilize the simplest and fastest TBRF-C to deal with the problem. We mention that we adopt the $1$-NN-based TBRF-C for this data set since the \textit{MSE} of mean-based TBRF-C is not smaller than those of PK and VP-SVM. 
For {\tt PTS}, since the performance of TBRF-L in terms of prediction accuracy is not much higher than that of TBRF-C, while the training speed is much lower, we adopt the mean-based TBRF-C for this case. It is the same for {\tt SARCOS}, {\tt AEP} and {\tt HPP}. However, {\tt MSD} data set is different from others for it not only has large volume, but more importantly, has high dimensionality. To address this, we employ the mean-based TBRF-G which has good command at dealing high-dimensional data set, since the other value assignment approaches are mainly fit for low-dimensional data. Moreover, adaptive oblique random partition is also employed to obtain better results. Having witnessed the better performance of TBRF over VP-SVM, especially on {\tt MSD}, we wonder the reasons for this phenomenon. It may be because that the boundary discontinuities thwart VP-SVM from being global smooth while our TBRF-G are able to achieve an asymptotic global smoothness thanks to the ensemble learning within the forest structure. Luckily, this viewpoint can be supported by the experimental results. For example, the lower part in Table \ref{tab::1} demonstrates that we obtain a much smaller \textit{MSE} of $81.11$ by utilizing TBRF-G and it only takes $326.85$ seconds to run the process. An even accurate result of $80.33$ as \textit{MSE} can be achieved if we sacrifice more time.

Experimental results presented so far are those we have temporarily tuned. 
More accurate results can be obtained if we sacrifice more training time, which is different from other methods for their accuracy are hard to be increased.
Readers who are interested in these experiments can try to use larger $T$, $k$ and appropriate $m$, {\tt pro} to further obtain even lower test errors.
\section{Proof}\label{sec::Proof}
To prove Proposition \ref{ApproximationError}, we need the following result which follows from Lemma 6.2 in \cite{Devroye86a}.

\begin{lemma} \label{Saturation}
	For a binary search tree with $n$ nodes, denote the saturation level $S_n$ as the number of full levels of nodes in the tree. Then for $k \geq 1$, $\log n > k + \log (k + 1)$, there holds
	\begin{align*}
	\mathrm{P} (S_n < k + 1) 
	\leq \Big( \frac{k + 1}{n} \Big) \Big( \frac{2e}{k} \log \Big( \frac{n}{k + 1} \Big) \Big)^k.
	\end{align*}
\end{lemma}

\begin{proof}[of Proposition \ref{ApproximationError}]
	Obviously we have
	\begin{align*}
	p^2 (g^*_{Z}) + \mathcal{R}_{L, \mathrm{P}}(g^*_{Z}) - \mathcal{R}_{L, \mathrm{P}}^*
	= \sum_{j=1}^m \big(\lambda_j (p^*_{Z_j})^2 + \mathcal{R}_{L_j, \mathrm{P}}(g^*_{ Z_j}) - \mathcal{R}_{L_j, \mathrm{P}}^* \big).
	\end{align*}
	Therefore, in order to analyze the global approximation error of $g^*_{Z}$, it suffices to consider the local approximation error
	of $g^*_{Z_j}$ corresponding to the cell $V_j$, $j \in \{1, \ldots, m\}$. For this purpose, we start by studying the behavior of the candidate best-scored random tree $g^*_{ Z_{j s}}$ for fixed $s \in \{1, \ldots, k_j\}$. 
	
	Denote $g^*_{Z_{j s}, p}$ as the function that minimize $\mathcal{R}_{L_j, \mathrm{P}}(g) - \mathcal{R}_{L_j, \mathrm{P}}^*$ in the function set $\hat{\mathcal{T}}_{Z_{j s}, p}$ defined in \eqref{ExtendedFunctionSet}.
	Elementary calculation shows that
	\begin{align*}
	g^*_{Z_{j s}, p}
	= \sum_{i=0}^p
	\frac{\int_{A_i} \mathbb{E}(Y|X) \ d \mathrm{P}_X}{\mathrm{P}_X(A_i)}
	\eins_{A_i}
	= \sum_{i=0}^p
	\frac{\int_{A_i} f_{L, \mathrm{P}}^* \ d \mathrm{P}_X}{\mathrm{P}_X(A_i)}
	\eins_{A_i},
	\end{align*}
	where $\mathcal{A} : = \{A_i\}_{i=0}^p$ are the rectangular cells generated by $Z_{j s}$ which forms a partition of $V_j$ and $\mathrm{P}_X(A_i)$ is the measure of $A_i$ with respect to the marginal distribution of $X$. Then
	\begin{align} \label{ApproximationOne}
	\mathcal{R}_{L_j, \mathrm{P}}(g^*_{Z_{j s}, p}) 
	- \mathcal{R}_{L_j, \mathrm{P}}^*
	& =  \| g^*_{ Z_{j s}, p} (x) - f_{L, \mathrm{P}}^* (x) \|_{L_2(\mathrm{P}_X|V_j)}^2
	\nonumber\\
	& = \| g^*_{ Z_{j s}, p} (x) \eins_{V_j} (x) - f_{L, \mathrm{P}}^* (x) \eins_{V_j} (x)\|_{L_2(\mathrm{P}_X)}^2 
	\nonumber\\
	& = \bigg\| \sum_{i=0}^p
	\frac{\int_{A_i} f_{L, \mathrm{P}}^* (z) \ d \mathrm{P}_X}{\mathrm{P}_X(A_i)}
	\eins_{A_i} (x) - \sum_{i=0}^p f_{L, \mathrm{P}}^* (x) \eins_{A_i} (x) \bigg\|_{L_2(\mathrm{P}_X)}^2 
	\nonumber\\
	& = \bigg\| \sum_{i=0}^p
	\frac{\eins_{A_i}(x)}{\mathrm{P}_X(A_i)} \int_{A_i} \big (f_{L, \mathrm{P}}^*(z) - f_{L, \mathrm{P}}^*(x) \big) \ d \mathrm{P}_X
	\bigg\|_{L_2(\mathrm{P}_X)}^2 
	\nonumber\\
	& \le \bigg\| \sum_{i=0}^p
	\frac{\eins_{A_i}(x)}{\mathrm{P}_X(A_i)} \int_{A_i} \big|f_{L, \mathrm{P}}^*(z) - f_{L, \mathrm{P}}^*(x) \big| \ d \mathrm{P}_X
	\bigg\|_{L_2(\mathrm{P}_X)}^2 
	\nonumber\\
	& = \bigg\| \sum_{i=0}^p \frac{\eins_{A_i}(x)}{\mathrm{P}_X(A_i)} \int_{A_i} \big|f_{L, \mathrm{P}}^*(z) - f_{L, \mathrm{P}}^*(x) \big| \ d \mathrm{P}_X
	\bigg\|_{L_2(\mathrm{P}_X | \mathcal{A}_1)}^2
	\nonumber\\
	& \phantom{=} + \bigg\| \sum_{i=0}^p \frac{\eins_{A_i}(x)}{\mathrm{P}_X(A_i)} \int_{A_i} \big|f_{L, \mathrm{P}}^*(z) - f_{L, \mathrm{P}}^*(x) \big| \ d \mathrm{P}_X
	\bigg\|_{L_2(\mathrm{P}_X | \mathcal{A}_2)}^2
	\end{align}
	where we decompose the error by the diameter of the cells in $\mathcal{A}$. That is
	\begin{align*}
	\mathcal{A}_1 := \{ A \in \mathcal{A} \ | \ \mathrm{diam}(A) \le h \} \quad \mathrm{and} \quad \mathcal{A}_2 := \{ A \in \mathcal{A} \ | \ \mathrm{diam}(A) > h \},
	\end{align*}
	where $\mathrm{diam}(A)$ is the diameter of $A$.
	In the following proof, we take the $L_1$-norm into consideration which leads to the definition of the diameter of the cell as
	$\mathrm{diam}(A) : = \sum_{i=1}^d V_i (A)$,
	where $V_i (A)$ denotes the length of the $i$-th dimension of the rectangle cell $A$.
	
	Let us now consider the first term in the decomposition \eqref{ApproximationOne}. For $A \in \mathcal{A}_1$, the diameter of the cell is less than $h$. Then for any $x, z \in A \in \mathcal{A}_1$, the distance between the two points satisfies $\| x - z \|_1 \leq h$.
	Using Assumption \eqref{Smoothness}, we get
	\begin{align} \label{DecompositionTermOne}
	\bigg\| \sum_{i=0}^p \frac{\eins_{A_i}(x)}{\mathrm{P}_X(A_i)} \int_{A_i} \big|f_{L, \mathrm{P}}^*(z) - f_{L, \mathrm{P}}^*(x) \big| \ d \mathrm{P}_X(z)
	\bigg\|_{L_2(\mathrm{P}_X | \mathcal{A}_1)}^2
	\leq \mathrm{P}_X (V_j) h^{2\alpha}.
	\end{align}
	For the second term in the decomposition \eqref{ApproximationOne}, 
	elementary considerations imply that
	\begin{align} \label{Approxempty} 
	\mathrm{P}_Z ( \{ A \ | \ A \in \mathcal{A}_2 \} = \emptyset)
	& = \mathrm{P}_Z ( \forall A \in \mathcal{A} : \text{diam}(A) \leq h ) 
	\nonumber \\
	& = \mathrm{P}_Z \Bigl( \max_{A \in \mathcal{A}} \  \mathrm{diam}(A) \leq h \Bigr).
	\end{align}	
	Then by Markov's inequality, we obtain
	\begin{align} \label{DecompositionMarkov}
	\mathrm{P}_Z \Big( \max_{A \in \mathcal{A}_Z} \mathrm{diam}(A) \leq h \Big)
	& \geq 1 - h^{-1} \mathbb{E}_Z \Big( \max_{A \in \mathcal{A}_Z} \ \mathrm{diam}(A) \Big) 
	\nonumber \\
	& = 1 - h^{-1} \mathbb{E}_Z \Big( \max_{A \in \mathcal{A}_Z} \sum_{i=1}^d V_i (A) \Big)  
	\nonumber \\
	& \geq 1 - h^{-1} \sum_{i=1}^d \mathbb{E}_Z \Big( \max_{A \in \mathcal{A}_Z} V_i (A) \Big).
	\end{align}
	As is mentioned previously, $Z$ is defined by $(Q_1, \ldots, Q_p, \ldots)$ where $Q_i = (L_i, R_i, S_i)$, $i=0, 1, \ldots$ in Section \ref{sec::RandomPartition}. From these we find that the randomness of $Z$ is the result of three factors, which are randomness in selecting leaves, randomness in picking dimensions, and randomness in determining cut points. Next, in order to calculate the expectation with respect to $Z$ in \eqref{DecompositionMarkov}, we conduct the following analysis suppose that the tree has already been well established.
	To be specific, for each dimension, we only need to consider one cell that has the longest side length in its respective dimension. Additionally, since there is symmetry between dimensions, it suffices to first concentrate on one dimension. For example, we consider the $i$-th dimension and denote the length of the $i$-th dimension of the corresponding cell as $\max_{A \in \mathcal{A}_Z} V_i (A) = : V_Z$. We do not have to know the exact constructing procedures of the tree entirety to calculate $\mathbb{E}_Z(V_Z)$. Instead, we still consider from three aspects which is intrinsically corresponding to the one stated above, but from a different view: 
	the total number of splits that generates that specific rectangle cell during the construction, $T_Z$; 
	the number of splits which come from the $i$-th dimension in $T_Z$, $K_Z$ and $K_Z$ follows the binomial distribution $\mathcal{B}(T_Z, 1/d)$; 
	and proportional factors $U_1, U_2, \ldots, U_{K_Z}$ which are independent and identically distributed from $\mathcal{U}[0,1]$.
	Accordingly, the expectation with regard to $Z$ can be decomposed as $\mathbb{E}_Z = \mathbb{E}_{T_Z}\mathbb{E}_{K_Z|T_Z}\mathbb{E}_{U_1\ldots U_{K_Z}|K_Z}$.
	Moreover, since $V_j$ is contained in a ball of radius $r_j$ and $\mathcal{A}$ form a partition of $V_j$, without loss of generality, we assume that the partition procedure is performed on a cube with side-length $2r_j$. 	
	According to the above analysis, the expectation in the last step in \eqref{DecompositionMarkov} can be further analyzed as follows:
	\begin{align*}
	\mathbb{E}_Z V_Z
	& \leq \mathbb{E}_{T_Z} \biggl( 
	\mathbb{E}_{K_Z} \bigg( 
	\mathbb{E}_{U_1\ldots U_{K_Z}} \biggl( 2 r_j \prod_{j=1}^{K_Z} \max \{ U_j, 1 - U_j \} \Big| K_Z \biggr) \Big| T_Z \biggr) \biggr)
	\\
	& = 2 r_j \mathbb{E}_{T_Z} \Big( \mathbb{E}_{K_Z} \Big( \Big( \mathbb{E}_U \big( \max \{ U, 1 - U \} \big) \Big)^{K_Z}  \big| T_Z \Big)\Big)
	= 2 r_j \mathbb{E}_{T_Z} \Big( \mathbb{E}_{K_Z} \Big( \big( 3 / 4 \big)^{K_Z} \big| T_Z \Big) \Big) 
	\\
	& = 2 r_j \mathbb{E}_{T_Z} \biggl( 
	\sum_{K_Z=1}^{T_Z} {T_Z \choose K_Z} 
	\biggl( \frac{3}{4} \biggr)^{K_Z} 
	\biggl( \frac{1}{d} \biggr)^{K_Z} 
	\biggl( 1 - \frac{1}{d} \biggr)^{T_Z - K_Z} \biggr)
	\\
	& = 2 r_j \mathbb{E}_{T_Z} \biggl( 1 - \frac{1}{d} + \frac{3}{4d} \biggr)^{T_Z}
	= 2 r_j \mathbb{E}_{T_Z} \biggl( 1 - \frac{1}{4d} \biggr)^{T_Z}.
	\end{align*}
	To notify, when the underlying partition rule $Z$ has number of splits $p$, 
	the partition tree is statistically related to a random binary search tree with $p+1$ external nodes and $p$ internal nodes. 
	Then, Lemma \ref{Saturation} states that for $k \ge 1$ and $\log (2p+1) > k + \log (k + 1)$,
	\begin{align*}
	\mathrm{P} (S_{2p+1} < k + 1) 
	\leq \biggl( \frac{k + 1}{2p+1} \biggr) 
	\biggl( \frac{2 e}{k} \log \biggl( \frac{2p+1}{k + 1} \biggr) \biggr)^k,
	\end{align*}
	where $S_{2p+1}$ is the \emph{saturation level}. In our specific setting, $S_{2p+1}$ can be viewed as the minimal number of splits that generates $A \in \mathcal{A}$. Now taking $k = \lfloor c_T \log (2p+1) \rfloor$ where $c_T < 1$ and $c_T (1 + \log (2 e / c_T)) < 1$, a simple calculation gives that
	\begin{align*}
	\mathrm{P}_Z (T_Z < \lfloor c_T \log (2p+1) \rfloor + 1)
	& \leq \mathrm{P} (S_{2p+1} < \lfloor c_T \log (2p+1) \rfloor + 1)
	\\
	& \leq C'(2p+1)^{c_T (1 + \log (2e / c_T)) - 1}
	\\
	& \leq C'(2p)^{c_T (1 + \log (2e / c_T)) - 1} 
	\\
	& \leq C p^{c_T (1 + \log (2e / c_T)) - 1},
	\end{align*}
	where $C'$ and $C$ are universal constants. As a result, we have
	\begin{align*}
	\mathbb{E}_Z V_Z
	& \leq 2 r_j \mathbb{E}_Z \Bigl( 1 - \frac{1}{4d} \Bigr)^{T_Z}
	\\
	& = 2 r_j \mathbb{E}_Z \Bigl( \Bigl( 1 - \frac{1}{4d} \Bigr)^{T_Z} \eins_{\{ T_Z < \lfloor c_T \log (2p+1) \rfloor + 1\}} \Bigr) \\
	& \phantom{=} + 2 r_j \mathbb{E}_Z \Bigl( \Bigl( 1 - \frac{1}{4d} \Bigr)^{T_Z} \eins_{\{T_Z \geq \lfloor c_T \log (2p+1) \rfloor + 1\}} \Bigr)
	\\
	& \leq 2 C r_j p^{c_T (1 + \log (2e / c_T)) - 1} + 2 r_j \big( 1 - 1/(4d) \big)^{c_T \log p}
	\\
	& \leq 2 C r_j p^{c_T (1 + \log (2e / c_T)) - 1} + 2 r_j p^{- c_T / (4d)},
	\end{align*}
	where the last inequality follows from the fact that $1 - 1 / x < e^{-x}$ for all $x > 1$. Since the function $f(c_T) = 1 - c_T (1 + \log (2e / c_T)) - c_T / (4d)$ is monotone decreasing on $(0, 1)$ for all $d$, numerical computation shows that the largest constant for which $1 - c_T (1 + \log (2e / c_T)) > c_T / (4d)$ holds for all $d \ge 1$ cannot be greater than $0.22563$.
	Therefore, taking $c_T = 0.22$ and $K = 2 C + 2$, there holds $\mathbb{E}_Z V_Z \le K r_j  p^{-c_T / (4d)}$. Therefore, we obtain that
	\begin{align} \label{DecompositionTwo}
	\sum_{i=1}^d \mathbb{E}_Z \Big( \max_{A \in \mathcal{A}_Z} V_i (A) \Big)
	\leq K r_j d p^{- c_T / (4d)}.
	\end{align}
	Combining \eqref{Approxempty}, \eqref{DecompositionMarkov} and \eqref{DecompositionTwo}, we have 
	\begin{align} \label{EstimationTwo}
	\mathrm{P}_Z ( \{ A \ | \ A \in \mathcal{A}_2 \} = \emptyset ) 
	\geq 1 - K r_j d h^{-1} p^{- c_T / (4d)}.
	\end{align}
	In other words, with probability at least $1 - K r_j d h^{-1} p^{- c_T / (4d)}$, the second term in the error decomposition \eqref{ApproximationOne} vanishes. 
	
	Now, the estimation \eqref{EstimationTwo} together with \eqref{DecompositionTermOne} yields that 
	\begin{align*}
	\mathcal{R}_{L_j, \mathrm{P}}(g^*_{Z_{j s}, p}) 
	- \mathcal{R}_{L_j, \mathrm{P}}^*
	\leq \mathrm{P}_X (V_j) h^{2\alpha}
	\end{align*}
	holds with probability at least $1 - K r_j d h^{-1} p^{- c_T / (4d)}$.
	With $e^{-\theta} := K r_j d h^{-1} p^{- c_T / (4d)}$, 
	simple calculation shows that with probability $\mathrm{P}_Z$ at least $1-e^{-\theta}$, there holds
	\begin{align} \label{ApproximationThree}
	\lambda_j p^2 + \mathcal{R}_{L_j, \mathrm{P}}(g^*_{Z_{j s}, p}) 
	- \mathcal{R}_{L_j, \mathrm{P}}^*
	\leq \lambda_j p^2 + \mathrm{P}_X (V_j) \big( K r_j d e^\theta p^{-c_T/(4d)} \big)^{2\alpha}.
	\end{align} 
	By minimizing both hand side of \eqref{ApproximationThree}, we obtain that
	\begin{align*}
	\lambda_j (p^*_{Z_{j s}})^2 + \mathcal{R}_{L_j, \mathrm{P}} (g^*_{Z_{j s}}) - \mathcal{R}_{L_j, \mathrm{P}}^*
	\leq c_{\alpha d}
	\Bigl( \mathrm{P}_X(V_j) r_j^{2 \alpha} e^{2 \alpha \theta} \Bigr)^{\frac{4 d}{c_T \alpha + 4 d}}
	\lambda_j^{\frac{c_T \alpha}{c_T \alpha + 4 d}}
	\end{align*}
	with the constant $c_{\alpha d}$ concerning only $\alpha$ and $d$.
	
	Then, the definition of $g^*_{ Z_j}$ in \eqref{BestScoreMinimizerPopulation} together with the independence of the $k_j$ trials implies that
	with probability at least $1-e^{- \theta}$ there holds
	\begin{align*}
	\lambda_j (p^*_{Z_j})^2 + \mathcal{R}_{L_j, \mathrm{P}}(g^*_{ Z_j}) - \mathcal{R}_{L_j, \mathrm{P}}^* 
	\leq 
	c_{\alpha d}
	\Bigl( \mathrm{P}_X(V_j)  r_j^{2 \alpha} e^{2 \alpha \theta / k_j} \Bigr)^{\frac{4 d}{c_T \alpha + 4 d}}
	\lambda_j^{\frac{c_T \alpha}{c_T \alpha + 4 d}}.
	\end{align*}
	Using the union bound, we obtain that
	\begin{align*}
	p^2 (g^*_{ Z}) & + \mathcal{R}_{L, \mathrm{P}}(g^*_{ Z}) - \mathcal{R}_{L, \mathrm{P}}^* 
	\leq c_{\alpha d} \sum_{j=1}^m
	\Bigl( \mathrm{P}_X(V_j)  r_j^{2 \alpha} e^{2 \alpha \theta / k_j} \Bigr)^{\frac{4 d}{c_T \alpha + 4 d}}
	\lambda_j^{\frac{c_T \alpha}{c_T \alpha + 4 d}}
	\end{align*}
	with probability at least $1 - m e^{- \theta}$. Once again with variable transformation $e^{-\tau} : = m e^{-\theta}$, we get the final result that
	\begin{align*}
	p^2 (g^*_{ Z}) & + \mathcal{R}_{L, \mathrm{P}}(g^*_{Z}) - \mathcal{R}_{L, \mathrm{P}}^* 
	\leq c_{\alpha d} \sum_{j=1}^m
	\Bigl( \mathrm{P}_X(V_j)  r_j^{2 \alpha} m^{2 \alpha / k_j} e^{2 \alpha \tau / k_j} \Bigr)^{\frac{4 d}{c_T \alpha + 4 d}}
	\lambda_j^{\frac{c_T \alpha}{c_T \alpha + 4 d}}
	\end{align*}
	with probability $\mathrm{P}_Z$ at least $1 - e^{- \tau}$.
\end{proof}

\begin{proof}[of Lemma \ref{VCIndex}]
	This proof is conducted from the perspective of geometric constructions. 
	Firstly, we concentrate on partition with the number of splits $p=1$. Because of the dimension of the feature space is $d$,  the smallest number of sample points that cannot be divided by $p=1$ split is $d+2$. Concretely, owing to the fact that $d$ points can be used to form $d-1$ independent vectors and hence a hyperplane in a $d$-dimensional space, we might take the following case into consideration: There is a hyperplane consisting of $d$ points all from one class, say class $A$, and two points $p_1^B$, $p_2^B$ from the opposite class $B$ located on the opposite sides of this hyperplane, respectively. We denote this hyperplane by $H_1^A$. In this case, points from two classes cannot be separated by one split (since the positions are $p_1^B, H_1^A, p_2^B$), so that we have $\mathrm{VC}(\mathcal{B}_1) \leq d + 2$. 
	
	Next, when the partition is with the number of splits $p=2$, we analyze in the similar way only by extending the above case a little bit. Now, we pick either of the two single sample points located on opposite side of the $H_1^A$, and add $d-1$ more points from class $B$ to it. Then, they together can form a hyperplane $H_2^B$ parallel to $H_1^A$. After that, we place one more sample point from class $A$ to the side of this newly constructed hyperplane $H_2^B$. In this case, the location of these two single points and two hyperplanes are $p_1^B, H_1^A, H_2^B, p_2^A$. Apparently, $p=2$ splits cannot separate these $2d+2$ points. As a result, we have $\mathrm{VC}(\mathcal{B}_2) \leq 2d + 2$.
	
	Inductively, the above analysis can be extended to the general case of number of splits $p \in \mathbb{N}$. In this manner, we need to add points continuously to form $p$ mutually parallel hyperplanes where any two adjacent hyperplanes should be constructed from different classes. Without loss of generality, we consider the case for $p=2k+1$, $k \in \mathbb{N}$, where two points (denoted as $p_1^B$, $p_2^B$) from class $B$ and $2k+1$ alternately appearing hyperplanes form the space locations: $p_1^B, H_1^A, H_2^B, H_3^A, H_4^B, \ldots, H_{(2k+1)}^A, p_2^B$. 
	Accordingly, the smallest number of points that cannot be divided by $p$ splits is $dp+2$, leading to $\mathrm{VC}(\mathcal{B}_p) \leq d p + 2$.
	
	Moreover, hyperplanes can be generated both vertically and obliquely according to the proof needs. This completes the proof.
\end{proof}

\begin{proof}[of Lemma \ref{BpTpCoveringNumbers}]
	The inequality \eqref{BpCoveringNumber} follows directly from Lemma \ref{VCIndex} and Theorem 9.2 in \cite{Kosorok08a}.
	
	For the inequality \eqref{TpCoveringNumber}, denote the covering number of $\eins_{\mathcal{B}_p}$ with respect to $\| \cdot \|_{L_2(Q)}$ as $\mathcal{N}(\varepsilon) : = \mathcal{N} \bigl( \eins_{\mathcal{B}_p}, \|\cdot\|_{L_2(Q)}, \varepsilon \bigr)$. Then, there exist $\eins_{B_1}, \ldots, \eins_{B_{\mathcal{N}(\varepsilon)}} \in \eins_{\mathcal{B}_p}$ such that the function set $\{\eins_{B_1}, \ldots, \eins_{B_{\mathcal{N}(\varepsilon)}}\}$ is an $\varepsilon$-net of $\eins_{\mathcal{B}_p}$ with respect to $\|\cdot\|_{L_2(Q)}$. This implies that for any $\eins_{B} \in \eins_{\mathcal{B}_p}$, there exists a $j \in \{ 1, \ldots, \mathcal{N}(\varepsilon) \}$ such that $\| \eins_B - \eins_{B_j} \|_{L_2(Q)} \leq \varepsilon$. 
	Now, for all $g \in \tilde{\mathcal{T}}_p$, the equivalent definition \eqref{Tp2} of $\tilde{\mathcal{T}}_p$ tells us that there exists a $\eins_{B} \in \eins_{\mathcal{B}_p}$ such that $g$ can be written as $g = \eins_B - \eins_{B^c} = 2 \eins_B - 1$. The above discussion yields that there exists a $j \in \{ 1, \ldots, \mathcal{N}(\varepsilon) \}$ such that for $g_j := 2 \eins_{B_j} - 1$, there holds
	\begin{align*}
	\| g - g_j \|_{L_2(Q)} 
	& = \| (2 \eins_B - 1) - (2 \eins_{B_j} - 1) \|_{L_2(Q)} 
	 = \| 2 \eins_B - 2 \eins_{B_j} \|_{L_2(Q)} 
	\\
	& = 2 \| \eins_B - \eins_{B_j} \|_{L_2(Q)} 
	 \leq 2 \varepsilon.
	\end{align*}
	This implies that $\{ g_1, \ldots, g_{\mathcal{N}(\varepsilon)} \}$ is a $2 \varepsilon$-net of $\tilde{\mathcal{T}}_p$ with respect to $\|\cdot\|_{L_2(Q)}$. Consequently, we obtain
	\begin{align*}
	\mathcal{N} \bigl( \tilde{\mathcal{T}}_p, \|\cdot\|_{L_2(Q)}, \varepsilon \bigr) 
	\leq \mathcal{N} \bigl( \eins_{\mathcal{B}_p}, \|\cdot\|_{L_2(Q)}, \varepsilon / 2 \bigr) 
	\leq K (d p + 2) (4 e)^{d p + 2} ( 2 / \varepsilon )^{2 (d p + 1)},
	\end{align*}
	we thus proved the assertion.
\end{proof}

\begin{proof}[of Lemma \ref{HrEntropyNumber}]
	First, we notice that for all $g \in \tilde{\mathcal{T}}_r$, the number of splits $p$ can be upper bounded by $q : = \sum_{j=1}^m \lceil (r / \lambda_j)^{1/2} \rceil$. 
	Then, the nested relation implies that $\tilde{\mathcal{T}}_r \subset \tilde{\mathcal{T}}_q$. 
	Therefore, Lemma \ref{BpTpCoveringNumbers} implies that the covering number of $\tilde{\mathcal{T}}_r$ with respect to $L_2(\mathrm{D})$ satisfies
	\begin{align} \label{GrCoveringNumber}
	\mathcal{N} ( \tilde{\mathcal{T}}_r, \|\cdot\|_{L_2(\mathrm{D})}, \varepsilon)
	\leq \mathcal{N} ( \tilde{\mathcal{T}}_q, \|\cdot\|_{L_2(\mathrm{D})}, \varepsilon)
	\leq  K (d q + 2) (4e)^{d q + 2} ( 2 / \varepsilon)^{2 ( d q + 1)}.
	\end{align}
	
	For the least square loss $L$, we have for any $h, \tilde{h} \in \tilde{\mathcal{H}}_r$,
	\begin{align*}
	\| h - \tilde{h} \|_{L_2(\mathrm{D})} 
	& = \biggl( \frac{1}{n} \sum_{i=1}^n \bigl( h(x_i, y_i) - \tilde{h}(x_i, y_i) \bigr)^2 \biggr)^{1/2} 
	 = 2  \biggl( \frac{1}{n} \sum_{i=1}^n \bigl( g(x) - \tilde{g} (x) \bigr)^2 \biggr)^{1/2}
	\\
	& = 2 \| g - \tilde{g} \|_{L_2(\mathrm{D})}. 
	\end{align*}
	Therefore, similarly as the proof in Lemma \ref{BpTpCoveringNumbers}, we obtain
	\begin{align*}
	\mathcal{N} (\tilde{\mathcal{H}}_r, \|\cdot\|_{L_2(\mathrm{D})}, \varepsilon)
	\leq \mathcal{N} (\tilde{\mathcal{G}}_r, \|\cdot\|_{L_2(\mathrm{D})}, 2 \varepsilon)
	\leq K (d q + 2) (4e)^{d q + 2} ( 4 / \varepsilon)^{2 ( d q + 1)},
	\end{align*}
	where the later inequality follows from the estimate \eqref{GrCoveringNumber}. Elementary calculations show that for any $ 0< \varepsilon < 1/ \max \{ e, K \}$ and $q \geq 1$, there holds
	\begin{align*}
	\log \mathcal{N} (\tilde{\mathcal{H}}_r, \|\cdot\|_{L_2(\mathrm{D})}, \varepsilon)
	\leq 25 d q \log ( 1 / \varepsilon )
	\leq 25 d \sum_{j=1}^m (r / \lambda_j)^{1/2} \log ( 1 / \varepsilon ).
	\end{align*}
	Consequently, for all $\delta \in (0, 1)$, we have
	\begin{align} \label{Differenting}
	\sup_{\varepsilon \in (0, 1 / \max \{ e, K \})} \  \varepsilon^{2 \delta} \log \mathcal{N} (\tilde{\mathcal{H}}_r, \|\cdot\|_{L_2(\mathrm{D})}, \varepsilon)
	\leq 25 d \sum_{j=1}^m (r / \lambda_j)^{1/2} \sup_{\varepsilon \in (0,1)} \varepsilon^{2 \delta} \log ( 1 / \varepsilon ).
	\end{align}
	For any fixed $\delta \in (0, 1)$,
	simple analysis shows that the right hand side of \eqref{Differenting} is maximized
	at $\varepsilon^* = e^{- 1 / (2 \delta)}$ and consequently we obtain
	\begin{align*}
	\log \mathcal{N} (\tilde{\mathcal{H}}_r, \|\cdot\|_{L_2(\mathrm{D})}, \varepsilon)
	\le \frac{25d}{2 e \delta} \sum_{j=1}^m (r / \lambda_j)^{1/2} \varepsilon^{-2 \delta}.
	\end{align*}
	Then Exercise 6.8 in \cite{StCh08} implies that the entropy number bound of $\tilde{\mathcal{H}}_r$ with respect to $L_2(\mathrm{D})$ satisfies
	\begin{align} \label{EntropyNumberHr}
	e_i (\tilde{\mathcal{H}}_r, \|\cdot\|_{L_2(\mathrm{D})})
	\leq \biggl( \frac{75 d}{2 e \delta} \sum_{j=1}^m (r / \lambda_j)^{1/2} \biggr)^{1/(2 \delta)}
	i^{- 1/(2 \delta)}.
	\end{align}
	Obviously this bound holds for $\mathbb{E}_{\mathrm{D} \sim \mathrm{P}} \, e_i (\tilde{\mathcal{H}}_r, \|\cdot\|_{L_2(\mathrm{D})})$ as well. The proof is finished.
\end{proof}

\begin{proof}[of Lemma \ref{Rademacher}]
	First notice that for all $h \in \tilde{\mathcal{H}}_r$, there holds
	\begin{align*}
	\mathbb{E}_{\mathrm{P}} h^2 \leq 16 r = : \sigma^2.
	\end{align*}
	Now $\|h\|_{\infty} \leq 4 = : B$,
	$a := \big( 75d / (2e \delta)\sum_{j=1}^m (r / \lambda_j)^{1/2} \big)^{1/(2 \delta)} \geq B$ in Lemma \ref{HrEntropyNumber} together with Theorem 7.16 in \cite{StCh08} yields that
	\begin{align*}
	\mathbb{E}_{\mathrm{D} \sim \mathrm{P}} \mathrm{Rad}_{\mathrm{D}} & (\tilde{\mathcal{H}}_r, n) \\
	& \le
	\mathrm{max} \bigg\{
	10 c_{\delta,1} 
	\bigg(\frac{6d}{e\delta n} \sum_{j=1}^m \lambda_j^{-1/2}\bigg)^{\frac{1}{2}} r^{\frac{3-2\delta}{4}},
	5 c_{\delta,2}  \bigg(\frac{6d}{e\delta n}\sum_{j=1}^m \lambda_j^{-1/2}\bigg)^{\frac{1}{1+\delta}} r^{\frac{1}{2+2\delta}}
	\bigg\},
	\end{align*}
	where
	\begin{align*}
	c_{\delta,1} : = \frac{2 \sqrt{\log 256} c_\delta^\delta}{(\sqrt{2} - 1) (1 - \delta) 2^{\delta/2}}
	\quad
	\mathrm{and}
	\quad
	c_{\delta,2} : = \Big( \frac{8 \sqrt{\log 16} c_\delta^\delta}{(\sqrt{2} - 1) (1 - \delta) 4^\delta}
	\Big)^{\frac{2}{1+\delta}},
	\end{align*}
	with
	\begin{align*}
	c_\delta : = \frac{\sqrt{2} - 1}{\sqrt{2} - 2^{\frac{2\delta - 1}{2\delta}}} \cdot \frac{1 - \delta}{\delta}.
	\end{align*}
	We thus derive that
	\begin{align*}
	\mathbb{E}_{\mathrm{D} \sim \mathrm{P}} \mathrm{Rad}_{\mathrm{D}} & (\mathcal{H}_r, n) 
	\leq M \mathbb{E}_{\mathrm{D} \sim \mathrm{P}} \mathrm{Rad}_{\mathrm{D}} (\tilde{\mathcal{H}}_r, n) 
	\\
	& \leq
	\mathrm{max} \bigg\{
	10 M c_{\delta,1}
	\bigg(\frac{6d}{e\delta n} \sum_{j=1}^m \lambda_j^{-1/2}\bigg)^{\frac{1}{2}} r^{\frac{3-2\delta}{4}},
	5 M c_{\delta,2} \bigg(\frac{6d}{e\delta n}\sum_{j=1}^m \lambda_j^{-1/2}\bigg)^{\frac{1}{1+\delta}} r^{\frac{1}{2+2\delta}}
	\bigg\},
	\end{align*}
	where $c_{\delta,1}$ and $c_{\delta,2}$ are the same as defined above.
\end{proof}

\begin{proof}[of Theorem \ref{OracleVPtree}]
	For the least square loss $L$, the supremum bound
	\begin{align*}
	L(x, y, t) \leq B, \quad 
	\forall \ (x, y) \in \mathcal{X} \times \mathcal{Y}, \ t \in [-M, M]
	\end{align*}
	holds for $B = 4 M^2$ and the variance bound
	\begin{align*}
	\mathbb{E}_{X \times Y} \big(L \circ g - L \circ f_{L, \mathrm{P}}^*\big)^2 
	\leq V \big(\mathbb{E}_{X \times Y} \big(L \circ g - L \circ f_{L, \mathrm{P}}^*\big)\big)^\vartheta, \quad
	\forall \ g : \mathcal{X} \to \mathcal{Y}
	\end{align*}
	holds for $V = 16 M^2$ and $\vartheta = 1$. 
	Moreover, Lemma \ref{Rademacher} implies that the expected empirical Rademacher average of $\mathcal{H}_r$ defined in \eqref{Hr} can be upper bounded by the function $\varphi_n(r)$ as
	\begin{align*}
	\varphi_n (r) : = \mathrm{max} \bigg\{
	10 M c_{\delta,1}
	\bigg(\frac{6d}{e\delta n} \sum_{j=1}^m \lambda_j^{-1/2}\bigg)^{\frac{1}{2}} r^{\frac{3-2\delta}{4}},
	5 M c_{\delta,2} \bigg(\frac{6d}{e\delta n}\sum_{j=1}^m \lambda_j^{-1/2}\bigg)^{\frac{1}{1+\delta}} r^{\frac{1}{2+2\delta}}
	\bigg\},
	\end{align*}
	where the constants $c_{\delta,1}$ and $c_{\delta,2}$ are defined as in the proof of Lemma \ref{Rademacher}. It can be easily concluded that for this $\varphi_n$, the condition that $\varphi_n(4r) \le 2 \sqrt{2} \varphi_n(r)$ is satisfied. This implies that the statements of the Peeling Theorem 7.7 in \cite{StCh08} still hold for $\varphi_n(4r) \le 2 \sqrt{2} \varphi_n(r)$. Accordingly, the assumption concerning $\varphi_n$ and $r$ in Theorem 7.20 in \cite{StCh08} should be modified to $\varphi_n(4r) \le 2 \sqrt{2} \varphi_n(r)$ and
	\begin{align*}
	r \ge \max \big\{ 75 \varphi_n(r), 1152 M^2 \tau / n, r^* \big\}
	\end{align*}
	respectively. Some elementary calculations show that the condition $r \ge 75 \varphi_n(r)$ is satisfied if
	\begin{align*}
	r \ge c_{d \delta M} \bigg(\frac{1}{n}\sum_{j=1}^m \lambda_j^{-1/2}\bigg)^{\frac{2}{1+2\delta}},
	\end{align*}
	where the constant $c_{d \delta}$ depends only on $d$, $\delta$ and $M$.
	In the end, from definition \eqref{RStar} we have $r^* \le p^2 (g_{L, D, Z}) + \mathcal{R}_{L, \mathrm{P}} (g_{L, \mathrm{P}, Z}) - \mathcal{R}_{L, \mathrm{P}}^*$ and the assertion follows from Theorem 7.20 in \cite{StCh08}.
\end{proof}

\begin{proof}[of Theorem \ref{RateVPTree}]
	Theorem \ref{OracleVPtree} together with Proposition \ref{ApproximationError} implies that with probability $\mathrm{P}_{(X \times Y) \otimes Z}$ at least $1 - 4 e^{-\tau}$, there holds that
	\begin{align} \label{RateOne}
	p^2 (g_{Z}) + \mathcal{R}_{L, \mathrm{P}} (g_{Z}) - \mathcal{R}_{L, \mathrm{P}}^*
	& \leq 
	9 c_{\alpha d}
	\sum_{j=1}^m
	\Bigl( \mathrm{P}_X(V_j) r_j^{2 \alpha} m^{2 \alpha / k_j} e^{2 \alpha \tau / k_j} \Bigr)^{\frac{4 d}{c_T \alpha + 4 d}}
	\lambda_j^{\frac{c_T \alpha}{c_T \alpha + 4 d}}
	\nonumber\\
	& \phantom{=} 
	+ 3 c_{d \delta} \bigg(\frac{1}{n}\sum_{j=1}^m \lambda_j^{-1/2}\bigg)^{\frac{2}{1+2\delta}}
	+ 3456 M^2 \tau / n,
	\end{align}
	where $c_{d \delta}$ and $c_{\alpha d}$ are the constants defined in Theorem \ref{OracleVPtree} and Proposition \ref{ApproximationError} respectively. To minimize the right hand side of \eqref{RateOne}, we choose the regularization parameter as
	\begin{align*}
	\lambda_j : = c_j n^{-\frac{c_T \alpha + 4d}{c_T \alpha (1 + \delta) + 2d}}, \quad \forall j \in \{1, \ldots, m\},
	\end{align*}
	where $c_j$ is a constant depending on $\alpha, \tau, m, d, \delta, M, r_j, k_j$ and $\mathrm{P}_X (V_j)$. Therefore, we obtain that
	\begin{align*} 
	p^2 (g_{Z}) 
	+ \mathcal{R}_{L, \mathrm{P}} (g_{Z}) 
	- \mathcal{R}_{L, \mathrm{P}}^* 
	\leq
	C n^{-\frac{c_T \alpha}{c_T \alpha (1 + \delta) + 2d}},
	\end{align*}
	with the constant $C$ depending on $\alpha, \tau, \delta, d, m, M$ and $\{r_j, k_j, \mathrm{P}_X (V_j) \}_{j=1}^m$. 
\end{proof}

\begin{proof} [of Theorem \ref{RateVPForest}]
	For the least square loss $L$, there holds
	\begin{align*}
	\mathcal{R}_{L, \mathrm{P}} (f) - \mathcal{R}_{L, \mathrm{P}}^* = \int_\mathcal{X} \big(f- f_{L, \mathrm{P}}^* \big)^2 \ d \mathrm{P}_X 
	= \| f - f_{L, \mathrm{P}}^* \|_{L_2(\mathrm{P}_X)}^2.
	\end{align*}
	Consequently, combining with Cauchy-Schwarz inequality we have for the two-stage best-scored random forest \eqref{TBRF}
	\begin{align*}
	\mathcal{R}_{L, \mathrm{P}} (f_{Z}) - \mathcal{R}_{L, \mathrm{P}}^*
	& = \int_\mathcal{X} \Big(\frac{1}{T} \sum_{t=1}^T g_{ Z_t} - f_{L, \mathrm{P}}^* \Big)^2 \ d \mathrm{P}_X 
	\\
	&\leq \frac{1}{T} \sum_{t=1}^T \int_{\mathcal{X}} \Big(g_{Z_t} - f_{L, \mathrm{P}}^* \Big)^2 \ d \mathrm{P}_X 
	= \frac{1}{T} \sum_{t=1}^T \Big( \mathcal{R}_{L, \mathrm{P}} (g_{Z_t}) - \mathcal{R}_{L,  \mathrm{P}}^* \Big).
	\end{align*}
	The union bound together with Theorem \ref {RateVPTree} states that
	\begin{align*}
	\mathrm{P}_{(X \times Y) \otimes Z} 
	& \Big( \mathcal{R}_{L, \mathrm{P}} (f_{ Z}) - \mathcal{R}_{L, \mathrm{P}}^*
	> C n^{-\frac{c_T \alpha}{c_T \alpha (1 + \delta) + 2d}} \Big) \\
	& \leq
	\sum_{t=1}^T
	\mathrm{P}_{(X, Y) \otimes Z} 
	\Big(\mathcal{R}_{L, \mathrm{P}} (g_{Z_t}) - \mathcal{R}_{L, \mathrm{P}}^* > C n^{-\frac{c_T \alpha}{c_T \alpha (1 + \delta) + 2d}} \Big)
	\leq 4T e^{-\tau},
	\end{align*}
	where $C$ is as in Theorem \ref{RateVPTree}. As a result, with probability $\mathrm{P}_{(X \times Y) \otimes Z} $ at least $1 - 4e^{- \tau}$, there holds
	\begin{align*}
	\mathcal{R}_{L, \mathrm{P}} (f_{Z}) - \mathcal{R}_{L, \mathrm{P}}^*
	\leq C n^{-\frac{c_T \alpha}{c_T \alpha (1 + \delta) + 2d}},
	\end{align*}
	where $C$ depending on $\alpha, \tau, \delta, d, m, M, T$ and $\{\{r_j^t, k_j^t, \mathrm{P}_X (V_j^t)\}_{j=1}^m\}_{t=1}^T$. 
\end{proof}

\section{Conclusion}\label{sec::Conclusion}

In this paper, we proposed and explored a new vertical method for large-scale regression called two-stage best-scored random forest (TBRF) by conducting a statistical learning treatment. This strategy is a just fit for the big data era for its computational efficiency by taking utmost advantage of the parallel computing. More valuable as it is, it is born to settle the boundary discontinuity which has long been a problem to the existing vertical strategies. The \textit{two-stage} stands for dividing the original random tree splitting procedure into two.
To elucidate, we first adopt an adaptive random partition method to split the feature space into different non-overlapping cells in stage one. This stage one serves as a preprocessing for the parallel computing. Then in stage two, we develop the child best-scored random trees for regression on cells, gather them together to form a parent tree. Here, \emph{best-scored} means to select the best performing one. By utilizing the randomness consisting in the partition, ensemble learning can naturally come into being which smoothes the discontinuous boundaries, and the resulting forest therefore reaches excellent asymptotic smoothness. Moreover, the TBRF can also be recognized as an inclusive and versatile framework where different mainstream regression strategies such as support vector regression can be incorporated as value assignment approaches to leaves of trees. Consequently, there can be a lot of high effective and efficient variants of TBRF available to be chosen according to the specific data sets at hand.
Various numerical experiments on synthetic data and real data are given to provide insight into our TBRF. Moreover, comparisons are conducted with other state-of-the-art vertical methods which once again verifies the effectiveness and high efficiency of our novel random forest model.

\acks{The authors are grateful to Professor Ingo Steinwart for his valuable comments and suggestions.
	Hanyuan Hang and Yingyi Chen are supported by fund for building world-class universities (disciplines) of Renmin University of China.
	Johan Suykens acknowledges support of ERC Advanced Grant E-DUALITY 
	(787960), Research Council KU Leuven C14/18/068, FWO GOA4917N.
	The corresponding author is Yingyi Chen.}

\end{document}